\documentclass[11pt]{article}

\usepackage{tpca_langevin}
\usepackage{macros}

\newif\ifparagraphindent
\paragraphindentfalse
\ifparagraphindent\else
\fi

\title{Averaged Annealed Markov Chains Achieve the Algorithmic Threshold for Tensor PCA, Part II: Langevin Dynamics}
\author{Aniket Das \qquad Tselil Schramm\\[0.5em]
  {\normalsize Stanford University}\\
  {\normalsize\texttt{\{aniketd,tselil\}@stanford.edu}}
}
\date{}

\begin{document}

\hypersetup{pageanchor=false}
\pagenumbering{gobble}
\maketitle
\begin{abstract}
\normalsize
\backtrackerfalse

We consider the tensor principal components analysis (tensor PCA) problem on the sphere, in which we are asked to recover an unknown vector in $\SN=\{x\in\R^N:\norm{x}^2=N\}$ from a noisy observation of its $k$-th tensor power, for a fixed $k\ge3$.
Tensor PCA exhibits a well-studied information-computation gap: global algorithms such as spectral methods and sum-of-squares work above an algorithmic threshold, below which we know matching lower bounds (against these and some other restricted models of computation).
In contrast, local algorithms such as Markov chains were only known to succeed from a random initialization at a polynomially larger signal-to-noise ratio, and were believed to be inferior to global algorithms for tensor PCA.
Continuing the program initiated in Part~I~\citep{das2026averaged}, we show that spherical Langevin dynamics, when supplemented with annealing and ergodic averaging, achieves polynomial-time strong recovery above the algorithmic threshold from a uniformly random initialization.
Together with Part~I, our results demonstrate that local Markov chains can succeed all the way down to the algorithmic threshold from an uninformative cold start.

\end{abstract}

\newpage
\tableofcontents
\vfill

\paragraph{Statement on AI Usage.}
All ideas and technical contributions in this paper, as well as Part~I~\citep{das2026averaged}, were biologically generated. Claude Code (Opus 5.5) was used for language editing, proofreading, and transcribing a hand-drawn version of \cref{fig:phase-diagram} into TikZ code. The authors bear sole responsibility for the correctness of this manuscript.
\label{par:ai-disclosure}

\newpage
\pagenumbering{arabic}
\hypersetup{pageanchor=true}

\section{Introduction}
\label{sec:introduction}

This paper is the second of a two-part series on local Markov chains for the tensor principal components analysis (tensor PCA) problem, which involves recovering an unknown vector from a noisy observation of its tensor power. Originally introduced by \citet{richard2014statistical} as a generalization of classical spiked matrix models in statistics and machine learning~\citep{johnstone2001distribution,baik2005phase}, tensor PCA is a canonical planted problem in average-case complexity theory~\citep{hopkins2015tensor,hopkins2016fast,hopkins2017power,bhattiprolu2017certificates,brennan2020reducibility,brennan2021queries,dudeja2021statistical,kunisky2022notes,potechin2022subexponential,kunisky2024cumulants,wein2025kikuchi,kothari2026smooth,bresler2026average}. It is also widely studied in high-dimensional statistics~\citep{lesieur2017statistical,perry2020statistical,jagannath2020statistical} and in the theory of spin glasses in statistical physics~\citep{auffinger2013random,benarous2019landscape,ros2019complex,biroli2020iron,arous2020algorithmic}. Part~I of this series~\citep{das2026averaged} analyzes Glauber dynamics for \emph{Ising tensor PCA}, in which the unknown vector lies on the signed Boolean hypercube $\{\pm1\}^N$. This paper serves as a complement to Part~I and analyzes Langevin dynamics for \emph{spherical tensor PCA}, in which the unknown vector lies on the sphere ${\SN\doteq\{x\in\R^N:\norm{x}^2=N\}}$.

Formally, for a fixed integer $k \geq 3$ and an unknown planted vector $\xopt\in\SN$, we observe the symmetric order-$k$ tensor
\begin{align*}
    Y=\lambda\sqrt{N}\,(\plant)^{\otimes k}+G, \qquad \plant\doteq\frac{\xopt}{\sqrt{N}}
\end{align*}
where $\lambda\ge0$ denotes the signal-to-noise ratio and $G \in (\R^N)^{\otimes k}$ denotes the standard symmetrized Gaussian tensor, i.e., the symmetrization of a $k$-tensor with i.i.d.\ $\cN(0,1)$ entries.
The goal is to recover $\xopt$ (modulo sign for even $k$) from a \emph{single observation} of the noisy tensor $Y$.
Two notions of recovery are standard in the literature. A (possibly randomized) estimator $\widehat x=\widehat x(Y)\in\SN$ for $\xopt$ is said to achieve \emph{weak} (resp. \emph{strong}) \emph{recovery} if its correlation $N^{-1}\dotp{\widehat x}{\xopt}$ is strictly bounded away from zero (resp. $=1-o_N(1)$). For even $k$, the correlation is instead defined as $N^{-1}\abs{\dotp{\widehat x}{\xopt}}$.

As is discussed at length in Part~I~\citep[Section~1]{das2026averaged}, the tensor PCA problem exhibits a well-studied \emph{statistical-computational gap} (see \cref{fig:phase-diagram}). Weak recovery is information-theoretically possible once $\lambda$ exceeds a threshold ${\lambda_{\mathrm{IT}}=\Theta_k(1)}$~\citep{lesieur2017statistical,perry2020statistical,jagannath2020statistical}, whereas global algorithms, such as tensor unfolding~\citep{richard2014statistical,hopkins2015tensor}, partial traces~\citep{hopkins2016fast,anandkumar2017homotopy}, and the sum-of-squares hierarchy~\citep{hopkins2015tensor}, achieve strong recovery in polynomial time above the polynomially larger \emph{algorithmic threshold} ${\lambda_{\mathrm{ALG}}\asymp N^{\tfrac{k-2}{4}}}$, below which no polynomial-time algorithm is believed to achieve even weak recovery. Local Markov chains, perhaps the most natural algorithms that statistical physics offers for planted problems, were believed to fall short of the algorithmic threshold when initialized from an uninformative \emph{cold start}. In particular, \citet{arous2020algorithmic} showed that spherical Langevin dynamics at any fixed temperature, initialized from the uniform measure on $\SN$, recovers the planted signal in polynomial time once ${\lambda\gg N^{\tfrac{k-2}{2}}}$, a polynomial factor of $N^{\tfrac{k-2}{4}}$ above the algorithmic threshold. Below this scale, they conjectured that spherical Langevin dynamics initialized from the uniform measure fails to recover the planted signal in polynomial time~\parencites[Section~1]{arous2020algorithmic}[see also][Section~2(b)]{bandeira2023free}, providing evidence via lower bounds for disorder-dependent Gibbs-type initializations. The same scale is also believed to be the threshold for other cold-started local algorithms, such as power iteration and approximate message passing~\citep{richard2014statistical,huang2022power,wu2024sharp}. This apparent weakness motivates the following central question of our two-part work:

\begin{center}
\interlinepenalty=10000
\emph{``Can polynomial-time algorithms based on canonical local Markov chains initialized from a cold start achieve recovery down to the algorithmic threshold?''}
\end{center}

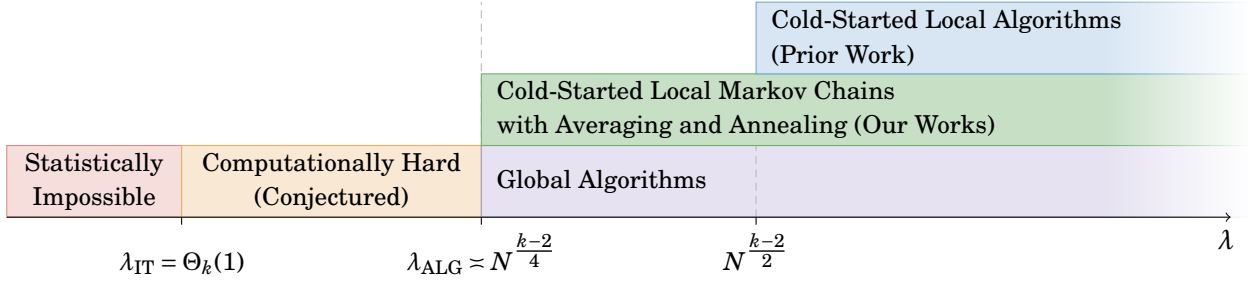
\begin{figure}[t]
\centering
\definecolor{phaseImpossible}{HTML}{D66B63}
\definecolor{phaseHard}{HTML}{E39A3B}
\definecolor{phaseGlobal}{HTML}{8C7BC0}
\definecolor{phaseLocal}{HTML}{4F8FCB}
\definecolor{phaseOurs}{HTML}{4E9A4A}
\begin{tikzpicture}[
    x=0.01\linewidth, y=1mm,
    bandtext/.style={font=\small, inner sep=0pt, align=center},
    ticktext/.style={font=\small, anchor=base, inner sep=0pt},
    guide/.style={gray!55, densely dashed},
  ]
  \def\xIT{14}
  \def\xALG{38}
  \def\xLOC{60}
  \def\xFade{88}
  \def\xEnd{100}
  \def\closedbar#1#2#3#4#5{\filldraw[draw=#1!80, fill=#1!22] (#2,#4) rectangle (#3,#5);}
  \def\openbar#1#2#3#4#5{\fill[#1!#2] (#3,#4) rectangle (\xFade,#5);
    \shade[left color=#1!#2, right color=white] (\xFade,#4) rectangle (\xEnd,#5);
    \draw[#1!80] (\xFade,#4) -- (#3,#4) -- (#3,#5) -- (\xFade,#5);
    \foreach \i in {1,...,20} {
      \pgfmathsetmacro\fadeop{1-\i/20}
      \draw[#1!80, draw opacity=\fadeop]
        ({\xFade+(\i-1)*(\xEnd-\xFade)/20},#4) -- ({\xFade+\i*(\xEnd-\xFade)/20},#4)
        ({\xFade+(\i-1)*(\xEnd-\xFade)/20},#5) -- ({\xFade+\i*(\xEnd-\xFade)/20},#5);
    }}
  \draw[guide] (\xALG,0) -- (\xALG,28.5);
  \closedbar{phaseImpossible}{0}{\xIT}{0}{9.5}
  \closedbar{phaseHard}{\xIT}{\xALG}{0}{9.5}
  \openbar{phaseGlobal}{22}{\xALG}{0}{9.5}
  \draw[guide] (\xLOC,0) -- (\xLOC,19);
  \openbar{phaseOurs}{32}{\xALG}{9.5}{19}
  \openbar{phaseLocal}{22}{\xLOC}{19}{28.5}
  \node[bandtext] at ({\xIT/2},4.75) {Statistically\\Impossible};
  \node[bandtext] at ({(\xIT+\xALG)/2},4.75) {Computationally Hard\\(Conjectured)};
  \node[bandtext, anchor=west] at ({\xALG+1.2},4.75) {Global Algorithms};
  \node[bandtext, anchor=west, align=left] at ({\xALG+1.2},14.25)
    {Cold-Started Local Markov Chains \\ with Averaging and Annealing (Our Works)};
  \node[bandtext, anchor=west, align=left] at ({\xLOC+1.2},23.75)
    {Cold-Started Local Algorithms\\(Prior Work)};
  \draw[->] (0,0) -- (98,0);
  \node[anchor=north, inner sep=0pt] at (97.6,-1.6) {$\lambda$};
  \foreach \x in {\xIT,\xALG,\xLOC} \draw (\x,0) -- (\x,-1);
  \node[ticktext] at (\xIT,-7) {$\lambda_{\mathrm{IT}}=\Theta_k(1)$};
  \node[ticktext] at (\xALG,-7) {$\lambda_{\mathrm{ALG}}\asymp N^{\tfrac{k-2}{4}}$};
  \node[ticktext] at (\xLOC,-7) {$N^{\tfrac{k-2}{2}}$};
\end{tikzpicture}
\caption{Phase diagram for the recovery problem in tensor PCA in terms of $\lambda$.}
\label{fig:phase-diagram}
\end{figure}

For Ising tensor PCA, Part~I~\citep{das2026averaged} answered this question in the affirmative by establishing that Glauber dynamics, a canonical local Markov chain on the hypercube, achieves strong recovery in polynomial time above the algorithmic threshold from a uniformly random initialization when supplemented with annealing and ergodic averaging. In this paper, we show that the same holds for spherical tensor PCA, with Glauber dynamics replaced by spherical Langevin dynamics, the canonical local Markov chain on the sphere, which was analyzed at a fixed temperature in~\citep{arous2020algorithmic}. This extension is far from straightforward, as spherical tensor PCA differs substantially from its Ising counterpart. In particular, $\SN$ is a continuous state space with non-Euclidean geometry, and spherical Langevin dynamics is a Markov diffusion process, as opposed to a jump process like Glauber dynamics. Nevertheless, as we shall discuss in \cref{sec:proof-sketch}, the key conceptual ideas of Part~I, which rely on only a few structural properties of the problem, carry over to the sphere once carefully adapted.

\subsection{Our Results}
\label{subsec:our-results}

\begin{algorithm}[t]
\caption{Annealed Spherical Langevin Dynamics with Ergodic Averaging}
\label[algorithm]{alg:averaged-annealed-langevin}
\textbf{Input:} Hamiltonian $H:\SN\to\R$, initial state $\tilX_0\in\SN$, inverse temperatures $\beta_0,\beta>0$, and \mbox{simulation} times $T_{\mathrm{av}},T>0$.
\makeatletter
\providecommand*{\theHALG@line}{}
\renewcommand*{\theHALG@line}{\thealgorithm.\arabic{ALG@line}}
\makeatother
\begin{algorithmic}[1]
    \State Run spherical Langevin dynamics with Hamiltonian $H$, inverse temperature $\beta_0$, initial state $\tilX_0$, and time $T_{\mathrm{av}}$ to obtain $(\tilX_t)_{0\le t\le T_{\mathrm{av}}}$:
    \begin{align*}
        d\tilX_t=\Bigl(\beta_0 \cdot Q_{\tilX_t}\grad H(\tilX_t)-\frac{N-1}{2N}\tilX_t\Bigr)\,dt+Q_{\tilX_t}\,d\widetilde B_t,
        \qquad
        Q_x\doteq I-\frac{xx^\top}{\norm{x}^2}.
    \end{align*}
    \State Compute the ergodic average:
    \begin{align*}
        \Xbar_{T_{\mathrm{av}}}\doteq\frac{1}{T_{\mathrm{av}}}\int_0^{T_{\mathrm{av}}}\tilX_t\,dt.
    \end{align*}
    \State Run spherical Langevin dynamics with Hamiltonian $H$, inverse temperature $\beta$, initial state ${X_0=\tfrac{\sqrt N\,\Xbar_{T_{\mathrm{av}}}}{\norm{\Xbar_{T_{\mathrm{av}}}}}}$, and time $T$ to obtain $(X_t)_{0\le t\le T}$:
    \begin{align*}
        dX_t=\Bigl(\beta \cdot Q_{X_t}\grad H(X_t)-\frac{N-1}{2N}X_t\Bigr)\,dt+Q_{X_t}\,dB_t,
        \qquad
        Q_x\doteq I-\frac{xx^\top}{\norm{x}^2}.
    \end{align*}
\end{algorithmic}
\textbf{Output:} $X_T$.
\end{algorithm}

Our first result is a simple, fully local algorithm for spherical tensor PCA that achieves polynomial-time strong recovery for all ${\lambda\gtrsim\lambda_{\mathrm{ALG}}}$. Following the template developed in Part~I~\citep[Section~1.1]{das2026averaged}, our algorithm is spherical Langevin dynamics from a uniform initialization, with a simple twist. \Cref{alg:averaged-annealed-langevin} first runs high temperature spherical Langevin dynamics and computes the ergodic average of its trajectory. It then normalizes this average to lie on $\SN$ and uses the resulting state as a warm start for spherical Langevin dynamics at a much lower, constant temperature. As described in the following theorem, we prove that this guarantee holds whenever the tensor order $k$ is odd.

\begin{theorem}[Informal]
\label{thm:langevin-odd-k-informal}
Let $k\ge3$ be odd, and let $H(x)\doteq N^{-\tfrac{k-1}{2}}\dotp{Y}{x^{\otimes k}}$ denote the tensor PCA Hamiltonian on $\SN$ with $\lambda\gtrsim\lambda_{\mathrm{ALG}}$. Then, with a uniformly random initialization and suitable inverse temperatures, \cref{alg:averaged-annealed-langevin} run with $T_{\mathrm{av}}\gtrsim N^{2k-2+o(1)}$ and $1\lesssim T\lesssim\poly(N)$ strongly recovers the planted signal with probability at least $1-\exp(-N^{1/2-o(1)})$.
\end{theorem}

When the tensor order $k$ is even, the ergodic average no longer provides an informative warm start. However, this is easily remedied by using ergodic second moments instead. To this end, \cref{alg:annealed-langevin-ergodic-second-moments}, which arises as a simple modification of \cref{alg:averaged-annealed-langevin}, initializes low temperature spherical Langevin dynamics using the top eigenvector of the ergodic second moment of ${F(x)=(x,1)\in\R^{N+1}}$ along the high temperature trajectory. We prove that the resulting algorithm achieves polynomial-time strong recovery above the algorithmic threshold for every tensor order $k\ge3$, irrespective of its parity.

\begin{algorithm}[t]
\caption{Annealed Spherical Langevin Dynamics with Augmented Ergodic Second Moment}
\label[algorithm]{alg:annealed-langevin-ergodic-second-moments}
\textbf{Input:} Hamiltonian $H:\SN\to\R$, initial state $\tilX_0\in\SN$, inverse temperatures $\beta_0,\beta>0$, and \mbox{simulation} times $T_{\mathrm{av}},T>0$.
\makeatletter
\providecommand*{\theHALG@line}{}
\renewcommand*{\theHALG@line}{\thealgorithm.\arabic{ALG@line}}
\makeatother
\begin{algorithmic}[1]
    \State Run spherical Langevin dynamics with Hamiltonian $H$, inverse temperature $\beta_0$, initial state $\tilX_0$, and time $T_{\mathrm{av}}$ to obtain $(\tilX_t)_{0\le t\le T_{\mathrm{av}}}$:
    \begin{align*}
        d\tilX_t=\Bigl(\beta_0 \cdot Q_{\tilX_t}\grad H(\tilX_t)-\frac{N-1}{2N}\tilX_t\Bigr)\,dt+Q_{\tilX_t}\,d\widetilde B_t,
        \qquad
        Q_x\doteq I-\frac{xx^\top}{\norm{x}^2}.
    \end{align*}
    \State Compute the augmented ergodic second moment:
    \begin{align*}
        \bar M_{T_{\mathrm{av}}}\doteq\frac{1}{T_{\mathrm{av}}}\int_0^{T_{\mathrm{av}}}F(\tilX_t)^{\otimes2}\,dt,
        \qquad
        F(x)\doteq(x,1)\in\R^{N+1}.
    \end{align*}
    \State Compute a top unit eigenvector $(v,a)\in\R^N\times\R$ of $\bar M_{T_{\mathrm{av}}}$, oriented so that $a\ge0$.
    \State Run spherical Langevin dynamics with Hamiltonian $H$, inverse temperature $\beta$, initial state ${X_0=\tfrac{\sqrt N\,v}{\norm{v}}}$, and time $T$ to obtain $(X_t)_{0\le t\le T}$:
    \begin{align*}
        dX_t=\Bigl(\beta \cdot Q_{X_t}\grad H(X_t)-\frac{N-1}{2N}X_t\Bigr)\,dt+Q_{X_t}\,dB_t,
        \qquad
        Q_x\doteq I-\frac{xx^\top}{\norm{x}^2}.
    \end{align*}
\end{algorithmic}
\textbf{Output:} $X_T$.
\end{algorithm}

\begin{theorem}[Informal]
\label{thm:langevin-every-k-informal}
Let $k\ge3$, and let $H(x)\doteq N^{-\tfrac{k-1}{2}}\dotp{Y}{x^{\otimes k}}$ denote the tensor PCA Hamiltonian on $\SN$ with $\lambda\gtrsim\lambda_{\mathrm{ALG}}$. Then, with a uniformly random initialization and suitable inverse temperatures, \cref{alg:annealed-langevin-ergodic-second-moments} run with $T_{\mathrm{av}}\gtrsim N^{2k+o(1)}$ and $1\lesssim T\lesssim\poly(N)$ strongly recovers the planted signal with probability at least $1-\exp(-N^{1/2-o(1)})$.
\end{theorem}

\begin{remark}
\label{rem:second-moment-algorithm-locality}
Mirroring Part~I~\citep[Remark~1.3]{das2026averaged}, the warm start in \cref{alg:annealed-langevin-ergodic-second-moments} is computed via an ergodic statistic of the underlying local Markov chain itself. This differs from prior warm-started local algorithms~\citep{richard2014statistical,anandkumar2017homotopy,lovig2025almost}, which obtain their initializations from the observed tensor $Y$ via a global algorithm. Furthermore, because this ergodic statistic is accumulated online, the eigenvector computation of Step~3 in \cref{alg:annealed-langevin-ergodic-second-moments} can itself be carried out locally, by evolving Oja's flow~\citep{oja1982simplified,kumar2023streaming} alongside the high temperature trajectory,
\begin{align*}
    \dot z_t=\eta_t\dotp{F(\tilX_t)}{z_t}F(\tilX_t),
    \qquad
    F(x)=(x,1)\in\R^{N+1},
    \qquad
    z_0\sim\Unif(\bS^N),
\end{align*}
with step sizes $\eta_t>0$, and replacing $(v,a)$ with $z_{T_{\mathrm{av}}}$ normalized to unit norm and oriented so that its last coordinate is nonnegative. Since the velocity of $z_t$ in the above ODE depends on the trajectory only through the current state $\tilX_t$, the resulting algorithm needs to maintain only the current state and a single auxiliary vector, much like the running average in \cref{alg:averaged-annealed-langevin}. For simplicity, we retain the explicit eigenvector computation, since analyzing this variant would obscure our main ideas.
\end{remark}

We highlight that \cref{alg:averaged-annealed-langevin,alg:annealed-langevin-ergodic-second-moments} are direct adaptations of the two algorithms of Part~I~\citep[Section~1.1]{das2026averaged}, with spherical Langevin dynamics in place of Glauber dynamics. As such, they remain agnostic to the Hamiltonian, i.e., they are defined for any smooth $H:\SN\to\R$ and make no use of the specific form of the tensor PCA Hamiltonian. We thus expect our approach to extend to other planted problems on the sphere and, more broadly, to state spaces beyond the hypercube and the sphere.

\subsection{Acknowledgements}
\label{subsec:ack}
We thank Brice Huang for helpful discussions and feedback.
The authors acknowledge the support of the NSF AI Institute for Foundations of Machine Learning (IFML) and of T.S.'s NSF CAREER award \# 2143246.

\subsection*{Organization}

We sketch the key ideas behind our proofs in \cref{sec:proof-sketch}, and review the necessary notation and preliminaries in \cref{sec:preliminaries}. In \cref{sec:averaged-spherical-derivatives}, we carry out a signal-noise decomposition of the average spherical gradient and Hessian of the tensor PCA Hamiltonian, and \cref{sec:ergodic-averaging} shows that ergodic averages of high temperature spherical Langevin dynamics recover these averaged derivatives. We then analyze the second phase of \cref{alg:averaged-annealed-langevin} in \cref{sec:strong-recovery-odd}, proving a formal version of \cref{thm:langevin-odd-k-informal}. Finally, we extend this analysis to \cref{alg:annealed-langevin-ergodic-second-moments} in \cref{sec:strong-recovery-every-order}, proving a formal version of \cref{thm:langevin-every-k-informal}.

\section{Proof Sketch}
\label{sec:proof-sketch}

Before presenting our key ideas, we first describe the tensor PCA Hamiltonian in more detail and fix some notation. A natural approach to recovery is to solve the Maximum Likelihood Estimation (MLE) problem, defined as follows:
\begin{align}
    \widehat{x}_{\mathrm{MLE}}
    \in \argmax_{x\in\SN}\dotp{Y}{x^{\otimes k}}.
    \label{eq:tensor-pca-mle}
\end{align}
We approach the objective in \cref{eq:tensor-pca-mle} indirectly, by running spherical Langevin dynamics with the corresponding Hamiltonian\footnote{Contrary to the usual convention in statistical physics, we define the Hamiltonian as the objective to be maximized.} $H:\SN\to\R$:
\begin{align}
    H(x)&\doteq N^{-\tfrac{k-1}{2}}\dotp{Y}{x^{\otimes k}}
        =N\lambda m(x)^k+H_0(x),\nonumber\\
    m(x)&\doteq\frac{1}{N}\dotp{\xopt}{x}
        =\frac{1}{\sqrt{N}}\dotp{\plant}{x},
        \qquad H_0(x)\doteq N^{-\tfrac{k-1}{2}}\dotp{G}{x^{\otimes k}},
        \label{eq:tensor-pca-overlap-noise}
\end{align}
where $m:\SN\to[-1,1]$ is the normalized correlation with the planted vector and $H_0$ is the spherical pure $k$-spin Hamiltonian~\citep{crisanti1992spherical,auffinger2013random}, a centered isotropic Gaussian process on $\SN$ with \mbox{covariance}
\begin{align}
    \label{eq:kspin-covariance}
    \E[H_0(x)H_0(y)]
    =N\left(\frac{\dotp{x}{y}}{N}\right)^k,
    \qquad x,y\in\SN.
\end{align}

We refer to $N\lambda m(x)^k$ and $H_0$ as the signal and the disorder, respectively, and let $\pi$ denote the uniform measure on $\SN$.

We now sketch the main ideas behind our results. Our analysis extends the proof techniques that were developed in Part~I~\citep{das2026averaged} for Glauber dynamics, to spherical Langevin dynamics, and each of its steps has a direct counterpart in Part~I. At the level of a proof overview, the core ideas behind the two analyses are in many cases identical. Thus, to keep this paper self-contained, this section closely follows the corresponding section of Part~I~\citep[Section~2]{das2026averaged}, often verbatim. The required technical modifications to the proof strategy stem from the Riemannian geometry of $\SN$ and the fact that spherical Langevin dynamics is a Markov diffusion process rather than a jump process.
We discuss our analysis of \cref{alg:averaged-annealed-langevin} for odd $k$ in \cref{subsec:sketch-curvature-entropy,subsec:sketch-averaging-disorder,subsec:sketch-ergodic-averaging,subsec:sketch-reverse-planting}, and the ideas behind our analysis of \cref{alg:annealed-langevin-ergodic-second-moments} in \cref{subsec:sketch-any-order}, highlighting these modifications along the way.

\subsection{Escaping the Equator}
\label{subsec:sketch-curvature-entropy}

The starting point of our analysis is the following heuristic picture of spherical Langevin dynamics proposed by Ben Arous, Gheissari, and Jagannath~\citep[Section~2.5.1]{arous2020algorithmic}. Let $(X_t)_{t\ge0}$ denote the trajectory of spherical Langevin dynamics for $H$ at a constant inverse temperature $\beta>0$. For a smooth $f:\R^N\to\R$, we write ${\gradsp f(x)\doteq Q_x\grad f(x)}$ for its spherical gradient at $x\in\SN$, where $\grad f$ denotes its Euclidean gradient and ${Q_x=I-\tfrac{xx^\top}{N}}$ is the orthogonal projection onto the tangent space of $\SN$ at $x$, as in \cref{alg:averaged-annealed-langevin}. In our conventions, the evolution equation~\citep[Eq.~(2.14)]{arous2020algorithmic} for the correlation $m_t\doteq m(X_t)$ shows that, modulo stochastic fluctuations,
\begin{align}
    \frac{dm_t}{dt}
    \approx\underbrace{\beta k\lambda\,m_t^{k-1}\bigl(1-m_t^2\bigr)}_{\substack{\text{signal}\\\text{drift}}}
    -\underbrace{\frac{N-1}{2N}\,m_t}_{\substack{\text{entropic}\\\text{drift}}}
    +\underbrace{\frac{\beta}{N}\dotp{\xopt}{\gradsp H_0(X_t)}}_{\substack{\text{disorder}\\\text{drift}}}.
    \label{eq:sketch-overlap-ode}
\end{align}
The disorder drift is a highly intricate random function of the trajectory, since $X_t$ itself depends on $H_0$ through the entire history of the dynamics. Controlling this term is perhaps the most challenging obstacle in our analysis, and we shall return to it later in this section.

Setting aside the disorder drift for now, the evolution of $m_t$ is essentially a competition between the entropic and signal drifts. The entropic drift, which is the drift of $m_t$ under spherical Brownian motion, i.e., for $\beta=0$, drags $m_t$ towards the equator ${\{x\in\SN:\abs{m(x)}\lesssim N^{-1/2}\}}$, where the uniform measure on $\SN$ concentrates. The signal drift is induced by the spherical gradient of the signal, which pushes the state towards $\xopt$. Near the equator, the ratio of the signal drift to the entropic drift is approximately $2\beta k\lambda m_t^{k-2}$, which increases with $m_t$. We therefore expect the correlation to grow at an accelerating rate if $2\beta k\lambda m_0^{k-2}>1$, and to be dragged back to the equator otherwise. In other words, spherical Langevin dynamics initialized at correlation $m_0$ should escape the equator and recover the planted signal efficiently if and only if $\lambda m_0^{k-2}\gtrsim1$, or equivalently, $m_0\gtrsim\lambda^{-1/(k-2)}$. This picture is also supported by landscape complexity analyses~\citep{benarous2019landscape,ros2019complex}, which show that spurious local maxima of the Hamiltonian are confined to the band $\abs{m}\lesssim\lambda^{-1/(k-2)}$ around the equator.

A uniformly random initialization has correlation of order $N^{-1/2}$, which would require $\lambda\gtrsim N^{\tfrac{k-2}{2}}$ to escape the equator. This matches the guarantee of~\citep[Theorem~2.1]{arous2020algorithmic} that spherical Langevin dynamics initialized uniformly at random on the hemisphere ${\{x\in\SN:m(x)>0\}}$ recovers the planted signal in $O(1)$ time whenever ${\lambda=N^\alpha}$ with ${\alpha>\tfrac{k-2}{2}}$. At $\lambda\asymp\lambda_{\mathrm{ALG}}$, however, spherical Langevin dynamics requires an initial correlation of order at least ${\lambda_{\mathrm{ALG}}^{-1/(k-2)}\asymp N^{-1/4}}$, even in this idealized picture that ignores the disorder. Our first obstacle is therefore to find a \emph{natural} warm start, which (i) achieves correlation $\Omega(N^{-1/4})$ when $\lambda\asymp\lambda_{\mathrm{ALG}}$, and (ii) can be approximated from the trajectory of spherical Langevin dynamics itself. As we shall now discuss, averaging the spherical gradient of $H$ over the uniform measure on $\SN$ yields a warm start that satisfies both desiderata.

\subsection{Warm Start via the Average Spherical Gradient}
\label{subsec:sketch-averaging-disorder}

Recall that our goal is to construct a warm start from the trajectory of spherical Langevin dynamics itself. Since spherical Langevin dynamics depends on $H$ only through its spherical gradient, so should any such warm start computed from its trajectory. This merits a closer analysis of the spherical gradient, which decomposes as ${\gradsp H(x)=k\lambda m(x)^{k-1}Q_x\xopt+\gradsp H_0(x)}$, and therefore forms a noisy observation of the projection $Q_x\xopt$ of the planted vector onto the tangent space at $x$, since $m(x)^{k-1}\ge0$ for odd $k$. However, at a uniformly random point, the signal contribution has norm of order $\lambda N^{-\tfrac{k-2}{2}}$, and is thus dwarfed by the disorder contribution, whose norm is of order $\sqrt N$ by Gaussian concentration.

Fortunately, although $H_0$ has a rugged landscape with exponentially many local maxima~\citep{auffinger2013random}, its values at typical pairs of points are nearly independent, since its covariance only depends on the overlap $\dotp{x}{y}/N$. Indeed, for independent $x,y\sim\pi$, the law of $\dotp{x}{y}/\sqrt N$ converges to $\cN(0,1)$ by the central limit theorem, and hence
\begin{align*}
    \E\bigl[H_0(x)H_0(y)\bigr]
    =N\Bigl(\frac{\dotp{x}{y}}{N}\Bigr)^{k}
    =N^{1-\frac{k}{2}}\Bigl(\frac{\dotp{x}{y}}{\sqrt N}\Bigr)^{k}
    \lesssim N^{1-\frac{k}{2}}
    \le N^{-\frac12},
\end{align*}
whereas $\E[H_0(x)^2]=N$. Averaging the spherical gradient over $\pi$ should therefore suppress the disorder contribution, while the signal contribution survives, since it points towards $\xopt$ at every point. This motivates the average spherical gradient $u_H\doteq\E_\pi[\gradsp H]$.

By Stein's lemma for the uniform measure on $\SN$, which we establish in \cref{subsec:sphere}, we have ${u_H=\kappa_N\E_\pi[xH(x)]}$ with ${\kappa_N\doteq\tfrac{N-1}{N}}$, so that $u_H$ is linear in the observed tensor $Y$. In \cref{sec:averaged-spherical-derivatives}, we use the rotation invariance of $\pi$ and of the Gaussian process $H_0$ to show that $u_H$ is a noisy observation of $\xopt$,
\begin{align}
    u_H=\kappa_N^{-1}\lambda\sigma_N^2\,\xopt+\zeta,
    \qquad
    \zeta\sim\cN(0,\sigma_N^2I),
    \qquad
    \sigma_N^2=\kappa_N^2N\,\E_{x,x'\,\iidsim\pi}\Bigl[\Bigl(\frac{\dotp{x}{x'}}{N}\Bigr)^{k+1}\Bigr]\asymp_kN^{-\frac{k-1}{2}}.
    \label{eq:sketch-average-gradient}
\end{align}
Thus, $u_H$ has signal-to-noise ratio $\kappa_N^{-1}\lambda\sigma_N\asymp_k\lambda N^{-\tfrac{k-1}{4}}$, in terms of the norms of its signal and noise components, which is a factor of $N^{\tfrac{k-1}{4}}$ larger than that of the spherical gradient at a single uniformly random point. Consequently, the normalized average spherical gradient ${\widehat u_H\doteq\tfrac{\sqrt N\,u_H}{\norm{u_H}}}$ (which replaces the sign rounding $\sgn(u_H)$ used in Part~I~\citep{das2026averaged}), satisfies the following with high probability:
\begin{align*}
    m(\widehat u_H)\asymp_k\min\bigl\{1,\lambda N^{-\frac{k-1}{4}}\bigr\}.
\end{align*}
In particular, if $\lambda\le N^{\tfrac{k-1}{4}}$, then ${\lambda\,m(\widehat u_H)^{k-2}\asymp_k(\lambda N^{-\tfrac{k-2}{4}})^{k-1}}$, and hence $\widehat u_H$ is a sufficiently good warm start for escaping the equator exactly when $\lambda\gtrsim\lambda_{\mathrm{ALG}}$.

This suggests initializing spherical Langevin dynamics at a constant temperature from $\widehat u_H$. However, since $u_H$ aggregates the Hamiltonian over the entire sphere, it is a global statistic. This raises the question of whether it can nevertheless be approximated from the trajectory of spherical Langevin dynamics, as required by the second desideratum.

The idea of averaging the gradient to suppress the disorder also underlies the averaged gradient descent algorithm of \citet{biroli2020iron}, which minimizes the energy~$-H$ by moving along the average of the \emph{Euclidean} gradient $\grad H$ evaluated at many random points of $\SN$ around the current iterate. Using annealed Kac--Rice heuristics and numerical simulations for $k=3$, \citep{biroli2020iron} provide non-rigorous evidence that their method succeeds above the algorithmic threshold. In contrast, our algorithms come with rigorous guarantees and, as we discuss below, approximate the average gradient $u_H$ implicitly from a single trajectory of high temperature Langevin dynamics.

\subsection{Approximating the Average Gradient via High Temperature Langevin Dynamics}
\label{subsec:sketch-ergodic-averaging}

Our analysis of Glauber dynamics in Part~I~\citep{das2026averaged} suggests that $u_H$ can be approximated by the ergodic average ${\Xbar_T\doteq\tfrac{1}{T}\int_0^T\tilX_t\,dt}$ of high temperature spherical Langevin dynamics $(\tilX_t)_{t\ge0}$ for $H$, run at a small inverse temperature $\beta_0>0$. Indeed, since this diffusion has stationary distribution ${\mu_{H,\beta_0}\propto\pi e^{2\beta_0H}}$, the ergodic theorem~\citep[cf.][Section~1.15.9]{bakry2013analysis} implies that $\Xbar_T$ converges to the mean of $\mu_{H,\beta_0}$ as $T\to\infty$. At high temperature, $\mu_{H,\beta_0}$ is a slight tilt of the uniform measure, and expanding to first order in $\beta_0$, we infer from $\E_\pi[x]=0$ and Stein's lemma that
\begin{align*}
    \lim_{T\to\infty}\Xbar_T
    =\frac{\E_\pi\bigl[x\,e^{2\beta_0H(x)}\bigr]}{\E_\pi\bigl[e^{2\beta_0H(x)}\bigr]}
    =2\beta_0\E_\pi\bigl[xH(x)\bigr]+O(\beta_0^2)
    =\frac{2\beta_0}{\kappa_N}u_H+O(\beta_0^2).
\end{align*}
In other words, the mean of the Gibbs measure at high temperature is the linear response of the uniform measure to the field $H$, which Stein's lemma identifies with a multiple of $u_H$.

Since $\beta_0u_H$ vanishes as $\beta_0\to0$, making the above heuristic rigorous requires quantitative versions of both the long time limit and the high temperature expansion, which we derive in \cref{sec:ergodic-averaging}. Notably, our proofs circumvent the analysis of the stationary distribution or the mixing time of high temperature spherical Langevin dynamics, and instead proceed by comparing the diffusion directly with spherical Brownian motion on $\SN$. Since the drift $\beta_0\gradsp H(x)-\tfrac{\kappa_N}{2}x$ of high temperature spherical Langevin dynamics averages to zero over long times, the ergodic average $\Xbar_T$ closely tracks that of the rescaled spherical gradient $\tfrac{2\beta_0}{\kappa_N}\gradsp H(\tilX_t)$. Moreover, the ergodic theorem suggests that the latter concentrates around ${\tfrac{2\beta_0}{\kappa_N}\E_\pi[\gradsp H]=\tfrac{2\beta_0}{\kappa_N}u_H}$, since high temperature spherical Langevin dynamics is a small perturbation of spherical Brownian motion, whose stationary distribution is $\pi$. To quantify this concentration non-asymptotically, we combine the Green operator of spherical Brownian motion, a classical tool for proving central limit theorems for ergodic averages of Markov processes~\citep{komorowski2012fluctuations,eberle2025markov}, with the Bakry--\'Emery curvature-dimension criterion~\citep{bakry2013analysis}. For tensor PCA, the resulting bounds establish that $\beta_0=\lambda^{-1}N^{-\tfrac{k-1}{2}-o(1)}$ and $T_{\mathrm{av}}=N^{2k-2+o(1)}$ suffice for the normalized ergodic average ${\widehat X_{T_{\mathrm{av}}}\doteq\tfrac{\sqrt N\,\Xbar_{T_{\mathrm{av}}}}{\norm{\Xbar_{T_{\mathrm{av}}}}}}$ to closely approximate $\widehat u_H$ with high probability, so that this warm start meets both desiderata. This motivates the construction of \cref{alg:averaged-annealed-langevin}, which first runs high temperature spherical Langevin dynamics and uses the normalization of its ergodic average to initialize spherical Langevin dynamics at a constant temperature.

Our non-asymptotic analysis of the ergodic average is established for generic smooth Hamiltonians satisfying mild regularity conditions, and could be of independent interest. It extends the corresponding analysis for Glauber dynamics in Part~I~\citep{das2026averaged}, which instead establishes a discrete version of the Bakry--\'Emery criterion on the hypercube using elementary Fourier analysis.

Computing the warm start and climbing towards the planted signal require different temperatures, which is why \cref{alg:averaged-annealed-langevin} is annealed. At high temperature, the signal drift in \cref{eq:sketch-overlap-ode} is negligible, so the first phase cannot climb towards the planted signal, even though its ergodic average encodes $u_H$. In contrast, the second phase, run at a constant temperature from $\widehat X_{T_{\mathrm{av}}}$, can exploit the signal drift to escape the equator. Proving that it indeed does so requires us to finally confront the disorder drift in \cref{eq:sketch-overlap-ode}. However, controlling this term now becomes all the more delicate since this warm start itself depends on the disorder.

\subsection{Controlling the Disorder via Reverse Planting}
\label{subsec:sketch-reverse-planting}

To illustrate our ideas, we make two simplifications throughout this subsection. First, we ignore the estimation error of the ergodic average in the first phase, and assume that the second phase of \cref{alg:averaged-annealed-langevin} is initialized exactly at $\widehat u_H$. Second, we assume that $\xopt$ is uniformly random on $\SN$. Our proofs in \cref{sec:strong-recovery-odd} dispense with both these simplifications, and apply to any initialization sufficiently close to $\widehat u_H$, and to any fixed $\xopt\in\SN$ via a carefully constructed gauge transformation, namely a random rotation of $\SN$ that maps $\widehat u_H$ to a fixed point.

By Gaussian concentration, the disorder drift in \cref{eq:sketch-overlap-ode} is of order one uniformly over $\SN$, irrespective of how the trajectory depends on $H_0$. Consequently, once the correlation $m_t$ exceeds a large constant multiple of the \emph{critical correlation} $r_N\doteq\lambda^{-1/(k-1)}$, the signal drift dominates both the entropic and disorder drifts, even in the worst case, until $m_t$ approaches one. As we show in \cref{sec:strong-recovery-odd}, the correlation then converges to one exponentially fast, such that spherical Langevin dynamics achieves strong recovery within $O(1)$ additional time, and retains it over polynomial time horizons. Recall from \cref{subsec:sketch-averaging-disorder} that the warm start has correlation $m_0\doteq m(\widehat u_H)$ of order $\min\{1,\lambda N^{-\tfrac{k-1}{4}}\}$, which itself becomes critical (i.e., $m_0\ge Cr_N$ for a large constant $C$) whenever ${\lambda\gtrsim N^{\tfrac{k-2}{4}+\tfrac{1}{4k}}}$. In this regime, strong recovery thus follows immediately, and it remains to treat the window ${\lambda_{\mathrm{ALG}}\lesssim\lambda\lesssim\lambda_{\mathrm{ALG}}N^{\tfrac{1}{4k}}}$, in which the warm start satisfies the escape condition $\lambda m_0^{k-2}\gtrsim1$ from \cref{subsec:sketch-curvature-entropy}, but is subcritical, i.e., $m_0\ll r_N$. In this window, our key technical challenge is to prove that the correlation nevertheless reaches $Cr_N$ within $O(1)$ time, before the disorder drift can overpower that of the signal.

However, neither worst-case nor typical bounds on the disorder drift suffice to justify this. Below the critical correlation, the signal drift can be much smaller than the worst-case magnitude of the disorder drift, which could thus overwhelm it. Moreover, although the disorder drift is $O(N^{-1/2})$ at points chosen independently of $H_0$, both the trajectory and its initialization depend on $H_0$. One can imagine a particularly dangerous scenario in which the warm start lies close to one of the exponentially many local maxima of $H_0$~\citep{auffinger2013random}, around which the disorder drift could dominate the signal drift and trap the dynamics at polynomial timescales.

We rule out such adversarial scenarios by showing that the dependence of the warm start on the disorder is benign. Recall from \cref{eq:sketch-average-gradient} that $u_H=\kappa_N^{-1}\lambda\sigma_N^2\xopt+\zeta$, where $\zeta=\E_\pi[\gradsp H_0]$ satisfies ${\zeta=\kappa_N\E_\pi[xH_0(x)]}$ by Stein's lemma, such that $\kappa_N^{-1}\dotp{\zeta}{x}$ corresponds to the linear component (or the degree-one spherical harmonic) of $H_0$. In other words, the warm start depends on $H_0$ only through its linear component. Since $H_0$ is an isotropic Gaussian process, Gaussian conditioning shows that $\zeta$, and hence the warm start, is independent of the remainder $H_0(x)-\kappa_N^{-1}\dotp{\zeta}{x}$, which comprises the bulk of the disorder.
This suggests comparing the second phase of \cref{alg:averaged-annealed-langevin} with the same procedure run on the pure disorder $H_0$. Precisely, we consider spherical Langevin dynamics initialized at the normalized average spherical gradient of its Hamiltonian, and call its law the \emph{planted law} when the Hamiltonian is $H$, and the \emph{null law} when the Hamiltonian is $H_0$. In other words, the planted law is the true law of the second phase, and the null law is simply the planted law with $\lambda=0$. The dynamics in the second phase is thus initialized at $\widehat u_H$ under the planted law, and at ${\tfrac{\sqrt N\,\zeta}{\norm{\zeta}}}$ under the null law. By \cref{eq:sketch-overlap-ode} with $\lambda=0$, the correlation under the null law evolves according to
\begin{align}
    \frac{dm_t}{dt}
    \approx\frac{1}{N}\dotp{\xopt}{\beta\gradsp H_0(X_t)-\frac{N-1}{2N}X_t}.
    \label{eq:sketch-null-drift}
\end{align}
We call the right-hand side of \cref{eq:sketch-null-drift} the \emph{null drift}, since it is precisely the net contribution of the entropic and disorder drifts. Since $\xopt$ is uniformly random and the null dynamics never sees it, $\xopt$ is independent of the null trajectory. Conditional on the trajectory, the null drift is therefore the projection of a vector of norm $O(1)$ onto the uniformly random direction $\xopt/\sqrt N$, and concentration of measure on the sphere shows that it remains of order $N^{-1/2}$ along the entire null trajectory.

Thus, to control the disorder drift in the true dynamics, it remains to transfer this estimate to the planted law via a change of measure argument based on Girsanov's theorem, which we call the \emph{reverse planting trick}. As the name suggests, it runs opposite to the classical planting trick~\citep{achlioptas2008algorithmic}, in which the null law is the true law of interest, and estimates established under the planted law are transferred to it via contiguity. However, executing this argument is quite challenging since the planted and null laws are far from contiguous. Indeed, their Hamiltonians differ by the signal $N\lambda m(x)^k$, which biases both the warm start and the drift of the planted dynamics at all times. The cost of transferring an estimate therefore grows with the signal accumulated along the trajectory. Fortunately, since the worst-case argument above takes over once the correlation reaches $Cr_N$, we only need to transfer estimates up to this time. To this end, we prove an \emph{adaptive contiguity} bound, which controls the cost of transferring an event up to any stopping time in terms of the signal accumulated before it.

We apply this adaptive contiguity bound up to the \emph{failure time}, i.e., the first time before the correlation reaches $Cr_N$ at which the null drift overpowers the signal drift. Until the failure time, the signal drift dominates, so the correlation grows at an accelerating rate, as in \cref{subsec:sketch-curvature-entropy}, and reaches $Cr_N$ within $O(1)$ time unless a failure occurs first. With a delicate stopped martingale analysis, we show that with high probability, no failure occurs, so that the dynamics reaches $Cr_N$ within $O(1)$ time and then achieves strong recovery. This completes our proof sketch of \cref{thm:langevin-odd-k-informal}.

\subsection{Strong Recovery for Any Tensor Order}
\label{subsec:sketch-any-order}

We now turn our attention to \cref{alg:annealed-langevin-ergodic-second-moments}, and first focus on the case when $k$ is even. In this case, $H(-x)=H(x)$, so by symmetry, the average spherical gradient $u_H$ vanishes identically, and carries no information about the planted signal. However, our previous intuition about suppressing the disorder via averaging still applies to the spherical Hessian $\gradsp^2H$, which decomposes as ${\gradsp^2H(x)\approx\tfrac{k(k-1)\lambda}{N}m(x)^{k-2}\xopt(\xopt)^\top+\gradsp^2H_0(x)}$, and therefore forms a noisy observation of $\xopt(\xopt)^\top$, since $m(x)^{k-2}\ge0$ for even $k$. Specifically, in \cref{sec:averaged-spherical-derivatives}, we show that the \emph{average spherical Hessian} has a spiked Wigner type structure with a planted spike in the direction of $\xopt$,
\begin{align}
    A_H\doteq\E_\pi\bigl[\gradsp^2H\bigr]
    =\frac{\lambda\gamma_N}{N}\bigl(\xopt(\xopt)^\top-I\bigr)+W_N,
    \qquad
    \gamma_N\asymp_kN^{-\frac{k-2}{2}},
    \label{eq:sketch-average-hessian}
\end{align}
where ${W_N\doteq\E_\pi[\gradsp^2H_0]}$ is a centered Gaussian random matrix. The spike in \cref{eq:sketch-average-hessian} has eigenvalue of order $\lambda\gamma_N$, whereas $\norm{W_N}_{\mathrm{op}}\lesssim\sqrt{\gamma_N}$ with high probability. Hence, the spike dominates the noise whenever $\lambda\sqrt{\gamma_N}\asymp_k\lambda N^{-\tfrac{k-2}{4}}\gtrsim1$, i.e., whenever $\lambda\gtrsim\lambda_{\mathrm{ALG}}$, in which case, by the Davis--Kahan theorem~\citep[Theorem~4.5.5]{vershynin2018high}, the top eigenvector of $A_H$, rescaled to lie on $\SN$, achieves constant correlation with $\xopt$.\footnote{As $\xopt$ and $-\xopt$ are indistinguishable for even $k$, we assume for simplicity that the top eigenvector is aligned with $\xopt$.} The rescaled top eigenvector is then a warm start with correlation far above the critical correlation $r_N$. Consequently, by the worst-case argument in \cref{subsec:sketch-reverse-planting}, spherical Langevin dynamics achieves strong recovery within $O(1)$ time.

As before, despite being a global statistic, $A_H$ can be approximated from the trajectory of high temperature spherical Langevin dynamics, now via the ergodic second moment ${\Sigmabar_T\doteq\tfrac{1}{T}\int_0^T\tilX_t\tilX_t^\top\,dt}$. Indeed, by the ergodic theorem, $\Sigmabar_T$ converges to the second moment of $\mu_{H,\beta_0}$ as $T\to\infty$, and expanding to first order in $\beta_0$ as in \cref{subsec:sketch-ergodic-averaging}, we infer from $\E_\pi[xx^\top]=I$ and Stein's lemma that
\begin{align*}
    \lim_{T\to\infty}\Sigmabar_T
    =\frac{\E_\pi\bigl[xx^\top e^{2\beta_0H(x)}\bigr]}{\E_\pi\bigl[e^{2\beta_0H(x)}\bigr]}
    =I+2\beta_0\E_\pi\bigl[(xx^\top-I)H(x)\bigr]+O(\beta_0^2)
    =I+2\beta_0A_H+O(\beta_0^2).
\end{align*}
Since shifting by the identity does not change eigenvectors, the top eigenvector of $\Sigmabar_T$ thus approximates that of $A_H$, as we show in \cref{sec:ergodic-averaging}.

However, by symmetry, $A_H$ vanishes identically when $k$ is odd, just as $u_H$ does when $k$ is even. To obtain an algorithm that is agnostic to the parity of $k$, \cref{alg:annealed-langevin-ergodic-second-moments} instead uses the ergodic second moment of ${F(x)\doteq(x,1)\in\R^{N+1}}$, whose blocks comprise both $\Sigmabar_T$ and the ergodic average $\Xbar_T$. Combining both high temperature expansions, we thus expect that for large $T$ and small $\beta_0$,
\begin{align*}
    \bar M_T\doteq\frac{1}{T}\int_0^TF(\tilX_t)^{\otimes2}\,dt
    =\begin{pmatrix}
        \Sigmabar_T & \Xbar_T\\
        \Xbar_T^\top & 1
    \end{pmatrix}
    \approx I_{N+1}+2\beta_0\cA_H,
    \qquad
    \cA_H\doteq\begin{pmatrix}
        A_H & \kappa_N^{-1}u_H\\
        \kappa_N^{-1}u_H^\top & 0
    \end{pmatrix}.
\end{align*}
When $k$ is odd, $A_H=0$, so the top eigenvector of $\cA_H$, when oriented such that its last coordinate is nonnegative, is ${\tfrac{1}{\sqrt2}(u_H/\norm{u_H},1)}$. Rescaling its first $N$ coordinates to lie on $\SN$ then yields $\widehat u_H$, i.e., the warm start of \cref{alg:averaged-annealed-langevin}, to which the analysis in \cref{subsec:sketch-reverse-planting} applies. When $k$ is even, $u_H=0$, so $\cA_H$ is block diagonal, and its top eigenvector is obtained by appending a zero to that of $A_H$, which recovers the warm start for even $k$. Thus, the same top eigenvector computation yields a sufficiently good warm start for either parity, and \cref{alg:annealed-langevin-ergodic-second-moments} is agnostic to the parity of $k$. This completes our proof sketch of \cref{thm:langevin-every-k-informal}.

\section{Notation and Preliminaries}
\label{sec:preliminaries}

For $d\geq2$, we write $M_d(\R)$ for the space of real $d\times d$ matrices and $\Sym_d\subseteq M_d(\R)$ for the subspace of real symmetric matrices. For $M\in M_d(\R)$, we write $\diag(M)$ for the diagonal matrix with the same diagonal entries as $M$. For a vector $a\in\R^d$, we also use $\diag(a)$ to denote the diagonal matrix with entries $a_1,\ldots,a_d$. For a symmetric matrix $A\in\Sym_d$, we denote its eigenvalues in nonincreasing order by $\lambda_1(A)\geq\lambda_2(A)\geq\cdots\geq\lambda_d(A)$.

For $\ell\geq1$, we use $(\R^d)^{\otimes\ell}$ to denote the space of order-$\ell$ tensors, viewed as multilinear forms on $\R^d$, and $(\R^d)^{\odot\ell}$ to denote the subspace of symmetric tensors.

Unless specified otherwise, $\norm{\cdot}$ denotes the Euclidean norm, and $\bS^{d-1}$ denotes the unit sphere in $\R^d$. For any matrix $A\in M_d(\R)$, its operator norm is defined by $\norm{A}_{\mathrm{op}}=\sup_{x\in\bS^{d-1}}\norm{Ax}$. This definition extends naturally to tensors, wherein, for $T\in(\R^d)^{\otimes\ell}$, we define $\norm{T}_{\mathrm{op}}=\sup_{x_1,\ldots,x_\ell\in\bS^{d-1}}|T(x_1,\ldots,x_\ell)|$. We use $e_1,\ldots,e_N$ to denote the standard basis vectors of $\R^N$.

We use $c,C>0$ to denote universal constants and $c_k,C_k>0$ to denote constants depending only on $k$, whose values may change from line to line unless fixed explicitly. For nonnegative quantities $a$ and $b$, we write $a\lesssim b$ or $a=O(b)$ if $a\le Cb$, and $a\gtrsim b$ or $a=\Omega(b)$ if $a\ge cb$. We write $a\asymp b$ or $a=\Theta(b)$ if both hold, i.e., if $cb\le a\le Cb$. We write $a\ll b$ or $a=o(b)$ if $a/b\to0$ as $N\to\infty$, and $a\gg b$ if $b/a\to0$. Wherever appropriate, we use subscripts to indicate the parameters on which the constants depend, e.g., $a\lesssim_kb$ means $a\le C_kb$.

\subsection{Riemannian Geometry of the Sphere}
\label{subsec:sphere}
We present a brief introduction to analysis on the sphere, and refer the readers to \citet{lee2018introduction} and \citet{boumal2023introduction} for a detailed treatment. Let ${\SN=\bS^{N-1}(\sqrt N)=\{x\in\R^N:\norm{x}^2=N\}}$ denote the sphere of radius $\sqrt N$, and let $\pi$ denote the uniform measure on $\SN$. For $x,y\in\SN$, the geodesic distance between $x$ and $y$ is ${d_{\SN}(x,y)=\sqrt N\arccos(\dotp{x}{y}/N)}$~\citep[cf.][Example~10.21]{boumal2023introduction}. Moreover, ${\norm{x-y}^2=2(N-\dotp{x}{y})}$ and $d_{\SN}(x,y)\le C\norm{x-y}$ for all $x,y\in\SN$, where $C>0$ is a universal constant.

We equip $\SN$ with the Riemannian metric induced by the Euclidean inner product, so that the tangent space of $\SN$ at $x$ is $T_x\SN=\{v\in\R^N:\dotp{v}{x}=0\}$. For any $x\in\R^N\setminus\{0\}$, we define the orthogonal projector $Q_x$ and the constant $\kappa_N$ as follows:
\begin{align*}
    Q_x\doteq I-\frac{xx^\top}{\norm{x}^2},
    \qquad
    \kappa_N\doteq\frac{N-1}{N}.
\end{align*}
In particular, for every $x\in\SN$, ${Q_x=I-\tfrac{xx^\top}{N}}$ is the orthogonal projector onto $T_x\SN$. Furthermore, we assume $N\geq3$ throughout, such that $\tfrac{2}{3}\le\kappa_N<1$, and in particular $\kappa_N\asymp1$.

We use $\grad$, $\grad^2$, and $\Delta$ to denote the Euclidean gradient, Hessian, and Laplacian, respectively. For a smooth $f:\SN\to\R$, these operators are applied to an arbitrary smooth extension of $f$ to a neighborhood of $\SN$ in $\R^N$. The spherical gradient $\gradsp f$, the covariant Hessian $\gradsp^2f$, and the Laplace--Beltrami operator $\lapsp f$ then satisfy the following identities for any such extension~\citep[cf.][\mbox{Section~7.2}]{boumal2023introduction}:
\begin{align}
\label{eq:spherical-derivatives}
    \begin{aligned}
        \gradsp f(x)&=Q_x\grad f(x),\\
        \gradsp^2f(x)&=Q_x\grad^2f(x)Q_x-\frac{\dotp{\grad f(x)}{x}}{N}Q_x,\\
        \lapsp f(x)&=\Tr\bigl(\gradsp^2f(x)\bigr)=\Delta f(x)-\frac{1}{N}\dotp{x}{\grad^2f(x)x}-\kappa_N\dotp{\grad f(x)}{x}.
    \end{aligned}
\end{align}
More generally, for any integer $\ell\ge1$, we use $\gradsp^\ell f(x)$ to denote the $\ell$-th covariant derivative of $f$ at $x$, which is an $\ell$-linear form on $T_x\SN$~\parencites[Chapter~4]{lee2018introduction}[Section~10.7]{boumal2023introduction}, whose operator norm is given by
\begin{align*}
    \norm*{\gradsp^\ell f(x)}_{\mathrm{op}}
    =\max_{v_1,\ldots,v_\ell\in T_x\SN\cap\bS^{N-1}}\abs{\gradsp^\ell f(x)[v_1,\ldots,v_\ell]}
    =\max_{u_1,\ldots,u_\ell\in\bS^{N-1}}\abs{\gradsp^\ell f(x)[Q_xu_1,\ldots,Q_xu_\ell]}.
\end{align*}
Accordingly, we identify $\gradsp^\ell f(x)$ with the tensor $(u_1,\ldots,u_\ell)\mapsto\gradsp^\ell f(x)[Q_xu_1,\ldots,Q_xu_\ell]$ in $(\R^N)^{\otimes\ell}$, which, for $\ell=1,2$, is given by the vector and the matrix in \cref{eq:spherical-derivatives}, respectively.

We shall also use the product rules
\begin{align*}
    \gradsp(fg)&=f\gradsp g+g\gradsp f,\\
    \lapsp(fg)&=f\lapsp g+g\lapsp f+2\dotp{\gradsp f}{\gradsp g}.
\end{align*}
For any smooth $F:\SN\to\R^m$, we define the spherical Jacobian $\gradsp F(x)\in\R^{N\times m}$ as the matrix whose $j$-th column is $\gradsp F_j(x)$ for every $j\in[m]$. Equivalently, $\gradsp F(x)=Q_x\grad F(x)$, where $\grad F(x)\in\R^{N\times m}$ denotes the Euclidean Jacobian of any smooth extension of $F$, with $(i,j)$-th entry $\partial_iF_j(x)$.

We recall that ${\tfrac{d}{dt}F(\gamma(t))=\gradsp F(\gamma(t))^\top\gamma'(t)}$ for any smooth $F:\SN\to\R^m$ and any smooth curve $\gamma$ in $\SN$. Thus, integrating along a unit-speed minimizing geodesic between $x$ and $y$ and using ${d_{\SN}(x,y)\le C\norm{x-y}}$, we obtain
\begin{align}
    \norm*{F(x)-F(y)}
    \le\Bigl[\sup_{z\in\SN}\norm*{\gradsp F(z)}_{\mathrm{op}}\Bigr]d_{\SN}(x,y)
    \lesssim\Bigl[\sup_{z\in\SN}\norm*{\gradsp F(z)}_{\mathrm{op}}\Bigr]\norm{x-y}
    \qquad\forall\,x,y\in\SN.
    \label{eq:spherical-lipschitz}
\end{align}

We use the spherical integration by parts identity~\citep[Section~2.2]{bakry2013analysis}, which holds for any smooth $f,g:\SN\to\R$:
\begin{align}
    \E_\pi\bigl[\dotp{\gradsp f}{\gradsp g}\bigr]=-\E_\pi[f\lapsp g].
    \label{eq:spherical-ibp}
\end{align}
The following lemma is a spherical analogue of the classical Stein's lemma for Gaussian random variables~\citep{chatterjee2014survey} and will be used often in our proofs.
\begin{lemma}[Stein's Lemma for the sphere]
\label{lem:spherical-stein}
For any smooth $f:\SN\to\R$,
\begin{align*}
    \E_\pi[\gradsp f(x)]&=\kappa_N\E_\pi[xf(x)],\\
    \E_\pi[\gradsp^2f(x)]
    &=\frac{1}{2}\E_\pi\left[x\gradsp f(x)^\top+\gradsp f(x)x^\top\right]
    =\E_\pi\left[(xx^\top-I)f(x)\right].
\end{align*}
\end{lemma}

\begin{proof}
Applying \cref{eq:spherical-derivatives} to the coordinate functions, we obtain $\gradsp x=Q_x$ and $\lapsp x=-\kappa_Nx$, where $\lapsp$ acts coordinatewise. Fix an $i\in[N]$. Since $\gradsp f(x)\in T_x\SN$ and $\gradsp x_i=Q_xe_i$, we have $(\gradsp f)_i=\dotp{\gradsp f}{\gradsp x_i}$.

Applying \cref{eq:spherical-ibp} to $f$ and $x_i$ and using $\lapsp x_i=-\kappa_Nx_i$, we obtain the following
\begin{align*}
    \E_\pi[(\gradsp f)_i]
    =\E_\pi\bigl[\dotp{\gradsp f}{\gradsp x_i}\bigr]
    =-\E_\pi[f\lapsp x_i]
    =\kappa_N\E_\pi[x_if(x)],
\end{align*}
which proves the first identity coordinatewise.

Now, consider any smooth $F:\SN\to\R^N$. Applying the first identity to each coordinate of $F$, we obtain
\begin{align}
    \E_\pi[\gradsp F]=\kappa_N\E_\pi[xF^\top].
    \label{eq:spherical-stein-jacobian}
\end{align}
Applying \cref{eq:spherical-stein-jacobian} to $F(x)=xf(x)$ and expanding via the product rule and $\gradsp x=Q_x$, we obtain
\begin{align*}
    \kappa_N\E_\pi[xx^\top f(x)]
    =\E_\pi\bigl[\gradsp f\,x^\top+fQ_x\bigr].
\end{align*}
Rearranging terms and using $Q_x=I-\tfrac{xx^\top}{N}$ and $\kappa_N=1-\tfrac{1}{N}$, we obtain
\begin{align*}
    \E_\pi\bigl[\gradsp f\,x^\top\bigr]
    =\kappa_N\E_\pi[xx^\top f(x)]-\E_\pi\left[\left(I-\frac{xx^\top}{N}\right)f(x)\right]
    =\E_\pi[(xx^\top-I)f(x)].
\end{align*}
Since the right-hand side is symmetric, we conclude that
\begin{align}
    \E_\pi[(xx^\top-I)f(x)]
    =\frac{1}{2}\E_\pi\left[x\gradsp f(x)^\top+\gradsp f(x)x^\top\right].
    \label{eq:spherical-stein-outer-product}
\end{align}
Now, let $F=\gradsp f$ and ${\alpha(x)=\tfrac{1}{N}\dotp{x}{\grad f(x)}}$, such that $\grad f-\alpha x$ is a smooth extension of $F$. It follows that
\begin{align}
    \gradsp F
    &=Q_x\bigl[\grad^2f-(\grad\alpha)x^\top-\alpha I\bigr]
    =Q_x\Bigl[\grad^2f-\frac{1}{N}\bigl(\grad f+(\grad^2f)x\bigr)x^\top-\alpha I\Bigr]\nonumber\\*
    &=Q_x\grad^2f\,Q_x-\frac{1}{N}\dotp{x}{\grad f}Q_x-\frac{1}{N}\gradsp f\,x^\top
    =\gradsp^2f-\frac{1}{N}\gradsp f\,x^\top.
    \label{eq:gradient-jacobian}
\end{align}
Applying \cref{eq:spherical-stein-jacobian} to $F=\gradsp f$ and using \cref{eq:gradient-jacobian}, we obtain
\begin{align*}
    \kappa_N\E_\pi\bigl[x\gradsp f^\top\bigr]
    =\E_\pi\bigl[\gradsp^2f\bigr]-\frac{1}{N}\E_\pi\bigl[\gradsp f\,x^\top\bigr].
\end{align*}
Averaging the above identity with its transpose using the symmetry of $\gradsp^2f$, we obtain the following after rearranging terms
\begin{align*}
    \E_\pi\bigl[\gradsp^2f\bigr]
    =\frac{1}{2}\left(\kappa_N+\frac{1}{N}\right)\E_\pi\bigl[\gradsp f\,x^\top+x\gradsp f^\top\bigr]
    =\frac{1}{2}\E_\pi\bigl[\gradsp f\,x^\top+x\gradsp f^\top\bigr].
\end{align*}
Together with \cref{eq:spherical-stein-outer-product}, this proves the second identity.
\end{proof}

We use $O(N,\R)$ to denote the real orthogonal group, i.e., the group of real unitary matrices. Recall that each $R\in O(N,\R)$ maps $\SN$ bijectively onto itself and preserves $d_{\SN}$. In fact, $O(N,\R)$ is the isometry group of $(\SN,d_{\SN})$. Consequently, $\pi$ is invariant under the action of $O(N,\R)$. In particular, $\pi$ is invariant under $x\mapsto-x$, and the law of $\dotp{w}{x}$ under $\pi$ does not depend on $w\in\bS^{N-1}$. The following lemma characterizes the structure of matrices that commute with $O(N,\R)$.

\begin{lemma}[Commutants of $O(N,\R)$]
\label{lem:commutants-orthogonal}
Any matrix that commutes with every element of $O(N,\R)$ must be a scalar multiple of the identity.
\end{lemma}

\begin{proof}
Let $A\in\R^{N\times N}$ satisfy $AM=MA$ for every $M\in O(N,\R)$. Consider distinct $i,j\in[N]$, and let $D$ be a diagonal sign matrix with $D_{ii}=+1$ and $D_{jj}=-1$. Since $D\in O(N,\R)$, we have $AD=DA$, so
\begin{align*}
    (AD)_{ij}=A_{ij}D_{jj}=-A_{ij},
    \qquad
    (DA)_{ij}=D_{ii}A_{ij}=A_{ij}.
\end{align*}
It follows that $A_{ij}=0$ for $i\ne j$, so $A=\diag(a)$ for some $a\in\R^N$. Now let $P$ be the permutation matrix that exchanges coordinates $i$ and $j$. Since $P\in O(N,\R)$, we have $A=PAP^\top$, and therefore $a_i=a_j$. Thus $a_1=\cdots=a_N$ and $A=cI$ for some $c\in\R$.
\end{proof}

We shall also make use of the following result on the moments and symmetry of $\pi$.
\begin{lemma}[Moments and symmetry of the uniform measure on $\SN$]
\label{lem:sphere-moments}
Let $N\ge2$, $x\sim\pi$, and $w\in\bS^{N-1}$. For every integer $\ell\ge1$, we have
\begin{align*}
    \E_\pi\left[\left(\frac{\dotp{w}{x}}{\sqrt N}\right)^\ell\right]
    =\begin{cases}
        0, & \text{if $\ell$ is odd},\\
        \dfrac{(\ell-1)!!}{N(N+2)\cdots(N+\ell-2)}, & \text{if $\ell$ is even}.
    \end{cases}
\end{align*}
In particular, for fixed even $\ell$, this moment is $(\ell-1)!!\,N^{-\nicefrac{\ell}{2}}(1+O_\ell(N^{-1}))$. Moreover, for any $z\in\bS^{N-1}$ orthogonal to $w$, any measurable odd function $g:\R\to\R$, and any measurable $h:\R\to\R$ such that the expectation is finite,
\begin{align*}
    \E_\pi\left[g\left(\frac{\dotp{w}{x}}{\sqrt N}\right)h\left(\frac{\dotp{z}{x}}{\sqrt N}\right)\right]=0.
\end{align*}
Finally, for every integer $\ell\ge0$,
\begin{align*}
    \E_\pi\left[xx^\top\left(\frac{\dotp{w}{x}}{\sqrt N}\right)^\ell\right]
    =\frac{N}{N+\ell}\,\E_\pi\left[\left(\frac{\dotp{w}{x}}{\sqrt N}\right)^\ell\right](I+\ell ww^\top).
\end{align*}
\end{lemma}

\begin{proof}
The odd moments vanish by the symmetry $x\disteq-x$. For even $\ell\ge2$, applying \cref{lem:spherical-stein} to $f(x)=\dotp{w}{x}^{\ell-1}$ and taking the inner product with $w$, we obtain
\begin{align*}
    \kappa_N\E_\pi\bigl[\dotp{w}{x}^\ell\bigr]
    =(\ell-1)\E_\pi\bigl[\dotp{w}{x}^{\ell-2}\dotp{w}{Q_xw}\bigr]
    =(\ell-1)\E_\pi\left[\dotp{w}{x}^{\ell-2}\left(1-\frac{\dotp{w}{x}^2}{N}\right)\right].
\end{align*}
Rearranging and using $N\kappa_N=N-1$, we obtain
\begin{align}
    \E_\pi\left[\left(\frac{\dotp{w}{x}}{\sqrt N}\right)^\ell\right]
    =\frac{\ell-1}{N+\ell-2}\,\E_\pi\left[\left(\frac{\dotp{w}{x}}{\sqrt N}\right)^{\ell-2}\right].
    \label{eq:sphere-moment-recursion}
\end{align}
Iterating from the zeroth moment proves the first claim.

For the second claim, observe that $\pi$ is invariant under the reflection $I-2ww^\top$, which maps $\dotp{w}{x}$ to $-\dotp{w}{x}$ and leaves $\dotp{z}{x}$ unchanged. Since $g$ is odd, the expectation is therefore equal to its negative, and thus vanishes.

For the last claim, let $s=\dotp{w}{x}/\sqrt N$ and $F(x)=s^\ell x$. Since $\gradsp s=Q_xw/\sqrt N$ and $Q_xw=w-\tfrac{s}{\sqrt N}x$, the product rule and $\gradsp x=Q_x$ yield
\begin{align}
    \gradsp F
    =\frac{\ell}{\sqrt N}s^{\ell-1}(Q_xw)x^\top+s^\ell Q_x
    =\frac{\ell}{\sqrt N}s^{\ell-1}wx^\top+s^\ell I-\frac{\ell+1}{N}s^\ell xx^\top,
    \label{eq:sphere-weighted-jacobian}
\end{align}
where the terms with the factor $\ell$ are absent when $\ell=0$. Moreover, for $\ell\ge1$, the second claim with the roles of $w$ and $z$ interchanged yields $\E_\pi[s^{\ell-1}\dotp{z}{x}]=0$ for every $z\in\bS^{N-1}$ orthogonal to $w$, and hence
\begin{align}
    \E_\pi\bigl[s^{\ell-1}x\bigr]
    =\E_\pi\bigl[s^{\ell-1}\dotp{w}{x}\bigr]\,w
    =\sqrt N\,\E_\pi\bigl[s^\ell\bigr]\,w.
    \label{eq:sphere-weighted-first-moment}
\end{align}
Applying Stein's lemma (\cref{lem:spherical-stein}) to each coordinate of $F$, we have
\begin{align}
    \E_\pi\bigl[\gradsp F\bigr]
    =\kappa_N\E_\pi\bigl[xF^\top\bigr]
    =\kappa_N\E_\pi\bigl[s^\ell xx^\top\bigr].
    \label{eq:sphere-weighted-stein}
\end{align}
Taking expectations in \cref{eq:sphere-weighted-jacobian} and using \cref{eq:sphere-weighted-first-moment,eq:sphere-weighted-stein}, we conclude that
\begin{align*}
    \kappa_N\E_\pi\bigl[s^\ell xx^\top\bigr]
    =\E_\pi[s^\ell](I+\ell ww^\top)-\frac{\ell+1}{N}\E_\pi\bigl[s^\ell xx^\top\bigr].
\end{align*}
Since $\kappa_N+\tfrac{\ell+1}{N}=\tfrac{N+\ell}{N}$, we obtain the last claim by rearranging terms.
\end{proof}

Lastly, we shall make use of the following elementary estimate for the normalization of nonzero vectors.
\begin{lemma}[Local Lipschitzness of normalization]
\label{lem:normalization-lipschitz}
For any nonzero $a,b\in\R^N$,
\begin{align*}
    \norm{\frac{a}{\norm{a}}-\frac{b}{\norm{b}}}\le\frac{2\norm{a-b}}{\max\{\norm{a},\norm{b}\}}.
\end{align*}
\end{lemma}

\begin{proof}
We write
\begin{align*}
    \frac{a}{\norm{a}}-\frac{b}{\norm{b}}
    =\frac{a-b}{\norm{a}}+b\left(\frac{1}{\norm{a}}-\frac{1}{\norm{b}}\right).
\end{align*}
Taking norms and using $|\norm{a}-\norm{b}|\le\norm{a-b}$, we obtain
\begin{align*}
    \norm{\frac{a}{\norm{a}}-\frac{b}{\norm{b}}}
    \le\frac{\norm{a-b}}{\norm{a}}+\frac{|\norm{b}-\norm{a}|}{\norm{a}}
    \le\frac{2\norm{a-b}}{\norm{a}}.
\end{align*}
Interchanging $a$ and $b$ proves the claim.
\end{proof}

\subsection{Sobolev Norms for Spherical Spin Glass Hamiltonians}
\label{subsec:spherical-sobolev-norms}
We now recall the Sobolev $\cG$-norm for smooth functions on $\SN$, introduced by Ben Arous, Gheissari, and Jagannath~\citep{ben2020bounding,arous2020algorithmic}.

\begin{definition}[Sobolev $\cG$-norm]
\label{def:spherical-sobolev-norm}
Fix an integer $n\geq0$. For any smooth $f:\SN\to\R$, its Sobolev $\cG^n$ norm is defined as follows~\citep[Definition~3.1]{arous2020algorithmic}:
\begin{align*}
    \norm{f}_{\cG^n}\doteq\sum_{\ell=0}^n N^{\nicefrac{\ell}{2}}\max_{x\in\SN}\norm*{\gradsp^\ell f(x)}_{\mathrm{op}},
\end{align*}
where $\norm*{\gradsp^0 f(x)}_{\mathrm{op}}=|f(x)|$ by convention.
\end{definition}

The following ladder relations~\citep[cf.][Lemma~3.5]{arous2020algorithmic} are used to control the Sobolev norms of the action of differential operators.
\begin{lemma}[Ladder relations]
\label{lem:ladder-relations}
For every integer $n\geq0$, there exists $C_n>0$ such that, for every $N$ and all smooth $f,g:\SN\to\R$,
\begin{align*}
    \norm{\lapsp f}_{\cG^n}\le\norm{f}_{\cG^{n+2}},
    \qquad
    \norm{\dotp{\gradsp g}{\gradsp f}}_{\cG^n}\le\frac{C_n}{N}\norm{g}_{\cG^{n+1}}\norm{f}_{\cG^{n+1}}.
\end{align*}
Moreover, fix integers $n\geq0$ and $\ell\geq1$ and constants $c_1,c_2>0$, let $g_1,g_2:\SN\to\R$ be smooth functions satisfying $\norm{g_1}_{\cG^{n+2\ell}}\le c_1N$ and $\norm{g_2}_{\cG^1}\le c_2N$, and define the operators $L\doteq\tfrac{1}{2}\lapsp+\dotp{\gradsp g_1}{\gradsp\,\cdot\,}$ and $A\doteq\dotp{\gradsp g_2}{\gradsp\,\cdot\,}$. There exists $C>0$, depending only on $n$, $\ell$, $c_1$, and $c_2$, such that, for every $N$ and all smooth $f:\SN\to\R$,
\begin{align*}
    \norm*{L^\ell f}_{\cG^n}\le C\norm{f}_{\cG^{n+2\ell}},
    \qquad
    \max_{x\in\SN}\abs{AL^\ell f(x)}\le C\norm{f}_{\cG^{2\ell+1}}.
\end{align*}
\end{lemma}

The following lemma controls the Sobolev norms of the spherical pure $k$-spin Hamiltonian~\parencites{ben2020bounding}[Theorem~3.3]{arous2020algorithmic}.
\begin{lemma}[Sobolev norm concentration for the spherical pure $k$-spin Hamiltonian]
\label{lem:kspin-sobolev-norm-concentration}
Let $H_0$ be the spherical pure $k$-spin Hamiltonian defined in \cref{eq:tensor-pca-overlap-noise}. For any integer $n\geq0$, there exist $C_{k,n},c_{k,n}>0$ such that, for every $r>0$,
\begin{align*}
    \Prb\bigl[\norm{H_0}_{\cG^n}\ge C_{k,n}N+r\bigr]
    \le C_{k,n}\exp\left(-\frac{c_{k,n}r^2}{N}\right).
\end{align*}
\end{lemma}

\begin{remark}
\label{rem:kspin-laplacian-gradient-sobolev}
As noted in~\citep[Corollary~3.6]{arous2020algorithmic}, we conclude by a direct application of \cref{lem:ladder-relations,lem:kspin-sobolev-norm-concentration} that $\norm*{\lapsp H_0}_{\cG^n}\le C_{k,n}N$ and $\norm{\norm*{\gradsp H_0}^2}_{\cG^n}\le C_{k,n}N$ for any fixed integer $n\geq0$ with probability at least $1-\exp(-c_{k,n}N)$.
\end{remark}

The following lemma, which is a straightforward consequence of Eq.~(3.4) in~\citep{arous2020algorithmic}, controls the Sobolev norms of the correlation with an arbitrary unit vector and its powers.

\begin{lemma}[Sobolev norm bounds for the signal term]
\label{lem:overlap-sobolev-bounds}
For any $v\in\bS^{N-1}$, let ${g(x)\doteq\tfrac{1}{\sqrt N}\dotp{v}{x}}$. For any integer $n\geq0$, we have $\norm{g}_{\cG^n}\le C_n$ and $\norm*{g^k}_{\cG^n}\le C_{k,n}$. In particular, these bounds hold for the correlation function $m$ defined in \cref{eq:tensor-pca-overlap-noise}.
\end{lemma}

\begin{remark}
\label{rem:tensor-pca-sobolev-bound}
By \cref{lem:kspin-sobolev-norm-concentration} with $r=N$ and \cref{lem:overlap-sobolev-bounds}, the tensor PCA Hamiltonian satisfies $\norm{H}_{\cG^n}\le C_{k,n}N(1+\lambda)$ for any fixed integer $n\geq0$ with probability at least $1-C_{k,n}\exp(-c_{k,n}N)$. In particular, on this event, \cref{def:spherical-sobolev-norm} yields $\max_{x\in\SN}\norm{\gradsp H(x)}\le C_{k,n}(1+\lambda)\sqrt N$ for $n\geq1$ and $\max_{x\in\SN}\norm*{\gradsp^2H(x)}_{\mathrm{op}}\le C_{k,n}(1+\lambda)$ for $n\geq2$.
\end{remark}

\subsection{Spherical Langevin Dynamics}
\label{subsec:spherical-langevin}

We use $(B_t)_{t\ge0}$ to denote the standard Brownian motion on $\R^N$, i.e., the stochastic process with continuous sample paths and independent increments satisfying $B_0=0$ and $B_t-B_s\sim\cN(0,(t-s)I)$ for any $0\le s<t$~\citep[cf.][Section~4.1]{stroock2007multidimensional}. Given a smooth Hamiltonian $V:\SN\to\R$ and an inverse temperature $\beta\ge0$, spherical Langevin dynamics with initial point $Z_0\in\SN$ is the unique strong solution $(Z_t)_{t\ge0}$ of the stochastic differential equation (SDE)
\begin{align}
    dZ_t=\Bigl(\beta\cdot Q_{Z_t}\grad V(Z_t)-\frac{N-1}{2N}Z_t\Bigr)\,dt+Q_{Z_t}\,dB_t,
    \label{eq:spherical-langevin-sde}
\end{align}
where the driving Brownian motion $(B_t)_{t\ge0}$ is independent of the initial point and of any randomness in $V$. When $V$ is random, all statements in this section are understood conditionally on $V$. Observe that the trajectory $(Z_t)_{t\ge0}$ is a deterministic function of the initial point and the driving Brownian motion.

Since $Q_xx=0$ and $\Tr(Q_x)=N-1$, It\^o's formula~\citep[Section~4.4]{stroock2007multidimensional} yields $d\norm{Z_t}^2=0$, and therefore $Z_t\in\SN$ for all $t\ge0$. The resulting trajectory is a reversible Markov process whose stationary distribution is the Gibbs measure $\mu_{V,\beta}(dx)\propto\exp(2\beta V(x))\,\pi(dx)$~\citep[cf.][p.~187]{bakry1985diffusions}. For $\beta=0$, the trajectory reduces to spherical Brownian motion on $\SN$.

The associated Markov semigroup $P_t^{(V,\beta)}$ describes the evolution of scalar observables along the trajectory. Specifically, for any bounded measurable $f:\SN\to\R$,
\begin{align*}
    P_t^{(V,\beta)}f(x)=\E[f(Z_t)\mid Z_0=x].
\end{align*}
The generator $\cL_{V,\beta}$, which characterizes the infinitesimal action of the semigroup, is given as follows for any smooth $f:\SN\to\R$:
\begin{align}
    \cL_{V,\beta}f(x)
    =\lim_{t\to0^+}\frac{P_t^{(V,\beta)}f(x)-f(x)}{t}
    =\frac{1}{2}\lapsp f(x)+\beta\dotp{\gradsp V(x)}{\gradsp f(x)},
    \label{eq:langevin-generator}
\end{align}
where the second equality follows from It\^o's formula applied to any smooth extension of $f$, together with \cref{eq:spherical-derivatives} and $\Tr(Q_x)=N-1$. We refer to $\tfrac12\lapsp f(x)$ and $\beta\dotp{\gradsp V(x)}{\gradsp f(x)}$ as the \emph{diffusion term} and the \emph{drift term} of the generator, respectively.

The Markov semigroup and its generator also satisfy the following Kolmogorov equations for every smooth $f:\SN\to\R$~\citep[cf.][Section~1.4]{bakry2013analysis}:
\begin{align}
    \frac{d}{dt}P_t^{(V,\beta)}f
    =\cL_{V,\beta}P_t^{(V,\beta)}f
    =P_t^{(V,\beta)}\cL_{V,\beta}f.
    \label{eq:langevin-kolmogorov}
\end{align}
Intuitively, for any smooth observable $f:\SN\to\R$, $\cL_{V,\beta}f(Z_t)$ represents the instantaneous expected rate of change of $f(Z_t)$ along the trajectory. This interpretation is made rigorous via Dynkin's formula~\citep[cf.][Section~4.4]{stroock2007multidimensional}.

\begin{lemma}[Dynkin's formula for spherical Langevin dynamics]
\label{lem:langevin-dynkin}
Let $(Z_t)_{t\ge0}$ denote the trajectory of spherical Langevin dynamics for the Hamiltonian $V$ at inverse temperature $\beta$, with driving Brownian motion $(B_t)_{t\ge0}$. For any smooth $f:\SN\to\R$, the process
\begin{align*}
    M_t^f=f(Z_t)-f(Z_0)-\int_0^t\cL_{V,\beta}f(Z_s)\,ds
    =\int_0^t\dotp{\gradsp f(Z_s)}{dB_s}
\end{align*}
is a continuous martingale with $M_0^f=0$ that is square integrable on every finite time interval, known as the Dynkin martingale. Its quadratic variation is
\begin{align*}
    \langle M^f\rangle_t=2\int_0^t\Gamma(f,f)(Z_s)\,ds,
\end{align*}
where $\Gamma$ denotes the carr\'e du champ operator associated with $\cL_{V,\beta}$, defined by
\begin{align*}
    \Gamma(f,g)(x)
    \doteq\frac12\bigl[\cL_{V,\beta}(fg)-f\cL_{V,\beta}g-g\cL_{V,\beta}f\bigr](x)
    =\frac12\dotp{\gradsp f(x)}{\gradsp g(x)}.
\end{align*}
In particular, $\Gamma$ does not depend on $V$ or $\beta$.
\end{lemma}
The following uniform concentration guarantee for Dynkin martingales shall be used often in our proofs.

\begin{lemma}[Concentration of Dynkin martingales]
\label{lem:dynkin-martingale-concentration}
Let $(Z_t)_{t\ge0}$ denote the trajectory of spherical Langevin dynamics for the Hamiltonian $V$ at inverse temperature $\beta$. Suppose that a smooth $f:\SN\to\R$ satisfies $\max_{x\in\SN}\norm*{\gradsp f(x)}\le A$ for some $A>0$, and let $M_t^f$ denote its Dynkin martingale. For every $T,r>0$, we have
\begin{align*}
    \Prb\left[\sup_{0\le t\le T}\abs{M_t^f}\ge r\right]
    \le2\exp\left(-\frac{r^2}{2A^2T}\right).
\end{align*}
\end{lemma}

\begin{proof}
Since $2\Gamma(f,f)=\norm*{\gradsp f}^2\le A^2$, we infer from \cref{lem:langevin-dynkin} that $M^f$ is a continuous martingale with $M_0^f=0$ and $\langle M^f\rangle_t\le A^2T$ for every $t\le T$. Applying a line-crossing inequality for continuous martingales~\citep[cf.][Corollary~11(a)]{howard2020time} to $M^f$ and $-M^f$, we infer the following for any $a, b > 0$
\begin{align*}
    \Prb\bigl[\exists\,t\ge0:\pm M_t^f\ge a+b\langle M^f\rangle_t\bigr]\le e^{-2ab}.
\end{align*}
If $\sup_{0\le t\le T}\abs{M_t^f}\ge r$, then, by continuity, $\abs{M_t^f}\ge r$ for some $t\le T$, and one of these two events occurs with $a=\tfrac{r}{2}$ and $b=\tfrac{r}{2A^2T}$, since $a+b\langle M^f\rangle_t\le r$. The proof is concluded by a union bound, since $2ab=\tfrac{r^2}{2A^2T}$.
\end{proof}

\begin{remark*}
By \cref{def:spherical-sobolev-norm}, $\max_{x\in\SN}\norm*{\gradsp f(x)}\le N^{-1/2}\norm{f}_{\cG^1}$. Hence \cref{lem:dynkin-martingale-concentration} may always be applied with $A=N^{-1/2}\norm{f}_{\cG^1}$.
\end{remark*}

The Markov semigroup $P_t^{(V,\beta)}$ and the generator $\cL_{V,\beta}$ act entrywise on vector and matrix-valued observables. Consequently, the corresponding vector (resp. matrix) valued Dynkin martingale is obtained by applying \cref{lem:langevin-dynkin} to each entry, and we obtain the following uniform concentration guarantee for operator norms of rectangular matrix valued Dynkin martingales via a direct application of~\cref{lem:dynkin-martingale-concentration}.

\begin{lemma}[Concentration of matrix-valued Dynkin martingales]
\label{lem:matrix-dynkin-uniform-concentration}
Let $m,n\ge1$, and let $F:\SN\to\R^{m\times n}$ be a smooth function satisfying $\max_{x\in\SN}\norm*{\gradsp(u^\top Fv)(x)}\le A$ for every $u\in\bS^{m-1}$, $v\in\bS^{n-1}$ and some $A>0$. Let $(Z_t)_{t\ge0}$ denote the trajectory of spherical Langevin dynamics for the Hamiltonian $V$ at inverse temperature $\beta$, and let $M_t^F$ denote the associated Dynkin martingale. For every $T>0$ and $\delta\in(0,1)$, the following holds with probability at least $1-\delta$:
\begin{align*}
    \sup_{0\le t\le T}\norm*{M_t^F}_{\mathrm{op}}
    \le CA\sqrt{T\bigl(m+n+\log(\nicefrac{1}{\delta})\bigr)},
\end{align*}
where $C>0$ is a universal constant. In particular, taking $n=1$ yields the same guarantee for vector-valued observables, with the Euclidean norm in place of the operator norm.
\end{lemma}

\begin{proof}
Fix $u\in\bS^{m-1}$ and $v\in\bS^{n-1}$, and set $f_{u,v}(x)=u^\top F(x)v$. By linearity of the generator, $u^\top M_t^Fv=M_t^{f_{u,v}}$. Since $\max_{x\in\SN}\norm*{\gradsp f_{u,v}(x)}\le A$, we obtain from \cref{lem:dynkin-martingale-concentration} with $r=A\sqrt{2Ts}$ that, for every $s>0$,
\begin{align*}
    \Prb\left[\sup_{0\le t\le T}\abs{u^\top M_t^Fv}\ge A\sqrt{2Ts}\right]
    \le2e^{-s}.
\end{align*}
Let $\cN_m$ and $\cN_n$ be fixed $\nicefrac{1}{4}$-nets of $\bS^{m-1}$ and $\bS^{n-1}$ with $|\cN_m|\le9^m$ and $|\cN_n|\le9^n$. A standard $\eps$-net comparison implies that, for every $t\ge0$,
\begin{align*}
    \norm*{M_t^F}_{\mathrm{op}}
    \le2\max_{u\in\cN_m,\,v\in\cN_n}\abs{u^\top M_t^Fv}.
\end{align*}
The proof is concluded by taking $s=(m+n)\log9+\log(\nicefrac{2}{\delta})$, applying a union bound over $\cN_m\times\cN_n$, and adjusting constants appropriately.
\end{proof}

\begin{remark}
\label{rem:symmetric-matrix-dynkin-concentration}
If $m=n$ and $F$ takes values in $\Sym_n$, it suffices to assume that, for every $w\in\bS^{n-1}$, $\max_{x\in\SN}\norm*{\gradsp(w^\top Fw)(x)}\le A$.
\end{remark}

Finally, the following change of measure theorem for spherical Langevin dynamics is a direct consequence of the Girsanov theorem~\citep[Section~6.4]{stroock2007multidimensional}.

\begin{lemma}[Girsanov's Theorem for Spherical Langevin Dynamics]
\label{lem:langevin-change-measure}
Fix $T>0$, and let $(Z_t)_{0\le t\le T}$ denote the trajectory of spherical Langevin dynamics for the Hamiltonian $V$ at inverse temperature $\beta$, with driving Brownian motion $(B_t)_{0\le t\le T}$. Let $\Q^T$ denote the joint law of $Z_0$ and $(B_t)_{0\le t\le T}$, and let $(\theta_t)_{0\le t\le T}$ be an $\R^N$-valued process that is progressively measurable with respect to the joint filtration of the initial point and the driving Brownian motion, satisfying $\norm{\theta_t}\le C$ for a deterministic $C<\infty$ and all $t\le T$. For $0\le t\le T$, define
\begin{align*}
    L_t^\theta
    =\exp\left(\int_0^t\dotp{\theta_s}{dB_s}-\frac{1}{2}\int_0^t\norm{\theta_s}^2\,ds\right).
\end{align*}
Then $(L_t^\theta)_{0\le t\le T}$ is a mean-one martingale under $\Q^T$. Define the tilted law $\Q^{T,\theta}$ by $\frac{d\Q^{T,\theta}}{d\Q^T}=L_T^\theta$. Under $\Q^{T,\theta}$, the process $\widetilde B_t\doteq B_t-\int_0^t\theta_s\,ds$ is a standard Brownian motion on $[0,T]$ independent of $Z_0$, $Z_0$ has the same distribution as under $\Q^T$, and
\begin{align*}
    dZ_t=\Bigl(\beta\cdot Q_{Z_t}\grad V(Z_t)-\frac{N-1}{2N}Z_t+Q_{Z_t}\theta_t\Bigr)\,dt+Q_{Z_t}\,d\widetilde B_t.
\end{align*}
In particular, if $\theta_t=\beta\gradsp U(Z_t)$ for a smooth $U:\SN\to\R$, then, under $\Q^{T,\theta}$, $(Z_t)_{0\le t\le T}$ is spherical Langevin dynamics for the Hamiltonian $V+U$ at inverse temperature $\beta$, with driving Brownian motion $(\widetilde B_t)_{0\le t\le T}$.
\end{lemma}

\section{Averaged Spherical Derivatives of the Tensor PCA Hamiltonian}
\label{sec:averaged-spherical-derivatives}

In this section, we analyze the statistical properties of the average spherical gradient and Hessian of the tensor PCA Hamiltonian. In \cref{subsec:average-spherical-gradient}, we show that, for odd $k$, the average spherical gradient is a noisy observation of the planted vector. Analogously, in \cref{subsec:average-spherical-hessian}, we show that, for even $k$, the average spherical Hessian has a spiked Wigner type structure, with a rank-one spike in the planted direction.

\subsection{Analysis of the Average Spherical Gradient}
\label{subsec:average-spherical-gradient}

We now establish a signal-noise decomposition of the average spherical gradient of the tensor PCA Hamiltonian.

\begin{theorem}[Signal-Noise Decomposition of the Average Spherical Gradient]
\label{thm:avg-spherical-grad-snr}
Let $H$ denote the tensor PCA Hamiltonian defined in \cref{eq:tensor-pca-overlap-noise}, and define the average spherical gradient $u_H\doteq\E_\pi[\gradsp H(x)]$. Then, the following holds:
\begin{enumerate}
    \item If $k$ is even, then $u_H=0$ holds deterministically.
    \item If $k$ is odd, then $u_H=N\lambda\alpha_N\plant+\zeta$, where $\zeta\doteq\E_\pi[\gradsp H_0(x)]\sim\cN(0,\sigma_N^2I)$, $\alpha_N\asymp_kN^{-\nicefrac{k}{2}}$, and $\sigma_N^2=\kappa_N\sqrt N\alpha_N\asymp_kN^{-\tfrac{k-1}{2}}$.
\end{enumerate}
\end{theorem}

\begin{proof}
By Stein's lemma (\cref{lem:spherical-stein}) and \cref{eq:tensor-pca-overlap-noise}, we obtain
\begin{align}
\label{eq:average-gradient-decomposition}
    u_H=\kappa_N\E_\pi[xH(x)]
    =N\lambda\kappa_N\E_\pi[xm(x)^k]+\zeta,
    \qquad
    \zeta=\kappa_N\E_\pi[xH_0(x)].
\end{align}
If $k$ is even, then $H(-x)=H(x)$, and invariance of $\pi$ under $x\mapsto-x$ implies that $u_H=0$.

Henceforth, assume $k$ is odd. We first analyze the signal term in \cref{eq:average-gradient-decomposition}. For every $v\in\bS^{N-1}$ which is orthogonal to $\plant$, the second claim of \cref{lem:sphere-moments}, applied with $(w,z)=(v,\plant)$, $g(t)=t$, and $h(t)=t^k$, yields $\E_\pi[\dotp{v}{x}m(x)^k]=0$. Hence, $\E_\pi[xm(x)^k]$ is parallel to $\plant$. Taking the inner product with $\plant$ and using $\dotp{\plant}{x}=\sqrt N\,m(x)$, we obtain
\begin{align}
\label{eq:average-gradient-signal}
    \kappa_N\E_\pi\bigl[xm(x)^k\bigr]
    =\kappa_N\E_\pi\bigl[\dotp{\plant}{x}m(x)^k\bigr]\plant
    =\alpha_N\plant,
    \qquad
    \alpha_N\doteq\kappa_N\sqrt N\,\E_\pi\bigl[m(x)^{k+1}\bigr].
\end{align}
Since $k+1$ is even, \cref{lem:sphere-moments} implies that $\alpha_N\asymp_kN^{-\nicefrac{k}{2}}$. We conclude that the signal term in \cref{eq:average-gradient-decomposition} is $N\lambda\alpha_N\plant$.

We now analyze the noise term in \cref{eq:average-gradient-decomposition}. Since $\zeta$ is a linear function of the centered Gaussian tensor $G$, it is a centered Gaussian vector, and its law is uniquely determined by its covariance $\Sigma\doteq\E_G[\zeta\zeta^\top]$. Furthermore, since $\zeta=\kappa_N\E_\pi[xH_0(x)]$, we infer the following from \cref{eq:kspin-covariance}
\begin{align}
\label{eq:average-gradient-noise-covariance}
    \Sigma
    =\kappa_N^2\E_{x,y\,\iidsim\pi}\bigl[xy^\top\E_G[H_0(x)H_0(y)]\bigr]
    =\kappa_N^2N\,\E_{x,y\,\iidsim\pi}\left[xy^\top\left(\frac{\dotp{x}{y}}{N}\right)^k\right].
\end{align}
Fix $R\in O(N,\R)$. Since $\pi$ is invariant under $O(N,\R)$ and $\dotp{Rx}{Ry}=\dotp{x}{y}$, the change of variables $(x,y)\mapsto(Rx,Ry)$ in \cref{eq:average-gradient-noise-covariance} yields $R\Sigma R^\top=\Sigma$. Thus, $\Sigma$ commutes with every element of $O(N,\R)$, and \cref{lem:commutants-orthogonal} implies that $\Sigma=\sigma_N^2I$ for some $\sigma_N^2\ge0$.

Observe that by $O(N,\R)$ invariance of $\pi$, the following holds for any $x\in\SN$ that is independent of $y$:
\begin{align}
\label{eq:average-gradient-overlap-moment}
    \E_{y\sim\pi}\left[\left(\frac{\dotp{x}{y}}{N}\right)^{k+1}\right]
    =\E_{y\sim\pi}\bigl[m(y)^{k+1}\bigr].
\end{align}
Taking the trace in \cref{eq:average-gradient-noise-covariance} and using $\Tr(xy^\top)=\dotp{x}{y}$, we obtain the following from \cref{eq:average-gradient-signal,eq:average-gradient-overlap-moment}:
\begin{align*}
    \sigma_N^2
    &=\frac{1}{N}\Tr(\Sigma)
    =\kappa_N^2\E_{x,y\,\iidsim\pi}\left[\dotp{x}{y}\left(\frac{\dotp{x}{y}}{N}\right)^k\right]\\
    &=\kappa_N^2N\,\E_{x\sim\pi}\left[\E_{y\sim\pi}\left[\left(\frac{\dotp{x}{y}}{N}\right)^{k+1}\right]\right]
    =\kappa_N^2N\,\E_{y\sim\pi}\bigl[m(y)^{k+1}\bigr]\\
    &=\kappa_N\sqrt N\cdot\kappa_N\sqrt N\,\E_{y\sim\pi}\bigl[m(y)^{k+1}\bigr]
    =\kappa_N\sqrt N\alpha_N.
\end{align*}
Thus, we conclude that $\zeta\sim\cN(0,\sigma_N^2I)$ with $\sigma_N^2=\kappa_N\sqrt N\alpha_N\asymp_kN^{-\tfrac{k-1}{2}}$.
\end{proof}

\begin{remark}
\label{rem:average-spherical-gradient-snr}
We make the following observations from \cref{thm:avg-spherical-grad-snr}. For even $k$, the symmetry $x\mapsto-x$ forces $u_H=0$, so the average spherical gradient carries no information about the planted signal. For odd $k$, the identity $\sigma_N^2=\kappa_N\sqrt N\alpha_N$ implies that $u_H=\kappa_N^{-1}\lambda\sigma_N^2\xopt+\zeta$. Thus, $u_H$ is a noisy observation of the planted vector $\xopt$, with signal-to-noise ratio $\kappa_N^{-1}\lambda\sigma_N\asymp_k\lambda N^{-\tfrac{k-1}{4}}$. In particular, at the conjectured algorithmic threshold $\lambda\gtrsim_kN^{\tfrac{k-2}{4}}$, the SNR is at least $\Omega(N^{-\nicefrac{1}{4}})$. This is formalized in \cref{cor:avg-spherical-grad-correlation} below.
\end{remark}

\begin{corollary}[Norm and Correlation of the Average Spherical Gradient]
\label{cor:avg-spherical-grad-correlation}
Let $k\ge3$ be odd and $\lambda>0$, and define $\widehat u_H\doteq\tfrac{\sqrt N\,u_H}{\norm{u_H}}$. Then,
\begin{align*}
    \Prb\left[
        c_k\min\left\{1,\lambda N^{-\tfrac{k-1}{4}}\right\}
        \le m(\widehat u_H)
        \le C_k\min\left\{1,\lambda N^{-\tfrac{k-1}{4}}\right\}
    \right]
    \ge1-C\exp\left(-c_k\min\left\{N,\lambda^2N^{-\tfrac{k-3}{2}}\right\}\right).
\end{align*}
In particular, if $\lambda\ge C_kN^{\tfrac{k-2}{4}}$, then, for all sufficiently large $N$,
\begin{align*}
    \Prb\left[m(\widehat u_H)\ge N^{-\nicefrac{1}{4}}\right]
    \ge1-\exp(-c_k\sqrt N).
\end{align*}
Furthermore,
\begin{align*}
    \Prb\left[\norm{u_H}\ge c\bigl(N\lambda\alpha_N+\sigma_N\sqrt N\bigr)\right]
    \ge1-C\exp(-cN).
\end{align*}
\end{corollary}

\begin{proof}
By \cref{thm:avg-spherical-grad-snr}, $u_H=N\lambda\alpha_N\plant+\zeta$, where $\zeta\sim\cN(0,\sigma_N^2I)$. Let $Z\doteq\sigma_N^{-1}\dotp{\zeta}{\plant}$ and $\zeta_\perp\doteq Q_{\plant}\zeta$, so that $Z\sim\cN(0,1)$ and $\sigma_N^{-2}\norm{\zeta_\perp}^2$ is a $\chi^2$ random variable with $N-1$ degrees of freedom. Since $u_H=(N\lambda\alpha_N+\sigma_NZ)\plant+\zeta_\perp$, we have
\begin{align}
\label{eq:average-gradient-orthogonal-decomposition}
    \dotp{u_H}{\plant}=N\lambda\alpha_N+\sigma_NZ,
    \qquad
    \norm{u_H}^2=(N\lambda\alpha_N+\sigma_NZ)^2+\norm{\zeta_\perp}^2.
\end{align}
By Gaussian concentration~\citep[cf.][Theorem~3.1.1]{vershynin2018high}, the event $\cE_\perp\doteq\{\tfrac14\sigma_N^2N\le\norm{\zeta_\perp}^2\le2\sigma_N^2N\}$ holds with probability at least $1-2\exp(-cN)$.

We now prove the correlation bound. Define $s_N\doteq\tfrac{N\lambda\alpha_N}{\sigma_N\sqrt N}$. Then, the identity $\sigma_N^2=\kappa_N\sqrt N\alpha_N$, together with $\kappa_N\asymp1$ and $\sigma_N\asymp_kN^{-\tfrac{k-1}{4}}$ yields the following:
\begin{align*}
    s_N=\frac{\lambda\sigma_N}{\kappa_N}\asymp_k\lambda N^{-\tfrac{k-1}{4}}.
\end{align*}
By a standard Gaussian tail bound applied to $Z$, and a union bound, we obtain
\begin{align*}
    \Prb\left[\cE_\perp\cap\left\{\abs{Z}\le\frac12\sqrt N\min\{s_N,1\}\right\}\right]
    \ge1-2\exp(-cN)-2\exp\left(-\frac18N\min\{1,s_N^2\}\right).
\end{align*}
On the above event, we have $\sigma_N\abs{Z}\le\tfrac12\sigma_N\sqrt Ns_N=\tfrac12N\lambda\alpha_N$, and \cref{eq:average-gradient-orthogonal-decomposition} yields
\begin{align*}
    \frac12N\lambda\alpha_N
    &\le\dotp{u_H}{\plant}
    \le\frac32N\lambda\alpha_N,\\
    \frac14\bigl(N^2\lambda^2\alpha_N^2+\sigma_N^2N\bigr)
    &\le\norm{u_H}^2
    \le\frac94\bigl(N^2\lambda^2\alpha_N^2+\sigma_N^2N\bigr).
\end{align*}
In particular, $u_H\ne0$, and since $m(\widehat u_H)=\tfrac{\dotp{u_H}{\plant}}{\norm{u_H}}$, we obtain
\begin{align*}
    \frac13\cdot\frac{s_N}{\sqrt{1+s_N^2}}
    \le m(\widehat u_H)
    \le3\cdot\frac{s_N}{\sqrt{1+s_N^2}}.
\end{align*}
Since $\tfrac{s}{\sqrt{1+s^2}}\asymp\min\{1,s\}$ for every $s>0$ and $s_N\asymp_k\lambda N^{-\tfrac{k-1}{4}}$, the correlation bound follows after adjusting constants appropriately. Moreover, if $\lambda\ge C_kN^{\tfrac{k-2}{4}}$, then $\lambda N^{-\tfrac{k-1}{4}}\ge C_kN^{-\nicefrac{1}{4}}$, so the second claim follows for sufficiently large $N$ by adjusting constants appropriately.

Finally, we prove the norm bound. By a standard Gaussian tail bound, the event $\cE_\perp\cap\{Z\ge-\tfrac12\sqrt N\}$ holds with probability at least $1-Ce^{-cN}$. On this event, we obtain the following from \cref{eq:average-gradient-orthogonal-decomposition}
\begin{align*}
    \norm{u_H}
    \ge\max\bigl\{N\lambda\alpha_N+\sigma_NZ,\norm{\zeta_\perp}\bigr\}
    \ge\max\Bigl\{N\lambda\alpha_N-\frac12\sigma_N\sqrt N,\frac12\sigma_N\sqrt N\Bigr\}
    \ge\frac14\bigl(N\lambda\alpha_N+\sigma_N\sqrt N\bigr),
\end{align*}
where the last inequality holds since $N\lambda\alpha_N-\tfrac12\sigma_N\sqrt N\ge\tfrac12N\lambda\alpha_N$ if $N\lambda\alpha_N\ge\sigma_N\sqrt N$, and $\tfrac12\sigma_N\sqrt N\ge\tfrac14(N\lambda\alpha_N+\sigma_N\sqrt N)$ otherwise. This concludes the proof.
\end{proof}

\subsection{Analysis of the Average Spherical Hessian}
\label{subsec:average-spherical-hessian}

We now characterize the structure of the average spherical Hessian. In particular, we observe that for even $k$, it exhibits a spiked Wigner type structure.

\begin{theorem}[Signal-Noise Decomposition of the Average Spherical Hessian]
\label{thm:avg-spherical-hess-snr}
Let $H$ denote the tensor PCA Hamiltonian defined in \cref{eq:tensor-pca-overlap-noise}, and define the average spherical Hessian $A_H\doteq\E_\pi[\gradsp^2H(x)]$. Then, the following holds:
\begin{enumerate}
    \item If $k$ is odd, then $A_H=0$ holds deterministically.
    \item If $k$ is even, then
    \begin{align*}
        A_H=\lambda\gamma_N\left(\plant(\plant)^\top-\frac{I}{N}\right)+W_N,
    \end{align*}
    where $\gamma_N\asymp_kN^{-\tfrac{k-2}{2}}$, and $W_N\doteq\E_\pi[\gradsp^2H_0(x)]$ satisfies
    \begin{align*}
        \Prb\left[\norm{W_N}_{\mathrm{op}}\ge C_kN^{-\tfrac{k-2}{4}}\right]
        \le\exp(-cN).
    \end{align*}
\end{enumerate}
\end{theorem}

\begin{proof}
By Stein's lemma (\cref{lem:spherical-stein}) and \cref{eq:tensor-pca-overlap-noise}, we obtain
\begin{align}
\label{eq:average-hessian-decomposition}
    \begin{aligned}
        A_H
        &=\E_\pi[(xx^\top-I)H(x)]
        =A_{\mathrm{sig}}+W_N,\\
        A_{\mathrm{sig}}
        &\doteq N\lambda\E_\pi[(xx^\top-I)m(x)^k],
        \qquad
        W_N=\E_\pi[(xx^\top-I)H_0(x)].
    \end{aligned}
\end{align}
If $k$ is odd, then $H(-x)=-H(x)$ while $xx^\top-I$ is invariant under $x\mapsto-x$. Consequently, the invariance of $\pi$ under the orthogonal transformation $x\mapsto-x$ implies $A_H=0$.

Henceforth, assume $k$ is even. We first analyze the signal term in \cref{eq:average-hessian-decomposition}. Applying the last claim of \cref{lem:sphere-moments} with $w=\plant$ and $\ell=k$, we obtain
\begin{align}
\label{eq:average-hessian-signal}
    \begin{aligned}
        N\E_\pi\bigl[(xx^\top-I)m(x)^k\bigr]
        &=N\E_\pi\bigl[m(x)^k\bigr]\left[\frac{N}{N+k}\bigl(I+k\plant(\plant)^\top\bigr)-I\right]\\
        &=N\E_\pi\bigl[m(x)^k\bigr]\left[\frac{Nk}{N+k}\plant(\plant)^\top-\frac{k}{N+k}I\right]\\
        &=\gamma_N\left(\plant(\plant)^\top-\frac{I}{N}\right),
    \end{aligned}
\end{align}
where $\gamma_N$ is defined as follows:
\begin{align*}
    \gamma_N\doteq\frac{kN^2}{N+k}\E_\pi\bigl[m(x)^k\bigr].
\end{align*}
Since $k$ is even, \cref{lem:sphere-moments} implies that $\gamma_N\asymp_kN^{-\tfrac{k-2}{2}}$. We conclude that the signal term in \cref{eq:average-hessian-decomposition} is $A_{\mathrm{sig}}=\lambda\gamma_N\bigl(\plant(\plant)^\top-\tfrac{I}{N}\bigr)$.

We now analyze the noise term in \cref{eq:average-hessian-decomposition}. Since $W_N$ is a linear function of the centered Gaussian tensor $G$, the quadratic form $w^\top W_Nw$ is a centered Gaussian random variable for every $w\in\bS^{N-1}$. Fixing an arbitrary $w\in\bS^{N-1}$ and observing that $w^\top W_Nw=\E_\pi[(\dotp{w}{x}^2-1)H_0(x)]$, we infer the following from \cref{eq:kspin-covariance}:
\begin{align}
\label{eq:average-hessian-noise-variance}
    \begin{aligned}
        \E_G\bigl[(w^\top W_Nw)^2\bigr]
        &=\E_{x,y\,\iidsim\pi}\bigl[(\dotp{w}{x}^2-1)(\dotp{w}{y}^2-1)\E_G[H_0(x)H_0(y)]\bigr]\\
        &=N\,\E_{x,y\,\iidsim\pi}\left[(\dotp{w}{x}^2-1)(\dotp{w}{y}^2-1)\left(\frac{\dotp{x}{y}}{N}\right)^k\right].
    \end{aligned}
\end{align}
Observe that by $O(N,\R)$ invariance of $\pi$, \cref{eq:average-hessian-signal} implies the following for any $x\in\SN$ that is independent of $y$:
\begin{align}
\label{eq:average-hessian-overlap-moment}
    \E_{y\sim\pi}\left[(yy^\top-I)\left(\frac{\dotp{x}{y}}{N}\right)^k\right]
    =\frac{\gamma_N}{N^2}\bigl(xx^\top-I\bigr).
\end{align}
Using $\dotp{w}{y}^2-1=w^\top(yy^\top-I)w$ together with \cref{lem:sphere-moments}, we obtain the following from \cref{eq:average-hessian-noise-variance,eq:average-hessian-overlap-moment}:
\begin{align*}
    \E_G\bigl[(w^\top W_Nw)^2\bigr]
    &=N\,\E_{x\sim\pi}\left[(\dotp{w}{x}^2-1)\,w^\top\E_{y\sim\pi}\left[(yy^\top-I)\left(\frac{\dotp{x}{y}}{N}\right)^k\right]w\right]\\
    &=\frac{\gamma_N}{N}\E_\pi\bigl[(\dotp{w}{x}^2-1)^2\bigr]
    =\frac{\gamma_N}{N}\Bigl(\E_\pi\bigl[\dotp{w}{x}^4\bigr]-1\Bigr)
    =\frac{2(N-1)\gamma_N}{N(N+2)}.
\end{align*}
Thus, $w^\top W_Nw$ is a centered Gaussian random variable with variance at most $\tfrac{2\gamma_N}{N}$. By a Gaussian tail bound and a routine $\eps$-net argument on $\bS^{N-1}$~\citep[cf.][Chapter~4]{vershynin2018high}, we obtain
\begin{align*}
    \Prb\left[\norm{W_N}_{\mathrm{op}}\ge8\sqrt{\gamma_N}\right]
    \le e^{-N}.
\end{align*}
Since $\gamma_N\asymp_kN^{-\tfrac{k-2}{2}}$, the proof is concluded by adjusting constants appropriately.
\end{proof}

\begin{remark}
\label{rem:average-spherical-hessian-snr}
For even $k$, we observe from \cref{thm:avg-spherical-hess-snr} that the rank-one spike in the average spherical Hessian is precisely $\plant$. Indeed, $A_{\mathrm{sig}}$ has top eigenvector $\plant$ and eigengap $\lambda_1(A_{\mathrm{sig}})-\lambda_2(A_{\mathrm{sig}})=\lambda\gamma_N$. Relative to the noise scale $\sqrt{\gamma_N}$, the signal-to-noise ratio is therefore $\lambda\sqrt{\gamma_N}\asymp_k\lambda N^{-\tfrac{k-2}{4}}$. Hence, the Davis--Kahan theorem~\citep[Theorem~4.5.5]{vershynin2018high} suggests that a top unit eigenvector of $A_H$ achieves constant absolute correlation with $\plant$ with probability at least $1-\exp(-cN)$ whenever $\lambda$ is above the conjectured algorithmic threshold, namely when $\lambda\gtrsim N^{\tfrac{k-2}{4}}$.
\end{remark}

\section{Ergodic Averages of High Temperature Langevin Dynamics}
\label{sec:ergodic-averaging}

Our goal in this section is to establish the following recovery guarantee for ergodic averages of high temperature spherical Langevin dynamics.

\begin{theorem}[Recovery by Ergodic Averaging for Tensor PCA]
\label{thm:tpca-averaged-differential-recovery}
Let $\lambda\ge C_kN^{\tfrac{k-2}{4}}$, and let $(\tilX_t)_{t\ge0}$ denote spherical Langevin dynamics for the tensor PCA Hamiltonian $H$ at inverse temperature $\beta_0>0$, initialized at $\tilX_0\sim\pi$. Define
\begin{align*}
    \Xbar_T\doteq\frac{1}{T}\int_0^T\tilX_t\,dt,
    \qquad
    \Sigmabar_T\doteq\frac{1}{T}\int_0^T\tilX_t\tilX_t^\top\,dt,
\end{align*}
and write $u_H=\E_\pi[\gradsp H]$ and $A_H=\E_\pi[\gradsp^2H]$. For every fixed $\vartheta\in(0,1)$ and all sufficiently large $N$, the following hold whenever $0<\beta_0\lesssim_k\lambda^{-1}N^{-\lfloor k/2\rfloor-\vartheta}$ and $T_{\mathrm{av}}\gtrsim_k\beta_0^{-2}\lambda^{-2}N^{2\lfloor k/2\rfloor+2\vartheta}$, with sufficiently small and large implicit constants, respectively:
\begin{enumerate}
    \item
    If $k$ is odd, then, with probability at least $1-\exp(-c_kN)$, we have $\Xbar_{T_{\mathrm{av}}}\ne0$ and
    \begin{align*}
        \frac{1}{N}\norm{\widehat X_{T_{\mathrm{av}}}-\widehat u_H}^2
        \le N^{-2\vartheta}\min\{1,\lambda^2\sigma_N^2\},
        \qquad
        \widehat X_{T_{\mathrm{av}}}\doteq\frac{\sqrt N\,\Xbar_{T_{\mathrm{av}}}}{\norm{\Xbar_{T_{\mathrm{av}}}}},
    \end{align*}
    where $\widehat u_H$ is as defined in \cref{cor:avg-spherical-grad-correlation} and $\sigma_N\asymp_kN^{-\tfrac{k-1}{4}}$ is as defined in \cref{thm:avg-spherical-grad-snr}.

    \item If $k$ is even and $\widehat v$ is a top unit eigenvector of $\Sigmabar_{T_{\mathrm{av}}}$, then, with probability at least $1-\exp(-c_kN)$,
    \begin{align*}
        \bigl|m(\sqrt N\,\widehat v)\bigr|
        \ge 1-C_k\left(\frac{N^{\tfrac{k-2}{2}}}{\lambda^2}+N^{-2\vartheta}\right).
    \end{align*}
\end{enumerate}
\end{theorem}

\Cref{thm:tpca-averaged-differential-recovery}, which is proved in \cref{subsec:tpca-recovery-proof}, is obtained by specializing our results on ergodic averages of spherical Langevin dynamics for smooth potentials to tensor PCA. These results, presented in \cref{subsec:recovering-average-spherical-gradient,subsec:recovering-average-spherical-hessian}, extend the framework of Part~I~\citep{das2026averaged} from Glauber dynamics on the hypercube to spherical Langevin dynamics: in both settings, we view the high temperature dynamics as a perturbation of its infinite temperature counterpart, and control the ergodic averages via their associated Green operators, as elucidated in \cref{subsec:bm-functional-analysis}. We believe that these results and their proof techniques, which require minor technical modifications for extending from the hypercube to the sphere, may be of independent interest.

\subsection{Functional Analysis of Spherical Brownian Motion}
\label{subsec:bm-functional-analysis}

We use $(S_t)_{t\ge0}$ to denote spherical Brownian motion on $\SN$, i.e., the unique strong solution of the SDE
\begin{align*}
    dS_t=-\frac{N-1}{2N}S_t\,dt+Q_{S_t}\,dB_t,
    \qquad
    S_0\in\SN,
\end{align*}
which corresponds to spherical Langevin dynamics with $\beta=0$. As noted in \cref{subsec:spherical-langevin}, $S_t\in\SN$ for all $t\ge0$, and $(S_t)_{t\ge0}$ is a reversible Markov process whose stationary distribution is $\pi$. We use $(P_t)_{t\ge0}$ to denote its Markov semigroup, i.e., $P_tf(x)=\E[f(S_t)\mid S_0=x]$ for any bounded measurable $f:\SN\to\R$. By \cref{eq:langevin-generator,lem:langevin-dynkin}, the associated generator $\cL$ and carr\'e du champ operator $\Gamma$ satisfy the following for any smooth $f,g:\SN\to\R$:
\begin{align*}
    \cL f=\frac12\lapsp f,
    \qquad
    \Gamma(f,g)=\frac12\dotp{\gradsp f}{\gradsp g}.
\end{align*}

An essential ingredient in our proofs is the following gradient estimate for the spherical Brownian semigroup, which is a corollary of the celebrated Bakry--\'Emery estimate for Markov diffusion semigroups on manifolds satisfying a curvature-dimension condition~\citep[Section~3.2.3]{bakry2013analysis}. Throughout, we write $\rho_N\doteq\tfrac{N-2}{2N}$, and note that $\rho_N\ge\tfrac16$ since $N\ge3$.

\begin{lemma}[Bakry--\'Emery Gradient Estimate for Spherical Brownian Motion]
\label{lem:sphere-bakry-emery}
The following holds for every smooth $f:\SN\to\R^m$, $x\in\SN$, and $t\ge0$:
\begin{align*}
    \norm{\gradsp(P_tf)(x)}_{\mathrm{op}}\le e^{-\rho_Nt}P_t\bigl(\norm{\gradsp f}_{\mathrm{op}}\bigr)(x).
\end{align*}
Moreover, if $\max_{x\in\SN}\norm{\gradsp f(x)}_{\mathrm{op}}\le L$, then
\begin{align*}
    \max_{x\in\SN}\norm{\gradsp(P_tf)(x)}_{\mathrm{op}}\le e^{-\rho_Nt}L
    \qquad\forall\ t\ge0.
\end{align*}
\end{lemma}

\begin{proof}
Since $\SN$ is a sphere of dimension $N-1$ and radius $\sqrt N$, it has constant positive Ricci curvature $\tfrac{N-2}{N}$~\citep[cf.][Section~2.2.4]{bakry2013analysis}. Consequently, $\cL=\tfrac12\lapsp$ satisfies the Bakry--\'Emery curvature-dimension condition $\mathrm{CD}(\rho_N,\infty)$~\citep[cf.][Section~C.6]{bakry2013analysis}. Fix $u\in\bS^{m-1}$ and set $g(x)=\dotp{f(x)}{u}$. Since $P_t$ acts entrywise, $P_tg=\dotp{P_tf}{u}$, and hence $\gradsp(P_tg)=\gradsp(P_tf)\,u$ and $\gradsp g=(\gradsp f)\,u$. By the strong gradient bound~\citep[Theorem~3.2.4]{bakry2013analysis} applied to $g$ and the fact that $\Gamma(g,g)=\tfrac12\norm{\gradsp g}^2$, we conclude that
\begin{align*}
    \norm{\gradsp(P_tf)(x)\,u}
    =\norm{\gradsp(P_tg)(x)}
    \le e^{-\rho_Nt}P_t\bigl(\norm{(\gradsp f)\,u}\bigr)(x)
    \le e^{-\rho_Nt}P_t\bigl(\norm{\gradsp f}_{\mathrm{op}}\bigr)(x).
\end{align*}
Since $u\in\bS^{m-1}$ was arbitrary, this proves the first claim.

Finally, suppose $\norm{\gradsp f(z)}_{\mathrm{op}}\le L$ for every $z\in\SN$. Since $P_th(x)=\E[h(S_t)\mid S_0=x]$ for any bounded measurable $h:\SN\to\R$, the first claim implies the following for every $x\in\SN$ and $t\ge0$:
\begin{align*}
    \norm{\gradsp(P_tf)(x)}_{\mathrm{op}}
    \le e^{-\rho_Nt}\E\bigl[\norm{\gradsp f(S_t)}_{\mathrm{op}}\bigm|S_0=x\bigr]
    \le e^{-\rho_Nt}L,
\end{align*}
which proves the second claim.
\end{proof}

We now introduce the Green operator (or restricted inverse of $-\cL$) for spherical Brownian motion. Let $L^2(\pi)\doteq\{f:\SN\to\R:\E_\pi[\abs{f}^2]<\infty\}$ denote the Hilbert space of square integrable functions on $\SN$, and let $L^2_0(\pi)\doteq\{f\in L^2(\pi):\E_\pi[f]=0\}$ denote the subspace of zero mean functions on $\SN$. The Green operator $\cR:L^2_0(\pi)\to L^2_0(\pi)$ is then defined as follows:
\begin{align}
\label{eq:bm-green-operator}
    \cR f=\int_0^\infty P_tf\,dt
    \qquad\forall\ f\in L^2_0(\pi).
\end{align}
For centered vector and matrix valued observables, $\cR$ acts entrywise. The following lemma shows that $\cR$ is a well-defined bounded linear operator and establishes several properties that will be vital to our analysis.

\begin{lemma}[Properties of the Green operator]
\label{lem:bm-green-operator}
The Green operator $\cR$ is a well-defined bounded linear operator on $L^2_0(\pi)$. Moreover, for every smooth $f\in L^2_0(\pi)$, $\cR f$ is a solution to the Poisson equation
\begin{align}
\label{eq:bm-poisson-equation}
    \cL(\cR f)=-f.
\end{align}
Furthermore, $\max_{x\in\SN}\norm{\gradsp(\cR f)(x)}\le\rho_N^{-1}\max_{x\in\SN}\norm{\gradsp f(x)}$ for every smooth $f\in L^2_0(\pi)$.
\end{lemma}

\begin{proof}
As noted in the proof of \cref{lem:sphere-bakry-emery}, $\cL$ satisfies the curvature-dimension condition $\mathrm{CD}(\rho_N,\infty)$. Since $\mathrm{CD}(\rho_N,\infty)$ implies exponential decay of variance, the following holds for every $f\in L^2_0(\pi)$ and $t\ge0$~\citep[Proposition~4.8.1 and Theorem~4.2.5]{bakry2013analysis}:
\begin{align*}
    \norm{P_tf}_{L^2(\pi)}\le e^{-\rho_Nt}\norm{f}_{L^2(\pi)}.
\end{align*}
Integrating the above inequality over $t\ge0$, we obtain the following for every $f\in L^2_0(\pi)$:
\begin{align*}
    \norm{\cR f}_{L^2(\pi)}
    \le\int_0^\infty\norm{P_tf}_{L^2(\pi)}\,dt
    \le\norm{f}_{L^2(\pi)}\int_0^\infty e^{-\rho_Nt}\,dt
    =\rho_N^{-1}\norm{f}_{L^2(\pi)}.
\end{align*}
Hence, $\cR$ is a well-defined bounded linear operator on $L^2_0(\pi)$.

Now fix a smooth $f\in L^2_0(\pi)$. By the Kolmogorov equations~\cref{eq:langevin-kolmogorov}, we obtain
\begin{align*}
    \cL(\cR f)
    =\lim_{T\to\infty}\int_0^T\cL(P_tf)\,dt
    =\lim_{T\to\infty}\int_0^T\frac{d}{dt}P_tf\,dt
    =\lim_{T\to\infty}P_Tf-f
    =\E_\pi[f]-f
    =-f,
\end{align*}
where we use the fact that $\pi$ is the stationary distribution of $P_t$ and $f\in L^2_0(\pi)$. This proves \cref{eq:bm-poisson-equation}. Finally, we infer the following from \cref{lem:sphere-bakry-emery} for every $x\in\SN$:
\begin{align*}
    \norm{\gradsp(\cR f)(x)}
    \le\int_0^\infty\norm{\gradsp(P_tf)(x)}\,dt
    \le\max_{z\in\SN}\norm{\gradsp f(z)}\int_0^\infty e^{-\rho_Nt}\,dt
    =\rho_N^{-1}\max_{z\in\SN}\norm{\gradsp f(z)},
\end{align*}
which proves the gradient bound.
\end{proof}

\begin{remark}
\label{rem:bm-green-operator-provenance}
By \cref{eq:bm-poisson-equation}, we may interpret $\cR$ as the inverse of $-\cL$ restricted to $L^2_0(\pi)$. Solutions to the Poisson equation are a canonical tool for proving central limit theorems for ergodic averages of Markov processes~\citep{komorowski2012fluctuations,eberle2025markov} (see also~\citep[Section~4.1]{chatterjee2005concentration} for connections to Stein's method of exchangeable pairs). In fact, Part~I of our work~\citep{das2026averaged} develops an analogous framework for the simple random walk on the hypercube. In subsequent sections, we shall use the Green operator to analyze ergodic averages along high-temperature spherical Langevin dynamics for Hamiltonians with bounded $\cG^2$ norm.
\end{remark}

\subsection{Recovering the Average Spherical Gradient}
\label{subsec:recovering-average-spherical-gradient}

We now show that the ergodic average of high-temperature spherical Langevin dynamics approximates a scalar multiple of the average spherical gradient under $\pi$. For a smooth potential $V:\SN\to\R$, let $(Z_t)_{t\ge0}$ denote the trajectory of spherical Langevin dynamics at inverse temperature $\beta_0\ge0$, initialized at $Z_0\sim\pi$. For any averaging time $T>0$, define its ergodic average by
\begin{align}
\label{eq:langevin-ergodic-average}
    \Zavg_T\doteq\frac{1}{T}\int_0^TZ_t\,dt.
\end{align}

In \cref{subsec:sketch-ergodic-averaging}, we argued heuristically that, for $V=H$, the ergodic average $\Zavg_T$ converges to the mean of the Gibbs measure $\mu_{V,\beta_0}$ as $T\to\infty$, and that this mean equals $\tfrac{2\beta_0}{\kappa_N}\E_\pi[\gradsp V]+O(\beta_0^2)$. We formalize this heuristic for general potentials in \cref{thm:langevin-recovers-average-gradient} below.

\begin{theorem}[Recovery of the Average Spherical Gradient]
\label{thm:langevin-recovers-average-gradient}
Suppose that $\norm{V}_{\cG^2}\le BN$ for some $B>0$. For every $T\ge1$ and $\delta\in(0,1)$, the following holds with probability at least $1-\delta$:
\begin{align*}
    \norm{\Zavg_T-\frac{2\beta_0}{\kappa_N}\E_\pi[\gradsp V]}
    \lesssim(1+\beta_0B)\sqrt{\frac{N+\log(\nicefrac{1}{\delta})}{T}}+\beta_0^2B^2\sqrt N.
\end{align*}
\end{theorem}

\begin{proof}
The proof proceeds in two steps. The first step combines the Poisson equation from \cref{lem:bm-green-operator} with \cref{eq:langevin-generator} to prove that the ergodic average of $\gradsp V(Z_t)$ concentrates around $\E_\pi[\gradsp V]$. The second step uses the action of the generator on the coordinate functions to show that the ergodic averages of $Z_t$ and $\tfrac{2\beta_0}{\kappa_N}\gradsp V(Z_t)$ are close to each other.

Define $F(x)=\gradsp V(x)-\E_\pi[\gradsp V]$ and $\Phi=\cR F$. For any $w\in\bS^{N-1}$, the function $f_w=\dotp{F}{w}$ is smooth and has zero mean under $\pi$. Since $f_w(x)=\dotp{Q_x\grad V(x)}{w}-\dotp{\E_\pi[\gradsp V]}{w}$, a direct computation using \cref{eq:spherical-derivatives} yields
\begin{align*}
    \gradsp f_w(x)=\gradsp^2V(x)\,w-\frac{\dotp{x}{w}}{N}\gradsp V(x).
\end{align*}
Since $\norm{V}_{\cG^2}\le BN$, we have $\norm{\gradsp^2V(x)}_{\mathrm{op}}\le B$ and $\norm{\gradsp V(x)}\le B\sqrt N$ for every $x\in\SN$, and hence $\norm{\gradsp f_w(x)}\le2B$. Since $\dotp{\Phi}{w}=\cR f_w$ by linearity of $\cR$, we obtain the following from \cref{lem:bm-green-operator} and $\rho_N\ge\tfrac16$:
\begin{align}
\label{eq:gradient-recovery-green-bound}
    \max_{x\in\SN}\norm{\gradsp\Phi(x)}_{\mathrm{op}}
    =\max_{x\in\SN}\,\sup_{w\in\bS^{N-1}}\norm{\gradsp\bigl(\dotp{\Phi}{w}\bigr)(x)}
    \le\frac{2B}{\rho_N}
    \le12B.
\end{align}

The Poisson equation~\cref{eq:bm-poisson-equation} from \cref{lem:bm-green-operator} implies $\cL\Phi=\tfrac12\lapsp\Phi=-F$. Combining this with \cref{eq:langevin-generator}, we obtain
\begin{align}
\label{eq:gradient-recovery-generator}
    \cL_{V,\beta_0}\Phi(x)
    =-F(x)+\beta_0\gradsp\Phi(x)^\top\gradsp V(x).
\end{align}
Let $M_t^\Phi$ denote the Dynkin martingale associated with $\Phi$ along $(Z_t)_{t\ge0}$. By \cref{lem:langevin-dynkin},
\begin{align*}
    M_T^\Phi
    =\Phi(Z_T)-\Phi(Z_0)-\int_0^T\cL_{V,\beta_0}\Phi(Z_t)\,dt.
\end{align*}
By rearranging terms and using \cref{eq:gradient-recovery-generator}, we obtain
\begin{align}
\label{eq:gradient-recovery-dynkin-decomposition}
    \frac{1}{T}\int_0^T\gradsp V(Z_t)\,dt-\E_\pi[\gradsp V]
    =\frac{\Phi(Z_0)-\Phi(Z_T)}{T}
      +\frac{M_T^\Phi}{T}
      +\frac{\beta_0}{T}\int_0^T\gradsp\Phi(Z_t)^\top\gradsp V(Z_t)\,dt.
\end{align}
We control the three terms separately. Since $\norm{Z_0-Z_T}\le2\sqrt N$ and $T\ge1$, we infer the following from \cref{eq:spherical-lipschitz,eq:gradient-recovery-green-bound}:
\begin{align}
\label{eq:gradient-recovery-endpoint-bound}
    \norm{\frac{\Phi(Z_0)-\Phi(Z_T)}{T}}
    \lesssim\frac{B\sqrt N}{T}
    \lesssim B\sqrt{\frac{N}{T}}.
\end{align}
For the drift term, since $\norm{\gradsp V(x)}\le B\sqrt N$, we obtain the following from \cref{eq:gradient-recovery-green-bound} for every $x\in\SN$:
\begin{align}
\label{eq:gradient-recovery-drift-bound}
    \norm{\beta_0\gradsp\Phi(x)^\top\gradsp V(x)}
    \le\beta_0\norm{\gradsp\Phi(x)}_{\mathrm{op}}\norm{\gradsp V(x)}
    \le12\beta_0B^2\sqrt N.
\end{align}

It remains to control the martingale term. By \cref{eq:gradient-recovery-green-bound}, $\max_{x\in\SN}\norm{\gradsp\bigl(\dotp{\Phi}{w}\bigr)(x)}\le12B$ for every $w\in\bS^{N-1}$. Thus, applying \cref{lem:matrix-dynkin-uniform-concentration} with $m=N$ and $n=1$, we obtain the following with probability at least $1-\delta/2$:
\begin{align}
\label{eq:gradient-recovery-martingale-bound}
    \frac{\norm{M_T^\Phi}}{T}
    \lesssim B\sqrt{\frac{N+\log(\nicefrac{1}{\delta})}{T}}.
\end{align}
Substituting \cref{eq:gradient-recovery-endpoint-bound,eq:gradient-recovery-drift-bound,eq:gradient-recovery-martingale-bound} into \cref{eq:gradient-recovery-dynkin-decomposition}, we obtain the following with probability at least $1-\delta/2$:
\begin{align}
\label{eq:gradient-recovery-gradient-average-bound}
    \norm{\frac{1}{T}\int_0^T\gradsp V(Z_t)\,dt-\E_\pi[\gradsp V]}
    \lesssim B\sqrt{\frac{N+\log(\nicefrac{1}{\delta})}{T}}+\beta_0B^2\sqrt N.
\end{align}

It remains to compare the respective ergodic averages of $Z_t$ and $\gradsp V(Z_t)$. To this end, applying \cref{eq:spherical-derivatives} to the coordinate functions, we obtain $\gradsp x=Q_x$ and $\lapsp x=-\kappa_Nx$. Since $Q_x\gradsp V(x)=\gradsp V(x)$, we infer the following from \cref{eq:langevin-generator}:
\begin{align*}
    \cL_{V,\beta_0}x=-\frac{\kappa_N}{2}x+\beta_0\gradsp V(x).
\end{align*}
Let $M_t^x$ denote the associated Dynkin martingale. By the above identity and \cref{lem:langevin-dynkin},
\begin{align*}
    M_T^x
    =Z_T-Z_0-\int_0^T\cL_{V,\beta_0}x(Z_t)\,dt
    =Z_T-Z_0-\int_0^T\Bigl[\beta_0\gradsp V(Z_t)-\frac{\kappa_N}{2}Z_t\Bigr]\,dt.
\end{align*}
By rearranging terms and using \cref{eq:langevin-ergodic-average}, we obtain
\begin{align}
\label{eq:gradient-recovery-state-decomposition}
    \Zavg_T-\frac{2\beta_0}{\kappa_N}\cdot\frac{1}{T}\int_0^T\gradsp V(Z_t)\,dt
    =\frac{2}{\kappa_N}\left[\frac{Z_0-Z_T}{T}+\frac{M_T^x}{T}\right].
\end{align}
Since $Z_0,Z_T\in\SN$, $\kappa_N\ge\tfrac23$ and $T\ge1$, the endpoint term is at most $6\sqrt N/T\le6\sqrt{N/T}$. Moreover, $\gradsp\dotp{x}{w}=Q_xw$ has norm at most $1$ for every $w\in\bS^{N-1}$. Thus, applying \cref{lem:matrix-dynkin-uniform-concentration} with $m=N$ and $n=1$, we obtain the following with probability at least $1-\delta/2$:
\begin{align*}
    \frac{\norm{M_T^x}}{T}
    \lesssim\sqrt{\frac{N+\log(\nicefrac{1}{\delta})}{T}}.
\end{align*}
Combining these estimates with \cref{eq:gradient-recovery-state-decomposition}, we conclude that, with probability at least $1-\delta/2$,
\begin{align}
\label{eq:gradient-recovery-state-average-bound}
    \norm{\Zavg_T-\frac{2\beta_0}{\kappa_N}\cdot\frac{1}{T}\int_0^T\gradsp V(Z_t)\,dt}
    \lesssim\sqrt{\frac{N+\log(\nicefrac{1}{\delta})}{T}}.
\end{align}
The proof is concluded by combining \cref{eq:gradient-recovery-gradient-average-bound,eq:gradient-recovery-state-average-bound} via the triangle inequality and a union bound, since $\tfrac{2\beta_0}{\kappa_N}\le3\beta_0$.
\end{proof}

\subsection{Recovering the Average Spherical Hessian}
\label{subsec:recovering-average-spherical-hessian}

Borrowing the notation from \cref{subsec:recovering-average-spherical-gradient}, we use $(Z_t)_{t\ge0}$ to denote spherical Langevin dynamics for a smooth potential $V:\SN\to\R$ at inverse temperature $\beta_0\ge0$, initialized at $Z_0\sim\pi$. For any averaging time $T>0$, define its uncentered ergodic second moment by
\begin{align*}
    \Zavg^{(2)}_T\doteq\frac{1}{T}\int_0^TZ_tZ_t^\top\,dt.
\end{align*}
The next theorem establishes a concentration guarantee for the ergodic second moment for potentials with bounded Sobolev norm.

\begin{theorem}[Recovery of the Average Spherical Hessian]
\label{thm:langevin-recovers-average-hessian}
Suppose that $\norm{V}_{\cG^2}\le BN$ for some $B>0$. For every $T\ge1$ and $\delta\in(0,1)$, the following holds with probability at least $1-\delta$:
\begin{align*}
    \norm{\Zavg^{(2)}_T-I-2\beta_0\E_\pi[\gradsp^2V]}_{\mathrm{op}}
    \lesssim(1+\beta_0B)\sqrt N\sqrt{\frac{N+\log(\nicefrac{1}{\delta})}{T}}+\beta_0^2B^2N.
\end{align*}
\end{theorem}

\begin{proof}
Define the matrix-valued observable $J_V:\SN\to\Sym_N$ by
\begin{align}
\label{eq:hessian-recovery-observable}
    J_V(x)\doteq\frac12\bigl(x\gradsp V(x)^\top+\gradsp V(x)x^\top\bigr).
\end{align}
By Stein's lemma (\cref{lem:spherical-stein}), $\E_\pi[J_V]=\E_\pi[\gradsp^2V]$. As in the proof of \cref{thm:langevin-recovers-average-gradient}, we first combine the Poisson equation from \cref{lem:bm-green-operator} with \cref{eq:langevin-generator} to show that the ergodic average of $J_V(Z_t)$ concentrates around $\E_\pi[\gradsp^2V]$. We then use the (entrywise) action of the generator on $xx^\top-I$ to compare $\Zavg^{(2)}_T-I$ with the ergodic average of $2\beta_0J_V(Z_t)$.

Define $F(x)=J_V(x)-\E_\pi[\gradsp^2V]$ and $\Phi=\cR F$, which takes values in $\Sym_N$. For any $w\in\bS^{N-1}$, the function $f_w=w^\top Fw$ is smooth and has zero mean under $\pi$. Since $w^\top J_V(x)w=\dotp{x}{w}\dotp{Q_x\grad V(x)}{w}$, a direct computation using \cref{eq:spherical-derivatives} and the product rule yields
\begin{align*}
    \gradsp f_w(x)
    =\dotp{\gradsp V(x)}{w}Q_xw+\dotp{x}{w}\left(\gradsp^2V(x)\,w-\frac{\dotp{x}{w}}{N}\gradsp V(x)\right).
\end{align*}
Since $\norm{V}_{\cG^2}\le BN$, we have $\norm{\gradsp^2V(x)}_{\mathrm{op}}\le B$ and $\norm{\gradsp V(x)}\le B\sqrt N$ for every $x\in\SN$, and hence $\norm{\gradsp f_w(x)}\le3B\sqrt N$. Since $w^\top\Phi w=\cR f_w$ by linearity of $\cR$, we obtain the following from \cref{lem:bm-green-operator} and $\rho_N\ge\tfrac16$:
\begin{align}
\label{eq:hessian-recovery-green-bound}
    \max_{x\in\SN}\,\sup_{w\in\bS^{N-1}}\norm{\gradsp(w^\top\Phi w)(x)}
    \le\frac{3B\sqrt N}{\rho_N}
    \le18B\sqrt N.
\end{align}

The Poisson equation~\cref{eq:bm-poisson-equation} from \cref{lem:bm-green-operator} implies $\cL\Phi=\tfrac12\lapsp\Phi=-F$. Combining this with \cref{eq:langevin-generator}, we obtain
\begin{align}
\label{eq:hessian-recovery-generator}
    \cL_{V,\beta_0}\Phi(x)
    =-F(x)+\beta_0\,\mathrm{D}\Phi(x)[\gradsp V(x)],
\end{align}
where $\mathrm{D}\Phi(x)[v]\in\Sym_N$ denotes the differential of $\Phi$ at $x$ in the direction $v\in T_x\SN$, i.e., the matrix with entries $\dotp{\gradsp\Phi_{ij}(x)}{v}$. Now, let $M_t^\Phi$ denote the Dynkin martingale associated with $\Phi$ along $(Z_t)_{t\ge0}$. Then, by \cref{lem:langevin-dynkin},
\begin{align*}
    M_T^\Phi
    =\Phi(Z_T)-\Phi(Z_0)-\int_0^T\cL_{V,\beta_0}\Phi(Z_t)\,dt.
\end{align*}
By rearranging terms and using \cref{eq:hessian-recovery-generator}, we obtain
\begin{align}
\label{eq:hessian-recovery-dynkin-decomposition}
    \frac{1}{T}\int_0^TJ_V(Z_t)\,dt-\E_\pi[\gradsp^2V]
    =\frac{\Phi(Z_0)-\Phi(Z_T)}{T}+\frac{M_T^\Phi}{T}
      +\frac{\beta_0}{T}\int_0^T\mathrm{D}\Phi(Z_t)[\gradsp V(Z_t)]\,dt.
\end{align}
We control the three terms separately. Since $\Phi$ takes values in $\Sym_N$, we have $\norm{\Phi(Z_0)-\Phi(Z_T)}_{\mathrm{op}}=\sup_{w\in\bS^{N-1}}\abs{w^\top\Phi(Z_0)w-w^\top\Phi(Z_T)w}$. Hence, since $\norm{Z_0-Z_T}\le2\sqrt N$ and $T\ge1$, we infer the following from \cref{eq:spherical-lipschitz,eq:hessian-recovery-green-bound}:
\begin{align}
\label{eq:hessian-recovery-endpoint-bound}
    \norm{\frac{\Phi(Z_0)-\Phi(Z_T)}{T}}_{\mathrm{op}}
    \lesssim\frac{BN}{T}
    \lesssim B\sqrt N\sqrt{\frac{N}{T}}.
\end{align}
For the drift term, observe that $w^\top\mathrm{D}\Phi(x)[\gradsp V(x)]\,w=\dotp{\gradsp(w^\top\Phi w)(x)}{\gradsp V(x)}$ for every $w\in\bS^{N-1}$. Since $\norm{\gradsp V(x)}\le B\sqrt N$, we obtain the following from Cauchy--Schwarz and \cref{eq:hessian-recovery-green-bound} for every $x\in\SN$:
\begin{align}
\label{eq:hessian-recovery-drift-bound}
    \norm{\beta_0\,\mathrm{D}\Phi(x)[\gradsp V(x)]}_{\mathrm{op}}
    =\beta_0\sup_{w\in\bS^{N-1}}\abs{\dotp{\gradsp(w^\top\Phi w)(x)}{\gradsp V(x)}}
    \le18\beta_0B^2N.
\end{align}

It remains to control the martingale term. By \cref{eq:hessian-recovery-green-bound} and \cref{rem:symmetric-matrix-dynkin-concentration}, applying \cref{lem:matrix-dynkin-uniform-concentration} with $m=n=N$ and $A=18B\sqrt N$, we obtain the following with probability at least $1-\delta/2$:
\begin{align}
\label{eq:hessian-recovery-martingale-bound}
    \frac{\norm{M_T^\Phi}_{\mathrm{op}}}{T}
    \lesssim B\sqrt N\sqrt{\frac{N+\log(\nicefrac{1}{\delta})}{T}}.
\end{align}
Substituting \cref{eq:hessian-recovery-endpoint-bound,eq:hessian-recovery-drift-bound,eq:hessian-recovery-martingale-bound} into \cref{eq:hessian-recovery-dynkin-decomposition}, we obtain the following with probability at least $1-\delta/2$:
\begin{align}
\label{eq:hessian-recovery-observable-average-bound}
    \norm{\frac{1}{T}\int_0^TJ_V(Z_t)\,dt-\E_\pi[\gradsp^2V]}_{\mathrm{op}}
    \lesssim B\sqrt N\sqrt{\frac{N+\log(\nicefrac{1}{\delta})}{T}}+\beta_0B^2N.
\end{align}

It remains to compare $\Zavg^{(2)}_T-I$ with the ergodic average of $2\beta_0J_V(Z_t)$. To this end, define $K(x)=xx^\top-I$. Applying \cref{eq:spherical-derivatives} to each entry of $xx^\top$, we infer the following from $Q_x\gradsp V(x)=\gradsp V(x)$ and \cref{eq:hessian-recovery-observable}:
\begin{align*}
    \frac12\lapsp K(x)=I-xx^\top=-K(x),
    \qquad
    \mathrm{D}K(x)[\gradsp V(x)]=x\gradsp V(x)^\top+\gradsp V(x)x^\top=2J_V(x).
\end{align*}
Combining the above identities with \cref{eq:langevin-generator}, where $\mathrm{D}K(x)[v]\in\Sym_N$ denotes the differential of $K$ at $x$ in the direction $v\in T_x\SN$, we obtain
\begin{align}
\label{eq:hessian-recovery-quadratic-generator}
    \cL_{V,\beta_0}K(x)
    =\frac12\lapsp K(x)+\beta_0\,\mathrm{D}K(x)[\gradsp V(x)]
    =-K(x)+2\beta_0J_V(x).
\end{align}
Let $M_t^K$ denote the associated Dynkin martingale. By \cref{eq:hessian-recovery-quadratic-generator} and \cref{lem:langevin-dynkin},
\begin{align*}
    M_T^K
    =K(Z_T)-K(Z_0)-\int_0^T\cL_{V,\beta_0}K(Z_t)\,dt
    =K(Z_T)-K(Z_0)-\int_0^T\bigl[2\beta_0J_V(Z_t)-K(Z_t)\bigr]\,dt.
\end{align*}
By rearranging terms and using $\tfrac{1}{T}\int_0^TK(Z_t)\,dt=\Zavg^{(2)}_T-I$, we obtain
\begin{align}
\label{eq:hessian-recovery-second-moment-decomposition}
    \Zavg^{(2)}_T-I-\frac{2\beta_0}{T}\int_0^TJ_V(Z_t)\,dt
    =\frac{K(Z_0)-K(Z_T)}{T}+\frac{M_T^K}{T}.
\end{align}
Since $\norm{K(x)}_{\mathrm{op}}\le N$ for every $x\in\SN$ and $T\ge1$, the endpoint term is at most $2N/T$. Moreover, for every $w\in\bS^{N-1}$, $\gradsp(w^\top Kw)(x)=2\dotp{x}{w}Q_xw$ has norm at most $2\sqrt N$. Thus, by \cref{rem:symmetric-matrix-dynkin-concentration}, applying \cref{lem:matrix-dynkin-uniform-concentration} with $m=n=N$, we obtain the following with probability at least $1-\delta/2$:
\begin{align*}
    \frac{\norm{M_T^K}_{\mathrm{op}}}{T}
    \lesssim\sqrt N\sqrt{\frac{N+\log(\nicefrac{1}{\delta})}{T}}.
\end{align*}
Combining these estimates with \cref{eq:hessian-recovery-second-moment-decomposition}, we conclude that, with probability at least $1-\delta/2$,
\begin{align}
\label{eq:hessian-recovery-second-moment-bound}
    \norm{\Zavg^{(2)}_T-I-\frac{2\beta_0}{T}\int_0^TJ_V(Z_t)\,dt}_{\mathrm{op}}
    \lesssim\sqrt N\sqrt{\frac{N+\log(\nicefrac{1}{\delta})}{T}}.
\end{align}
The proof is concluded by combining \cref{eq:hessian-recovery-observable-average-bound,eq:hessian-recovery-second-moment-bound} via the triangle inequality and a union bound.
\end{proof}

\subsection{Proof of \texorpdfstring{\cref{thm:tpca-averaged-differential-recovery}}{Theorem~\ref{thm:tpca-averaged-differential-recovery}}}
\label{subsec:tpca-recovery-proof}

\begin{proof}[Proof of \cref{thm:tpca-averaged-differential-recovery}]
\leavevmode\par
Let $\cE=\{\norm{H}_{\cG^2}\le C_k\lambda N\}$. Since $\lambda\ge1$, \cref{rem:tensor-pca-sobolev-bound} implies $\Prb[\cE]\ge1-\exp(-c_kN)$. Since our assumptions imply $T_{\mathrm{av}}\ge1$, for every realization of $H$ in $\cE$, we may apply \cref{thm:langevin-recovers-average-gradient,thm:langevin-recovers-average-hessian} with $V=H$, $B\asymp_k\lambda$, and $\delta=\exp(-N)$. Since $\kappa_N\le1$, taking a union bound, we obtain the following jointly with $\cE$, with probability at least $1-\exp(-c_kN)$:
\begin{align}
\label{eq:tpca-averaging-normalized-errors}
\begin{aligned}
    \norm{\frac{\kappa_N}{2\beta_0}\Xbar_{T_{\mathrm{av}}}-u_H}
    &\le\eps_N,
    \qquad
    \norm{\frac{\Sigmabar_{T_{\mathrm{av}}}-I}{2\beta_0}-A_H}_{\mathrm{op}}
    \le\sqrt N\,\eps_N,\\
    \eps_N
    &\doteq C_k\left[(\beta_0^{-1}+\lambda)\sqrt{\frac{N}{T_{\mathrm{av}}}}+\beta_0\lambda^2\sqrt N\right].
\end{aligned}
\end{align}
Since $\lambda\le\beta_0^{-1}$ under our assumptions, the bounds on $\beta_0$ and $T_{\mathrm{av}}$ imply
\begin{align}
\label{eq:tpca-averaging-error-scale}
    \eps_N
    \le2C_k\left[\beta_0^{-1}\sqrt{\frac{N}{T_{\mathrm{av}}}}+\beta_0\lambda^2\sqrt N\right]
    \le c_k\lambda N^{\nicefrac{1}{2}-\lfloor k/2\rfloor-\vartheta}.
\end{align}
Here $c_k>0$ can be made arbitrarily small by choosing the leading constant in the upper bound on $\beta_0$ sufficiently small and that in the lower bound on $T_{\mathrm{av}}$ sufficiently large.

For odd $k$, we have $\lfloor k/2\rfloor=\tfrac{k-1}{2}$, so by \cref{eq:tpca-averaging-error-scale} we have $\eps_N\le c_k\lambda N^{-\tfrac{k-2}{2}-\vartheta}$. By \cref{cor:avg-spherical-grad-correlation}, the following holds with probability at least $1-C\exp(-cN)$:
\begin{align}
\label{eq:tpca-gradient-norm-lower-bound}
    \norm{u_H}\ge c\bigl(N\lambda\alpha_N+\sigma_N\sqrt N\bigr),
\end{align}
where $\alpha_N$ and $\sigma_N$ are as defined in \cref{thm:avg-spherical-grad-snr}. Henceforth, we assume that both \cref{eq:tpca-averaging-normalized-errors,eq:tpca-gradient-norm-lower-bound} hold, which occurs with probability at least $1-\exp(-c_kN)$ by a union bound.

Combining \cref{eq:tpca-averaging-error-scale,eq:tpca-gradient-norm-lower-bound} with $\alpha_N\asymp_kN^{-\nicefrac{k}{2}}$ and $\sigma_N^2=\kappa_N\sqrt N\alpha_N$, we obtain
\begin{align}
\label{eq:tpca-gradient-relative-error}
\begin{aligned}
    \frac{\eps_N}{\norm{u_H}}
    &\le\frac{c_k\lambda N^{-\tfrac{k-2}{2}-\vartheta}}{c\bigl(N\lambda\alpha_N+\sigma_N\sqrt N\bigr)}
    \lesssim_k\frac{c_kN^{-\vartheta}\,N\lambda\alpha_N}{N\lambda\alpha_N+\sigma_N\sqrt N}\\
    &\le c_kN^{-\vartheta}\min\left\{1,\frac{N\lambda\alpha_N}{\sigma_N\sqrt N}\right\}
    =c_kN^{-\vartheta}\min\left\{1,\frac{\lambda\sigma_N}{\kappa_N}\right\},
\end{aligned}
\end{align}
where the implicit constant does not depend on $c_k$. In particular, since $c_k$ can be made arbitrarily small, we have $\eps_N<\norm{u_H}$, and hence $\Xbar_{T_{\mathrm{av}}}\ne0$ by \cref{eq:tpca-averaging-normalized-errors}. Now, applying \cref{lem:normalization-lipschitz} to $\tfrac{\kappa_N}{2\beta_0}\Xbar_{T_{\mathrm{av}}}$ and $u_H$, we infer the following from \cref{eq:tpca-averaging-normalized-errors,eq:tpca-gradient-relative-error}:
\begin{align*}
    \frac{1}{\sqrt N}\norm{\widehat X_{T_{\mathrm{av}}}-\widehat u_H}
    =\norm{\frac{\Xbar_{T_{\mathrm{av}}}}{\norm{\Xbar_{T_{\mathrm{av}}}}}-\frac{u_H}{\norm{u_H}}}
    \le\frac{2\eps_N}{\norm{u_H}}
    \lesssim_kc_kN^{-\vartheta}\min\left\{1,\frac{\lambda\sigma_N}{\kappa_N}\right\}.
\end{align*}
Since $\kappa_N\ge\tfrac23$ and $c_k$ can be made arbitrarily small, squaring proves the claim for odd $k$.

For even $k$, we infer from \cref{thm:avg-spherical-hess-snr} that
\begin{align}
\label{eq:tpca-hessian-signal-noise}
    A_H=A_{\mathrm{sig}}+W_N,
    \qquad
    A_{\mathrm{sig}}\doteq\lambda\gamma_N\left(\plant(\plant)^\top-\frac{I}{N}\right),
    \qquad
    \gamma_N\asymp_kN^{-\tfrac{k-2}{2}},
\end{align}
where $\norm{W_N}_{\mathrm{op}}\le C_kN^{-\tfrac{k-2}{4}}$ with probability at least $1-\exp(-cN)$. Since $\lfloor k/2\rfloor=\tfrac{k}{2}$, \cref{eq:tpca-averaging-error-scale} implies ${\sqrt N\,\eps_N\le c_k\lambda N^{-\tfrac{k-2}{2}-\vartheta}}$. Combining \cref{eq:tpca-averaging-normalized-errors,eq:tpca-hessian-signal-noise} and taking a union bound, we obtain the following with probability at least $1-\exp(-c_kN)$:
\begin{align}
\label{eq:tpca-hessian-perturbation-bound}
    \frac{\norm{\Sigmabar_{T_{\mathrm{av}}}-(I+2\beta_0A_{\mathrm{sig}})}_{\mathrm{op}}}{2\beta_0\lambda\gamma_N}
    \le\frac{\sqrt N\,\eps_N+\norm{W_N}_{\mathrm{op}}}{\lambda\gamma_N}
    \lesssim_kN^{-\vartheta}+\frac{N^{\tfrac{k-2}{4}}}{\lambda}.
\end{align}
The matrix $I+2\beta_0A_{\mathrm{sig}}$ has top eigenvector $\plant$ and eigengap $2\beta_0\lambda\gamma_N$. Applying the Davis--Kahan theorem~\citep[Theorem~4.5.5]{vershynin2018high} to $\Sigmabar_{T_{\mathrm{av}}}$ and $I+2\beta_0A_{\mathrm{sig}}$, and using \cref{eq:tpca-hessian-perturbation-bound}, we obtain
\begin{align*}
    \min_{s\in\{\pm1\}}\norm{\widehat v-s\plant}
    \lesssim\frac{\norm{\Sigmabar_{T_{\mathrm{av}}}-(I+2\beta_0A_{\mathrm{sig}})}_{\mathrm{op}}}{2\beta_0\lambda\gamma_N}
    \lesssim_kN^{-\vartheta}+\frac{N^{\tfrac{k-2}{4}}}{\lambda}.
\end{align*}
Since $m(\sqrt N\,\widehat v)=\dotp{\plant}{\widehat v}$ and $\norm{\widehat v-s\plant}^2=2-2s\dotp{\plant}{\widehat v}$ for every $s\in\{\pm1\}$, we conclude that
\begin{align*}
    1-\bigl|m(\sqrt N\,\widehat v)\bigr|
    =\frac12\min_{s\in\{\pm1\}}\norm{\widehat v-s\plant}^2
    \lesssim_k\frac{N^{\tfrac{k-2}{2}}}{\lambda^2}+N^{-2\vartheta},
\end{align*}
which proves the claim for even $k$.
\end{proof}

\section{Strong Recovery at the Algorithmic Threshold for Odd \texorpdfstring{$k$}{k}}
\label{sec:strong-recovery-odd}

Throughout this section, we assume that $k$ is odd. We use $(X_t)_{t\ge0}$ to denote the trajectory of spherical Langevin dynamics for the tensor PCA Hamiltonian $H$ at a constant inverse temperature $\beta>0$, and assume that the signal strength satisfies $\lambda\ge C_{\beta,k}N^\gamma$, where $\gamma\ge\tfrac{k-2}{4}$ and $C_{\beta,k}$ is a sufficiently large constant.

Our goal in this section is to prove the following theorem, which is a formal version of \cref{thm:langevin-odd-k-informal}.

\begin{theorem}[Persistent Strong Recovery above the Algorithmic Threshold]
\label{thm:persistent-strong-recovery-algorithmic-threshold}
Let $k\ge3$ be odd, let $\beta>0$, and assume $\lambda\ge C_{\beta,k}N^{\tfrac{k-2}{4}}$, where $C_{\beta,k}$ is a sufficiently large constant. Fix $\vartheta\in(0,1)$ and ${\eta\in(0,\min\{\nicefrac{1}{2},\tfrac{k-2}{4}\})}$, and let $(X_t)_{t\ge0}$ denote the trajectory of the second phase of \cref{alg:averaged-annealed-langevin}, initialized at $\tilX_0\sim\pi$ with inverse temperatures $\beta_0$ and $\beta$ for the two respective phases, where $\beta_0\asymp_k\lambda^{-1}N^{-\tfrac{k-1}{2}-\vartheta}$ and $T_{\mathrm{av}}\asymp_kN^{2k-2+4\vartheta}$, with sufficiently small and large implicit constants, respectively. Then, there exists a constant $T_{\mathrm{str}}>0$, depending only on $\beta$ and $k$, such that the following holds for all sufficiently large $N$ and every deterministic $T$ satisfying $T_{\mathrm{str}}\le T\le\poly(N)$:
\begin{align*}
    \Prb\left[\inf_{t\in[T_{\mathrm{str}},T]}m(X_t)\ge1-N^{-\eta}\right]
    \ge1-\exp\left(-c_{\beta,k}N^{\min\{1/2,\,1-2\eta\}}\right).
\end{align*}
\end{theorem}

The behavior of the second phase depends on how warm its initialization is. Once the correlation exceeds a large constant multiple of the critical correlation $\lambda^{-\tfrac{1}{k-1}}$, the signal drift dominates the disorder drift uniformly over $\SN$, and the dynamics achieves persistent strong recovery within constant time, as we show in \cref{subsec:boosting-critical-to-strong-recovery}. However, for $\lambda$ just above the algorithmic threshold, the warm start computed by the first phase of \cref{alg:averaged-annealed-langevin} attains only subcritical correlation, at which the signal drift need not dominate the disorder drift uniformly. This subcritical regime is the topic of \cref{subsec:critical-recovery-warm-start,subsec:average-gradient-gauge,subsec:reverse-planting,subsec:critical-recovery-reverse-planting}, where we show that the warm start depends on the disorder only in a benign way, so that the signal drift still dominates the disorder drift along the trajectory until the dynamics achieves critical recovery. The complete proof is presented in \cref{subsec:proof-persistent-strong-recovery-algorithmic-threshold}.

\subsection{Boosting Critical Recovery to Persistent Strong Recovery}
\label{subsec:boosting-critical-to-strong-recovery}

Recall from \cref{subsec:sketch-reverse-planting} that the signal drift dominates both the entropic and disorder drifts, even in the worst case, whenever the correlation exceeds a large constant multiple of the \emph{critical correlation} $\lambda^{-\tfrac{1}{k-1}}$. Indeed, the signal drift is of order $\lambda m(X_t)^{k-1}$, whereas the disorder drift is of order one uniformly over $\SN$. We formalize this worst-case argument in \cref{thm:strong-recovery-after-critical-correlation} below, which shows that, once the process has attained critical correlation, running spherical Langevin dynamics for an additional constant time suffices to achieve $1-\poly(\nicefrac{1}{N})$ correlation with high probability. Beyond that, the correlation stays above this level throughout any $\poly(N)$ time horizon.

\begin{theorem}[Persistent Strong Recovery after Critical Correlation]
\label{thm:strong-recovery-after-critical-correlation}
Fix $\eta\in(0,\min\{\nicefrac{1}{2},\tfrac{k-2}{4}\})$ and $s\in(0,\nicefrac{1}{2})$, define $\bar r_N\doteq N^{-s}$, and assume that $\lambda\bar r_N^{k-1}\ge C_{\beta,k}$. Define the stopping time
\begin{align*}
    \tau_{C\bar r_N}\doteq\inf\{t\ge0:m(X_t)\ge C\bar r_N\},
\end{align*}
where $C\geq 1$ is a constant. Then, the following holds for all sufficiently large $N$, every deterministic $T$ satisfying $2\le T\le\poly(N)$, and any initial state $X_0$, possibly depending on $H$,
\begin{align*}
    \Prb\left[\inf_{t\in[2,T]}m(X_{\tau_{C\bar r_N}+t})\ge1-N^{-\eta}\right]
    \ge1-\exp(-c_kN)-2\exp\left(-cN\min\{\bar r_N^2,N^{-2\eta}\}\right).
\end{align*}
\end{theorem}

\begin{proof}
Write $q_N\doteq N^{-\eta}$ and $\eps_N\doteq\tfrac{1}{32}\min\{\bar r_N,q_N\}$. Throughout the proof, we assume that $N$ is sufficiently large and $C_{\beta,k}$ is a sufficiently large constant. Since $\SN$ is a compact manifold and spherical Langevin dynamics is a Markov diffusion process, $\tau_{C\bar r_N}<\infty$ holds almost surely by the support theorem~\citep{millet1994simple}. Hence, by the strong Markov property~\citep[Theorem~6.2.2]{stroock2007multidimensional} applied at $\tau_{C\bar r_N}$, it suffices to consider an initial state $X_0$ with $m(X_0)\ge C\bar r_N$, possibly depending on $H$.

Since $m(x)$ is linear, \cref{eq:spherical-derivatives} implies $\gradsp m(x)=N^{-1/2}Q_x\plant$ and $\lapsp m=-\kappa_Nm$, and hence $N\norm{\gradsp m(x)}^2=\norm{Q_x\plant}^2=1-m(x)^2$. Thus, we obtain the following from \cref{eq:langevin-generator}
\begin{align}
    \cL_{H,\beta}m(x)
    &=\frac12\lapsp m(x)+\beta\dotp{\gradsp H(x)}{\gradsp m(x)}\nonumber\\*
    &=-\frac{\kappa_N}{2}m(x)+\beta\dotp{N\lambda km(x)^{k-1}\gradsp m(x)+\gradsp H_0(x)}{\gradsp m(x)}\nonumber\\*
    &=-\frac{\kappa_N}{2}m(x)+\beta\lambda km(x)^{k-1}\bigl(1-m(x)^2\bigr)+\beta\dotp{\gradsp H_0(x)}{\gradsp m(x)}\nonumber\\*
    &=\nu(x)\bigl(1-m(x)\bigr)-b(x),
    \label{eq:strong-recovery-overlap-drift}
\end{align}
where $\nu(x)$ and $b(x)$ are defined as follows:
\begin{align*}
    \nu(x)\doteq\beta\lambda km(x)^{k-1}\bigl(1+m(x)\bigr),
    \qquad
    b(x)\doteq\frac{\kappa_N}{2}m(x)-\beta\dotp{\gradsp H_0(x)}{\gradsp m(x)}.
\end{align*}
Let $M_t^m$ denote the Dynkin martingale associated with $m$. By \cref{lem:langevin-dynkin},
\begin{align}
\label{eq:strong-recovery-dynkin}
    M_t^m=m(X_t)-m(X_0)-\int_0^t\cL_{H,\beta}m(X_s)\,ds.
\end{align}
Consider the events
\begin{align*}
    \cE\doteq\left\{\max_{x\in\SN}\norm{\gradsp H_0(x)}\le C_1\sqrt N\right\},
    \qquad
    \calM_j\doteq\left\{\sup_{j\le t\le j+1}\bigl|M_t^m-M_j^m\bigr|\le\eps_N\right\},
\end{align*}
and let $\calM\doteq\bigcap_{j=0}^{\lceil T\rceil-1}\calM_j$. By \cref{lem:kspin-sobolev-norm-concentration} with $n=1$ and $r=N$, together with \cref{def:spherical-sobolev-norm}, we can choose $C_1$ such that $\Prb[\cE]\ge1-\exp(-c_kN)$. Moreover, since $\norm{\gradsp m(x)}\le N^{-1/2}$, we infer from the Markov property, \cref{lem:dynkin-martingale-concentration}, and a union bound that
\begin{align}
\label{eq:strong-recovery-martingale-bound}
    \Prb[\calM^c]
    \le\sum_{j=0}^{\lceil T\rceil-1}\Prb[\calM_j^c]
    \le2\lceil T\rceil\exp\left(-\frac{N\eps_N^2}{2}\right).
\end{align}
Henceforth, we work on the event $\cE\cap\calM$.

We first express the deficit $1-m(X_t)$ in terms of $b(X_t)$. To this end, we define the integrating factor ${R_t\doteq\int_0^t\nu(X_s)\,ds}$. By \cref{eq:strong-recovery-overlap-drift,eq:strong-recovery-dynkin},
\begin{align*}
    1-m(X_t)+M_t^m
    =1-m(X_0)-\int_0^t\Bigl[\nu(X_s)\bigl(1-m(X_s)\bigr)-b(X_s)\Bigr]\,ds.
\end{align*}
By the continuity of Brownian sample paths, the integrand is continuous in $s$ and thus, the right-hand side is continuously differentiable in $t$. It follows that,
\begin{align}
    \frac{d}{dt}\Bigl[e^{R_t}\bigl(1-m(X_t)+M_t^m\bigr)\Bigr]
    &=e^{R_t}\nu(X_t)\bigl(1-m(X_t)+M_t^m\bigr)
      -e^{R_t}\Bigl[\nu(X_t)\bigl(1-m(X_t)\bigr)-b(X_t)\Bigr]\nonumber\\
    &=e^{R_t}\bigl[\nu(X_t)M_t^m+b(X_t)\bigr].
    \label{eq:strong-recovery-derivative}
\end{align}
Integrating \cref{eq:strong-recovery-derivative} from $0$ to $t$, using $M_0^m=0$, and multiplying both sides by $e^{-R_t}$, we obtain
\begin{align}
\label{eq:strong-recovery-duhamel}
    1-m(X_t)
    =e^{-R_t}\bigl(1-m(X_0)\bigr)
     +\int_0^te^{-(R_t-R_s)}b(X_s)\,ds
     +\xi_t,
    \qquad0\le t\le T,
\end{align}
where the martingale error term $\xi_t$ is given by
\begin{align}
\label{eq:strong-recovery-error-term}
    \xi_t\doteq\int_0^t\nu(X_s)e^{-(R_t-R_s)}M_s^m\,ds-M_t^m.
\end{align}
To control the integral involving $b$ and the martingale error term $\xi_t$ in \cref{eq:strong-recovery-duhamel}, we shall repeatedly use the following identity, which holds since $R_0=0$ and $\tfrac{d}{ds}e^{-(R_t-R_s)}=\nu(X_s)e^{-(R_t-R_s)}$:
\begin{align}
\label{eq:strong-recovery-telescoping}
    \int_0^t\nu(X_s)e^{-(R_t-R_s)}\,ds=1-e^{-R_t}.
\end{align}
Multiplying \cref{eq:strong-recovery-telescoping} by $M_t^m$, we can write
\begin{align}
\label{eq:strong-recovery-martingale-rewrite}
    M_t^m=\int_0^t\nu(X_s)e^{-(R_t-R_s)}M_t^m\,ds+e^{-R_t}M_t^m.
\end{align}
Substituting \cref{eq:strong-recovery-martingale-rewrite} into \cref{eq:strong-recovery-error-term} and using $M_0^m=0$, we express the martingale error term in terms of the increments of $M^m$:
\begin{align}
\label{eq:strong-recovery-martingale-error}
    \xi_t
    =\int_0^t\nu(X_s)e^{-(R_t-R_s)}\bigl(M_s^m-M_t^m\bigr)\,ds-e^{-R_t}\bigl(M_t^m-M_0^m\bigr).
\end{align}
Since we work on $\calM$, we infer from the triangle inequality that
\begin{align}
\label{eq:strong-recovery-increment-bound}
\begin{aligned}
    \bigl|M_t^m-M_s^m\bigr|
    &\le\bigl|M_t^m-M_{\lfloor t\rfloor}^m\bigr|
      +\sum_{l=\lfloor s\rfloor+1}^{\lfloor t\rfloor}\bigl|M_l^m-M_{l-1}^m\bigr|
      +\bigl|M_s^m-M_{\lfloor s\rfloor}^m\bigr|\\
    &\le\bigl(\lfloor t\rfloor-\lfloor s\rfloor+2\bigr)\eps_N
    \le(t-s+3)\eps_N,
    \qquad0\le s\le t\le T.
\end{aligned}
\end{align}

We now bound $\nu(x)$ and $b(x)$ in terms of the correlation $m(x)$. Since we work on $\cE$, $\kappa_N\le1$, and $\norm{\gradsp m(x)}\le N^{-1/2}$, the Cauchy--Schwarz inequality implies
\begin{align}
\label{eq:strong-recovery-b-bound}
    b(x)
    \le\frac{\kappa_N}{2}+\beta\norm{\gradsp H_0(x)}\,\norm{\gradsp m(x)}
    \le\frac12+\beta C_1,
    \qquad x\in\SN.
\end{align}
Moreover, $\nu(x)\ge\beta\lambda k\,m(x)^{k-1}$ whenever $m(x)>0$. Hence, if $m(x)\ge C\bar r_N/2$, we infer from $\lambda\bar r_N^{k-1}\ge C_{\beta,k}$ and $C\ge1$ that
\begin{align}
\label{eq:strong-recovery-nu-crude}
    \nu(x)
    \ge\beta\lambda k\Bigl(\frac{C\bar r_N}{2}\Bigr)^{k-1}
    \ge\frac{\beta k}{2^{k-1}}C_{\beta,k}
    \ge\max\Bigl\{1,16\Bigl(\frac12+\beta C_1\Bigr)\Bigr\},
\end{align}
and if $m(x)\ge\nicefrac{1}{3}$, we infer from $\lambda\ge C_{\beta,k}N^\gamma$ that
\begin{align}
\label{eq:strong-recovery-nu-sharp}
    \nu(x)
    \ge\nu_\star\doteq\frac{\beta\lambda k}{3^{k-1}}
    \ge\frac{\beta k}{3^{k-1}}C_{\beta,k}N^\gamma
    \ge\max\Bigl\{1,\frac{64}{q_N}\Bigl(\frac12+\beta C_1\Bigr)\Bigr\}.
\end{align}
Here, the last inequalities in \cref{eq:strong-recovery-nu-crude,eq:strong-recovery-nu-sharp} hold since $C_{\beta,k}$ is sufficiently large and, for \cref{eq:strong-recovery-nu-sharp}, since $q_N^{-1}=N^\eta$ with $\eta<\tfrac{k-2}{4}\le\gamma$ and $N$ is sufficiently large. Combining \cref{eq:strong-recovery-nu-crude,eq:strong-recovery-nu-sharp} with \cref{eq:strong-recovery-b-bound}, we conclude that
\begin{align}
\label{eq:strong-recovery-rate-bounds}
\begin{aligned}
    &\nu(x)\ge1
    &&\text{and}\quad b(x)\le\frac{1}{16}\,\nu(x)
    &&\text{whenever }m(x)\ge\frac{C\bar r_N}{2},\\
    &\nu(x)\ge\nu_\star\ge1
    &&\text{and}\quad b(x)\le\frac{q_N}{64}\,\nu(x)
    &&\text{whenever }m(x)\ge\frac13.
\end{aligned}
\end{align}

We now show that the correlation stays above $C\bar r_N/2$ throughout $[0,T]$. Let $\rho\doteq\inf\{t\ge0:m(X_t)\le C\bar r_N/2\}$, and note that $\rho>0$ since $m(X_0)\ge C\bar r_N$. Fix $t\le\rho\wedge T$. Then $m(X_s)\ge C\bar r_N/2$ for all $s\in[0,t]$, and hence, by \cref{eq:strong-recovery-rate-bounds},
\begin{align}
\label{eq:strong-recovery-crude-rate}
    \nu(X_s)\ge1
    \qquad\text{and}\qquad
    b(X_s)\le\frac{1}{16}\,\nu(X_s),
    \qquad0\le s\le t.
\end{align}
We use the two bounds in \cref{eq:strong-recovery-crude-rate} to control the martingale error term and the integral in \cref{eq:strong-recovery-duhamel}, respectively. First, since $\nu(X_s)\ge1$ for $s\in[0,t]$, we have $R_t-R_s\ge t-s$, and hence \cref{eq:strong-recovery-increment-bound} implies
\begin{align}
\label{eq:strong-recovery-increment-clock}
    \bigl|M_t^m-M_s^m\bigr|\le(R_t-R_s+3)\eps_N,
    \qquad0\le s\le t.
\end{align}
Substituting \cref{eq:strong-recovery-increment-clock} into \cref{eq:strong-recovery-martingale-error} and changing variables to $u=R_t-R_s$, for which $du=-\nu(X_s)\,ds$, we obtain
\begin{align}
\label{eq:strong-recovery-error-bound}
\begin{aligned}
    |\xi_t|
    &\le\eps_N\int_0^t\nu(X_s)e^{-(R_t-R_s)}\bigl(R_t-R_s+3\bigr)\,ds+\eps_Ne^{-R_t}(R_t+3)\\
    &=\eps_N\int_0^{R_t}e^{-u}(u+3)\,du+\eps_Ne^{-R_t}(R_t+3)\\
    &\le\eps_N\int_0^\infty e^{-u}(u+3)\,du+\eps_N\sup_{v\ge0}e^{-v}(v+3)
    =4\eps_N+3\eps_N=7\eps_N.
\end{aligned}
\end{align}
The bound on $b$ in \cref{eq:strong-recovery-crude-rate}, together with \cref{eq:strong-recovery-telescoping}, implies the following
\begin{align}
\label{eq:strong-recovery-source-bound}
    \int_0^te^{-(R_t-R_s)}b(X_s)\,ds
    \le\frac{1}{16}\int_0^t\nu(X_s)e^{-(R_t-R_s)}\,ds
    =\frac{1}{16}\bigl(1-e^{-R_t}\bigr).
\end{align}
Substituting \cref{eq:strong-recovery-error-bound,eq:strong-recovery-source-bound} into \cref{eq:strong-recovery-duhamel}, we obtain the following for every $t\le\rho\wedge T$:
\begin{align}
    1-m(X_t)
    &\le e^{-R_t}\bigl(1-m(X_0)\bigr)+\frac{1}{16}\bigl(1-e^{-R_t}\bigr)+7\eps_N
    \label{eq:strong-recovery-crude-deficit}\\
    &\le1-C\bar r_N+7\eps_N<1-\frac{C\bar r_N}{2},
    \label{eq:strong-recovery-no-exit}
\end{align}
where the second inequality uses $1-m(X_0)\le1-C\bar r_N$ and $\nicefrac{1}{16}\le1-C\bar r_N$, and the last inequality uses $7\eps_N\le\tfrac{7}{32}\bar r_N<\tfrac{C\bar r_N}{2}$.
If $\rho\le T$, then $m(X_\rho)=C\bar r_N/2$ by continuity, which contradicts \cref{eq:strong-recovery-no-exit} at $t=\rho$. Hence $\rho>T$, and \cref{eq:strong-recovery-crude-rate,eq:strong-recovery-error-bound,eq:strong-recovery-crude-deficit} hold for every $t\in[0,T]$. In particular, $R_t\ge t$, and \cref{eq:strong-recovery-crude-deficit} implies the following:
\begin{align}
\label{eq:strong-recovery-one-third}
    1-m(X_t)\le e^{-1}+\frac{1}{16}+7\eps_N<\frac23,
    \qquad1\le t\le T.
\end{align}
Consequently, $m(X_t)>\nicefrac{1}{3}$ for every $t\in[1,T]$.

Finally, we show that the deficit falls below $q_N$ by time $t=2$ and stays below $q_N$ up to time $T$. Let $t\in[2,T]$. Since $m(X_s)>\nicefrac{1}{3}$ for all $s\in[1,t]$ by \cref{eq:strong-recovery-one-third}, we infer from \cref{eq:strong-recovery-rate-bounds} that
\begin{align}
\label{eq:strong-recovery-sharp-rate}
    \nu(X_s)\ge\nu_\star
    \qquad\text{and}\qquad
    b(X_s)\le\frac{q_N}{64}\,\nu(X_s),
    \qquad1\le s\le t.
\end{align}
In particular, $R_t-R_1=\int_1^t\nu(X_s)\,ds\ge\nu_\star(t-1)\ge\nu_\star$, since $t\ge2$. To use the bound on $b$ in \cref{eq:strong-recovery-sharp-rate}, we split the integral in \cref{eq:strong-recovery-duhamel} at $s=1$. As in \cref{eq:strong-recovery-telescoping}, the identity $\tfrac{d}{ds}e^{-(R_t-R_s)}=\nu(X_s)e^{-(R_t-R_s)}$ implies
\begin{align}
\label{eq:strong-recovery-split-integrals}
    \int_0^1\nu(X_s)e^{-(R_t-R_s)}\,ds=e^{-(R_t-R_1)}-e^{-R_t},
    \qquad
    \int_1^t\nu(X_s)e^{-(R_t-R_s)}\,ds=1-e^{-(R_t-R_1)}.
\end{align}
Combining \cref{eq:strong-recovery-duhamel} with the bounds on $b$ in \cref{eq:strong-recovery-crude-rate,eq:strong-recovery-sharp-rate}, $1-m(X_0)\le1$, and $|\xi_t|\le7\eps_N\le\tfrac{7q_N}{32}$ from \cref{eq:strong-recovery-error-bound}, we obtain
\begin{align}
\label{eq:strong-recovery-persistence}
\begin{aligned}
    1-m(X_t)
    &=e^{-R_t}\bigl(1-m(X_0)\bigr)
      +\int_0^1e^{-(R_t-R_s)}b(X_s)\,ds
      +\int_1^te^{-(R_t-R_s)}b(X_s)\,ds
      +\xi_t\\
    &\le e^{-R_t}
      +\frac{1}{16}\int_0^1\nu(X_s)e^{-(R_t-R_s)}\,ds
      +\frac{q_N}{64}\int_1^t\nu(X_s)e^{-(R_t-R_s)}\,ds
      +\frac{7q_N}{32}\\
    &=\frac{15}{16}e^{-R_t}
      +\frac{1}{16}e^{-(R_t-R_1)}
      +\frac{q_N}{64}\bigl(1-e^{-(R_t-R_1)}\bigr)
      +\frac{7q_N}{32}\\
    &\le e^{-\nu_\star}+\frac{q_N}{64}+\frac{7q_N}{32}
    <q_N.
\end{aligned}
\end{align}
Here, the second equality follows from \cref{eq:strong-recovery-split-integrals}, the second inequality uses $e^{-R_t}\le e^{-(R_t-R_1)}\le e^{-\nu_\star}$, and the last inequality holds for all sufficiently large $N$, since $\nu_\star\ge\beta k3^{1-k}C_{\beta,k}N^\gamma$ implies $e^{-\nu_\star}=o(N^{-\eta})$.

By \cref{eq:strong-recovery-persistence}, we have $m(X_t)>1-q_N$ for every $t\in[2,T]$ on $\cE\cap\calM$. Taking a union bound and using \cref{eq:strong-recovery-martingale-bound}, we obtain
\begin{align*}
    \Prb\left[\inf_{t\in[2,T]}m(X_t)<1-q_N\right]
    \le\Prb[\cE^c]+\Prb[\calM^c]
    \le\exp(-c_kN)+2\lceil T\rceil\exp\left(-cN\min\{\bar r_N^2,q_N^2\}\right).
\end{align*}
Since $T\le\poly(N)$ and $N\min\{\bar r_N^2,q_N^2\}/\log N\to\infty$ for $s,\eta<\nicefrac{1}{2}$, the proof is concluded by absorbing the factor $\lceil T\rceil$ after adjusting constants appropriately.
\end{proof}

\begin{remark}
\label{rem:strong-recovery-parity-free}
Although this section assumes that $k$ is odd, the proof of \cref{thm:strong-recovery-after-critical-correlation} does not use this assumption. In \cref{subsec:boosting-weak-to-strong-recovery}, we use this fact to show that, for even $k$, spherical Langevin dynamics boosts weak recovery of the planted signal, up to a global sign, to persistent strong recovery.
\end{remark}

A direct application of \cref{thm:tpca-averaged-differential-recovery,thm:strong-recovery-after-critical-correlation} establishes persistent strong recovery for \cref{alg:averaged-annealed-langevin} in the regime $\lambda\gtrsim N^{\tfrac{k-2}{4}+\tfrac{1}{4k}}$, which is above the algorithmic threshold by a factor of $N^{\tfrac{1}{4k}}$. Specifically, we prove the following.

\begin{corollary}[Persistent Strong Recovery Slightly above the Algorithmic Threshold]
\label{cor:averaged-langevin-persistent-strong-recovery}
Suppose $\gamma\ge\tfrac{k-2}{4}+\tfrac{1}{4k}$. Fix $\vartheta\in(0,1)$ and $\eta\in(0,\min\{\nicefrac{1}{2},\tfrac{k-2}{4}\})$. Let $(X_t)_{t\ge0}$ denote the trajectory of the second phase of \cref{alg:averaged-annealed-langevin}, initialized at $\tilX_0\sim\pi$ with inverse temperatures $\beta_0$ and $\beta$ for the two respective phases, where $\beta_0$ and $T_{\mathrm{av}}$ are chosen as in \cref{thm:tpca-averaged-differential-recovery} for this choice of $\vartheta$. Then, for all sufficiently large $N$ and every deterministic $T$ satisfying $2\le T\le\poly(N)$,
\begin{align*}
    \Prb\left[\inf_{t\in[2,T]}m(X_t)\ge1-N^{-\eta}\right]
    \ge1-\exp\left(-c_kN^{\min\{(k+1)/(2k),\,1-2\eta\}}\right).
\end{align*}
\end{corollary}

\begin{proof}
Define $a\doteq\tfrac{k-1}{4k}$. Since $\tfrac{k-2}{4}+\tfrac{1}{4k}=(k-1)a$, we have $\gamma\ge(k-1)a$, and in particular $\gamma-\tfrac{k-1}{4}\ge-a$. Since $\lambda\ge C_{\beta,k}N^\gamma$ and $\sigma_N\asymp_kN^{-\tfrac{k-1}{4}}$ by \cref{thm:avg-spherical-grad-snr}, it follows that $\lambda\sigma_N\gtrsim_kC_{\beta,k}N^{\gamma-\tfrac{k-1}{4}}\ge C_{\beta,k}N^{-a}$.

On the event of \cref{thm:tpca-averaged-differential-recovery}, we have $\Xbar_{T_{\mathrm{av}}}\ne0$, and the initial state of the second phase is $X_0=\widehat X_{T_{\mathrm{av}}}$. Moreover, $|m(x)-m(y)|\le N^{-1/2}\norm{x-y}$ for all $x,y\in\R^N$ by the Cauchy--Schwarz inequality. Hence, we infer from \cref{thm:tpca-averaged-differential-recovery,cor:avg-spherical-grad-correlation} and a union bound that the following holds with probability at least $1-\exp(-c_kN^{(k+1)/(2k)})$, after choosing $C_{\beta,k}$ sufficiently large:
\begin{align}
\label{eq:averaged-langevin-warm-start}
    m(X_0)
    \ge m(\widehat u_H)-\frac{1}{\sqrt N}\norm{\widehat X_{T_{\mathrm{av}}}-\widehat u_H}
    \ge\bigl(c_k-N^{-\vartheta}\bigr)\min\{1,\lambda\sigma_N\}
    \ge N^{-a}.
\end{align}
Here, the second inequality combines the lower bound on $m(\widehat u_H)$ from \cref{cor:avg-spherical-grad-correlation}, where we use $\sigma_N\asymp_kN^{-\tfrac{k-1}{4}}$, with the bound on $\norm{\widehat X_{T_{\mathrm{av}}}-\widehat u_H}$ from \cref{thm:tpca-averaged-differential-recovery}, and the last inequality uses $\lambda\sigma_N\gtrsim_kC_{\beta,k}N^{-a}$.

By construction, $(X_t)_{t\ge0}$ is spherical Langevin dynamics for $H$ at inverse temperature $\beta$, started from $X_0=\widehat X_{T_{\mathrm{av}}}$. As \cref{thm:strong-recovery-after-critical-correlation} allows the initial state to depend on $H$, we may apply it with $s=a$ and $C=1$. Since $\gamma\ge(k-1)a$, the correlation level $\bar r_N=N^{-a}$ satisfies $\lambda\bar r_N^{k-1}\ge C_{\beta,k}$, and $\tau_{\bar r_N}=0$ on the event in \cref{eq:averaged-langevin-warm-start}. Finally, since $N\bar r_N^2=N^{1-2a}=N^{\tfrac{k+1}{2k}}$, the proof is concluded via a union bound after adjusting constants appropriately.
\end{proof}

\subsection{Achieving Critical Recovery above the Algorithmic Threshold from a Sufficiently Warm Start}
\label{subsec:critical-recovery-warm-start}

We now focus our attention on strong recovery in the regime $N^{\tfrac{k-2}{4}}\lesssim\lambda\lesssim N^{\tfrac{k-2}{4}+\tfrac{1}{4k}}$. Specifically, we set $\lambda=C_{\beta,k}N^\gamma$ with $\tfrac{k-2}{4}\le\gamma<\tfrac{k-2}{4}+\tfrac{1}{4k}$, and let ${r_N\doteq\lambda^{-\tfrac{1}{k-1}}\asymp N^{-\tfrac{\gamma}{k-1}}}$ denote the critical correlation parameter. In this regime, \cref{cor:avg-spherical-grad-correlation} suggests that the warm start computed by the first phase of \cref{alg:averaged-annealed-langevin} attains only subcritical correlation, of order $\lambda N^{-\tfrac{k-1}{4}}\ll r_N$.

In view of \cref{thm:strong-recovery-after-critical-correlation}, it suffices to show that the second phase of \cref{alg:averaged-annealed-langevin} attains critical correlation in polynomial time. To this end, we prove a stronger guarantee in \cref{thm:critical-recovery-warm-start}, which establishes that, for any initialization sufficiently close to the normalized average gradient, possibly depending on the disorder, spherical Langevin dynamics attains critical correlation within $O(1)$ time.

For $e\in\SN$, let $(X_t(e))_{t\ge0}$ denote the trajectory of spherical Langevin dynamics for $H$ at inverse temperature $\beta$, initialized at $X_0(e)=e$. As noted in \cref{subsec:spherical-langevin}, this trajectory is a deterministic function of its initial point and its driving Brownian motion, and hence this definition extends to any random initial state $e$, possibly depending on $H$. Since $k$ is odd, we recall from \cref{thm:avg-spherical-grad-snr} that the average spherical gradient $u_H=\E_\pi[\gradsp H]$ is distributed as $\cN(N\lambda\alpha_N\plant,\sigma_N^2I)$, where ${\sigma_N^2=\kappa_N\sqrt N\alpha_N\asymp_kN^{-\tfrac{k-1}{2}}}$.

\begin{theorem}[Critical Recovery from a Sufficiently Warm Start]
\label{thm:critical-recovery-warm-start}
Suppose $\tfrac{k-2}{4}\le\gamma\le\tfrac{k-2}{4}+\tfrac{1}{4k}$. Fix $\vartheta\in(0,1)$, let $u_H=\E_\pi[\gradsp H]$ denote the average spherical gradient, and let ${\widehat u_H=\tfrac{\sqrt N\,u_H}{\norm{u_H}}}$ denote its normalization. For every $e\in\SN$, define the stopping time
\begin{align*}
    T_*(e)\doteq\inf\bigl\{t\ge0:m(X_t(e))\ge Cr_N\bigr\},
\end{align*}
where $C>0$ is a constant. Then, the following holds for all sufficiently large $N$ and any initial state $e$, possibly depending on $H$:
\begin{align*}
    \Prb\left[\frac{1}{\sqrt N}\norm{e-\widehat u_H}\le\lambda\sigma_NN^{-\vartheta}
    \ \ \text{and}\ \
    T_*(e)>C_1N^{-(k-1)\left(\gamma-\tfrac{k-2}{4}\right)}\right]
    \le\exp\bigl(-cN\lambda^2\sigma_N^2\bigr),
\end{align*}
where $C_1,c>0$ are constants depending only on $\beta$ and $k$.
\end{theorem}

The key challenge in establishing \cref{thm:critical-recovery-warm-start} is that the initialization may depend nontrivially on the disorder $H_0$. For $\tfrac{k-2}{4}\le\gamma<\tfrac{k-2}{4}+\tfrac{1}{4k}$, this dependence can significantly complicate the analysis, since the correlation attained by the warm start is too weak to ensure that the contribution of the signal term $N\lambda m^k$ uniformly dominates that of $H_0$. Indeed, as we already discuss in \cref{subsec:sketch-reverse-planting}, one can imagine a potentially dangerous initialization with nontrivial but subcritical correlation that lies close to one of the exponentially many local maxima of $H_0$~\citep{auffinger2013random,benarous2019landscape,ros2019complex}. In the local neighborhood of such an initialization, the drift induced by $H_0$ could dominate that of the signal, thereby trapping the dynamics at polynomial timescales.

Our analysis rules out such a scenario for any initialization sufficiently close to $\widehat u_H$ by proving that the dependence of $u_H$ on the disorder is benign. Indeed, by Stein's lemma (\cref{lem:spherical-stein}), $u_H$ depends on $H_0$ only through the linear component of $H_0$, i.e., its degree-one spherical harmonic projection. Since $H_0$ is an isotropic Gaussian process, a Gaussian conditioning argument shows that $u_H$ is independent of the higher order spherical harmonic projections of $H$, which carry the bulk of the disorder. Our proof exploits this independence through an intricate dynamical change of measure argument, which proves that for any initialization sufficiently close to $\widehat u_H$, the drift contribution of the signal term dominates that of the disorder until the dynamics achieves critical recovery. In \cref{subsec:average-gradient-gauge}, we formalize the independence of $u_H$ from the higher order spherical harmonic projections of $H$ via Gaussian conditioning. \Cref{subsec:reverse-planting} develops the dynamical change of measure argument which we call the reverse planting trick. Finally, we prove \cref{thm:critical-recovery-warm-start} in \cref{subsec:critical-recovery-reverse-planting}.

Our proof of \cref{thm:critical-recovery-warm-start} extends the techniques that Part~I~\citep[Section~6]{das2026averaged} develops to prove the analogous critical recovery guarantee for Glauber dynamics for Ising tensor PCA~\citep[Theorem~6.4]{das2026averaged}. As in Part~I, we combine Gaussian conditioning, a gauge transformation of the warm start, and a reverse planting argument based on adaptive contiguity. The key differences stem from the Riemannian geometry of $\SN$ and the fact that Langevin dynamics is a Markov diffusion process rather than a jump process. To this end, the coordinatewise sign flips of Part~I are replaced by random rotations, and the dynamical change of measure is performed by Girsanov's theorem.

\subsection{The Average Gradient Gauge Transformation}
\label{subsec:average-gradient-gauge}

The following lemma, which is a straightforward consequence of Gaussian conditioning, allows us to control the dependence of $u_H$ on the disorder.

\begin{lemma}[Decomposition of the Spherical $k$-Spin Hamiltonian]
\label{lem:kspin-gaussian-conditioning}
Let $k\ge3$ be odd, and let $\zeta\doteq\E_\pi[\gradsp H_0]$. Then, $H_0(x)=\kappa_N^{-1}\dotp{\zeta}{x}+H_0^\perp(x)$, where $H_0^\perp:\SN\to\R$ is a centered Gaussian process independent of $\zeta$ with
\begin{align*}
    \E\bigl[H_0^\perp(x)H_0^\perp(y)\bigr]
    =N\left(\frac{\dotp{x}{y}}{N}\right)^k-\frac{\sigma_N^2}{\kappa_N^2}\dotp{x}{y},
\end{align*}
where $\sigma_N\asymp_kN^{-\tfrac{k-1}{4}}$ is as defined in \cref{thm:avg-spherical-grad-snr}.
\end{lemma}

\begin{proof}
By Stein's lemma (\cref{lem:spherical-stein}), $\zeta=\kappa_N\E_{y\sim\pi}[yH_0(y)]$. Fix $x\in\SN$. For every $v\in\bS^{N-1}$ orthogonal to $x$, the second claim of \cref{lem:sphere-moments}, applied with $(w,z)=(v,x/\sqrt N)$, $g(t)=t$, and $h(t)=t^k$, yields $\E_{y\sim\pi}[\dotp{v}{y}\dotp{x}{y}^k]=0$. Hence, $\E_{y\sim\pi}[y\dotp{x}{y}^k]$ is parallel to $x$. Taking the inner product with $x$, we obtain the following from \cref{eq:kspin-covariance} and \cref{eq:average-gradient-signal,eq:average-gradient-overlap-moment} in the proof of \cref{thm:avg-spherical-grad-snr}:
\begin{align*}
    \E[H_0(x)\zeta]
    &=\kappa_N\E_{y\sim\pi}\bigl[y\,\E[H_0(x)H_0(y)]\bigr]
    =\kappa_NN\,\E_{y\sim\pi}\left[y\left(\frac{\dotp{x}{y}}{N}\right)^k\right]\\
    &=\kappa_NN\,\E_{y\sim\pi}\left[\left(\frac{\dotp{x}{y}}{N}\right)^{k+1}\right]x
    =\kappa_NN\,\E_{y\sim\pi}\bigl[m(y)^{k+1}\bigr]\,x
    =\sqrt N\alpha_N\,x
    =\frac{\sigma_N^2}{\kappa_N}x,
\end{align*}
where the last equality uses $\sigma_N^2=\kappa_N\sqrt N\alpha_N$, as stated in \cref{thm:avg-spherical-grad-snr}.

Since $k$ is odd, \cref{thm:avg-spherical-grad-snr} implies that $\sigma_N^2>0$. Since $(H_0(x))_{x\in\SN}$ and $\zeta$ are jointly Gaussian and $\E[\zeta\zeta^\top]=\sigma_N^2I$, the Gaussian conditioning formula~\citep[Chapter~9]{janson1997gaussian} implies that, for every $x\in\SN$, ${\E[H_0(x)\mid\zeta]}=\kappa_N^{-1}\dotp{\zeta}{x}$, and that $H_0^\perp(x)\doteq H_0(x)-\kappa_N^{-1}\dotp{\zeta}{x}$ is a centered Gaussian process independent of $\zeta$. Moreover, since $H_0^\perp(x)=H_0(x)-\kappa_N^{-1}\dotp{\zeta}{x}$, we obtain the following:
\begin{align*}
    \E\bigl[H_0^\perp(x)H_0^\perp(y)\bigr]
    &=\E\Bigl[\bigl(H_0(x)-\kappa_N^{-1}\dotp{\zeta}{x}\bigr)\bigl(H_0(y)-\kappa_N^{-1}\dotp{\zeta}{y}\bigr)\Bigr]\\
    &=\E[H_0(x)H_0(y)]-\frac{\dotp{y}{\E[H_0(x)\zeta]}}{\kappa_N}-\frac{\dotp{x}{\E[H_0(y)\zeta]}}{\kappa_N}+\frac{x^\top\E[\zeta\zeta^\top]y}{\kappa_N^2}\\
    &=N\left(\frac{\dotp{x}{y}}{N}\right)^k-\frac{2\sigma_N^2}{\kappa_N^2}\dotp{x}{y}+\frac{\sigma_N^2}{\kappa_N^2}\dotp{x}{y}\\
    &=N\left(\frac{\dotp{x}{y}}{N}\right)^k-\frac{\sigma_N^2}{\kappa_N^2}\dotp{x}{y},
\end{align*}
which concludes the proof.
\end{proof}

Since $\E_\pi[xx^\top]=I$, Stein's lemma (\cref{lem:spherical-stein}) shows that $\kappa_N^{-1}\dotp{\zeta}{x}=\dotp{\E_{y\sim\pi}[yH_0(y)]}{x}$ is the linear component of $H_0$, i.e., its degree-one spherical harmonic projection. By \cref{thm:avg-spherical-grad-snr}, we have $u_H = \E_\pi[\gradsp H]=N\lambda\alpha_N\plant+\zeta$, and hence \cref{lem:kspin-gaussian-conditioning} shows that $u_H$ is independent of $H_0^\perp$, which collects the higher order spherical harmonic projections of $H_0$ and comprises the bulk of the disorder.

However, the initialization considered in \cref{thm:critical-recovery-warm-start} still depends on the disorder through $u_H$. To simplify the analysis, we apply a randomized gauge transformation, which we call the \emph{average gradient gauge transformation}. This transformation maps the normalized average gradient ${\widehat u_H=\tfrac{\sqrt N\,u_H}{\norm{u_H}}}$ to the fixed state $\sqrt Ne_1$, so that the disorder dependence of the initialization is transferred to the gauge-transformed Hamiltonian. Our approach is motivated by the Nishimori gauge transformation in the spin glass literature~\citep{nishimori2001statistical}, except that we map the initial state of Langevin dynamics, rather than the planted signal, to $\sqrt Ne_1$. Precisely, we define the following random variables:
\begin{align}
\label{eq:gauge-variables}
    R\doteq\norm{u_H},
    \qquad
    \Phi\doteq\Psi\Lambda(u_H),
    \qquad
    \phi\doteq \Phi\plant,
\end{align}
where $\Psi$ is distributed according to the Haar measure~\citep[cf.][Theorem~1.4]{meckes2019random} on the subgroup ${\{S\in O(N,\R):Se_1=e_1\}}\cong O(N-1,\R)$, independently of all other randomness, and $\Lambda(u)$ is the Householder reflection map, defined as follows for every $u\in\R^N\setminus\{0\}$:
\begin{align*}
    \Lambda(u)\doteq
    \begin{cases}
        I-\dfrac{2(u-\norm{u}e_1)(u-\norm{u}e_1)^\top}{\norm{u-\norm{u}e_1}^2}, & \text{if $u\ne\norm{u}e_1$},\\
        I, & \text{if $u=\norm{u}e_1$}.
    \end{cases}
\end{align*}
By definition, $\Psi e_1=e_1$ and the Householder reflection map satisfies $\Lambda(u)u=\norm{u}e_1$ for every nonzero $u \in \R^N$. Replacing $u$ with $cu$ for any $c>0$, we note that $\Lambda(u)$ depends on $u$ only through its direction $u/\norm{u}$. Consequently, $\Phi u_H=Re_1$ and $\Phi\widehat u_H=\sqrt Ne_1\in\SN$. Furthermore, since $\zeta\sim\cN(0,\sigma_N^2I)$ with $\sigma_N^2>0$, $u_H$ is almost surely nonzero, and hence these random variables are almost surely well defined. As we shall see in \cref{subsec:reverse-planting}, the randomization over orthogonal transformations fixing $e_1$ ensures that, when $\lambda=0$, $\phi$ is uniformly distributed on $\bS^{N-1}$ and independent of $R$. Finally, we formally define the average gradient gauge transformation and the associated gauge-transformed dynamics as follows:
\begin{align}
\label{eq:gauge-transformation}
    y=y(x)\doteq \Phi x,
    \qquad
    Y_t(e)\doteq \Phi X_t(\Phi^\top e).
\end{align}
Observe that $y(\widehat u_H)=\sqrt Ne_1$ and $Y_t(\sqrt Ne_1)=\Phi X_t(\widehat u_H)$. The following lemma analyzes the effect of the gauge transformation on the tensor PCA Hamiltonian.

\begin{lemma}[Gauge-Transformed Hamiltonian]
\label{lem:tpca-hamiltonian-under-gauge}
Let $m_\phi(y)\doteq m(\Phi^\top y)$ and $H_\phi(y)\doteq H(\Phi^\top y)$ denote the gauge-transformed correlation and the gauge-transformed tensor PCA Hamiltonian, respectively. Then, $m_\phi(y)=\tfrac{1}{\sqrt N}\dotp{\phi}{y}$ and $H_\phi=H_{\mathrm{null}}+H_{\mathrm{sig}}$, where
\begin{align*}
    H_{\mathrm{null}}(y)
    &\doteq\frac{R}{\kappa_N}\dotp{e_1}{y}+W(y),
    \qquad
    W(y)\doteq H_0^\perp(\Phi^\top y),\\
    H_{\mathrm{sig}}(y)
    &\doteq N\lambda m_\phi(y)^k-\frac{N\lambda\sigma_N^2}{\kappa_N^2}m_\phi(y).
\end{align*}
Moreover, the following hold:
\begin{enumerate}
    \item The Gaussian process $W$ is independent of $(\Phi,R)$ and has the same law as $H_0^\perp$.
    \item For any initial state $e\in\SN$, possibly depending on $(\Phi,R,W)$, the trajectory $(Y_t(e))_{t\ge0}$ is spherical Langevin dynamics for the gauge-transformed Hamiltonian $H_\phi$ at inverse temperature $\beta$, initialized at $Y_0(e)=e$.
\end{enumerate}
\end{lemma}

\begin{proof}
Since $\Phi$ is orthogonal and $\phi=\Phi\plant$, by direct computation,
\begin{align*}
    m_\phi(y)
    =m(\Phi^\top y)
    =\frac{1}{\sqrt N}\dotp{\plant}{\Phi^\top y}
    =\frac{1}{\sqrt N}\dotp{\Phi\plant}{y}
    =\frac{1}{\sqrt N}\dotp{\phi}{y}.
\end{align*}
Since $k$ is odd, we recall from \cref{thm:avg-spherical-grad-snr} that $u_H=N\lambda\alpha_N\plant+\zeta$, where ${\sigma_N^2=\kappa_N\sqrt N\alpha_N\asymp_kN^{-\tfrac{k-1}{2}}}$. Using $\dotp{\plant}{x}=\sqrt N\,m(x)$, we then infer from \cref{eq:tensor-pca-overlap-noise,lem:kspin-gaussian-conditioning} that
\begin{align*}
    H(x)
    &=N\lambda m(x)^k+\kappa_N^{-1}\dotp{\zeta}{x}+H_0^\perp(x)\\
    &=\kappa_N^{-1}\dotp{u_H}{x}+N\lambda m(x)^k-\frac{N\lambda\sigma_N^2}{\kappa_N^2}m(x)+H_0^\perp(x).
\end{align*}
Moreover, since $\Phi u_H=Re_1$,
\begin{align*}
    \kappa_N^{-1}\dotp{u_H}{\Phi^\top y}
    =\kappa_N^{-1}\dotp{\Phi u_H}{y}
    =\frac{R}{\kappa_N}\dotp{e_1}{y}.
\end{align*}
Consequently,
\begin{align*}
    H_\phi(y)=H(\Phi^\top y)
    &=\frac{R}{\kappa_N}\dotp{e_1}{y}+H_0^\perp(\Phi^\top y)+N\lambda m_\phi(y)^k-\frac{N\lambda\sigma_N^2}{\kappa_N^2}m_\phi(y)\\
    &=H_{\mathrm{null}}(y)+H_{\mathrm{sig}}(y).
\end{align*}

We now prove the first claim. Since $\plant$ is fixed, \cref{eq:gauge-variables} shows that $(\Phi,R)$ is a measurable function of $(\zeta,\Psi)$. Moreover, by \cref{lem:kspin-gaussian-conditioning}, $H_0^\perp$ is a centered Gaussian process independent of $\zeta$, and $\Psi$ is independent of $(\zeta,H_0^\perp)$ by construction. Since the map $y\mapsto \Phi^\top y$ is an isometry of $\SN$ for every fixed $\Phi\in O(N,\R)$ and the covariance of $H_0^\perp$ depends only on $\dotp{x}{y}$, the law of $H_0^\perp$ is invariant under this transformation. Hence,
\begin{align*}
    \Law(W\mid \Phi,R)
    =\Law\bigl(H_0^\perp(\Phi^\top\cdot)\mid \Phi,R\bigr)
    =\Law(H_0^\perp\mid \Phi,R)
    =\Law(H_0^\perp).
\end{align*}
Thus, $W$ is independent of $(\Phi,R)$ and has the same law as $H_0^\perp$.

The second claim is a consequence of the rotation invariance of spherical Langevin dynamics. Let $(B_t)_{t\ge0}$ denote the driving Brownian motion of the trajectories $(X_t(\cdot))_{t\ge0}$, and write $X_t\doteq X_t(\Phi^\top e)$ and $Y_t\doteq Y_t(e)=\Phi X_t$. Recall that $X_t$ solves \cref{eq:spherical-langevin-sde} for the Hamiltonian $H$ with initial point $X_0=\Phi^\top e$. Since $\Phi$ is fixed at time zero, multiplying \cref{eq:spherical-langevin-sde} by $\Phi$ and using $X_t=\Phi^\top Y_t$, we obtain the following:
\begin{align*}
    dY_t=\Phi\,dX_t
    &=\Bigl(\beta\cdot \Phi Q_{X_t}\grad H(X_t)-\frac{N-1}{2N}\Phi X_t\Bigr)\,dt+\Phi Q_{X_t}\,dB_t\\
    &=\Bigl(\beta\cdot\bigl(\Phi Q_{X_t}\Phi^\top\bigr)\bigl(\Phi\grad H(\Phi^\top Y_t)\bigr)-\frac{N-1}{2N}Y_t\Bigr)\,dt+\bigl(\Phi Q_{X_t}\Phi^\top\bigr)\,\Phi\,dB_t\\
    &=\Bigl(\beta\cdot Q_{Y_t}\grad H_\phi(Y_t)-\frac{N-1}{2N}Y_t\Bigr)\,dt+Q_{Y_t}\,d(\Phi B_t),
\end{align*}
where the last equality uses $\Phi Q_x\Phi^\top=Q_{\Phi x}$ for every $x\in\R^N\setminus\{0\}$, and $\grad H_\phi(y)=\Phi\grad H(\Phi^\top y)$ by the chain rule. Moreover, $Y_0=e$, and since $(B_t)_{t\ge0}$ is independent of $(\Phi,R,W,e)$, the rotation invariance of Brownian motion~\citep[cf.][Theorem~4.1.1]{stroock2007multidimensional} implies that $(\Phi B_t)_{t\ge0}$ is a standard Brownian motion independent of $(\Phi,R,W,e)$. Hence, $(Y_t)_{t\ge0}$ is spherical Langevin dynamics for $H_\phi$ at inverse temperature $\beta$, initialized at $Y_0=e$, with driving Brownian motion $(\Phi B_t)_{t\ge0}$.
\end{proof}

\Cref{lem:tpca-hamiltonian-under-gauge} shows that, modulo the lower order linear terms $\kappa_N^{-1}R\dotp{e_1}{y}$ and $-N\lambda\kappa_N^{-2}\sigma_N^2m_\phi(y)$, the average gradient gauge transformation turns $H$ into a tensor PCA Hamiltonian with planted vector $\sqrt N\phi$ and disorder $W$, which is independent of $(\Phi,R)$. The decomposition $H_\phi=H_{\mathrm{null}}+H_{\mathrm{sig}}$ is constructed so that the entire dependence of $H_\phi$ on $\phi$ resides in $H_{\mathrm{sig}}$. To this end, $H_{\mathrm{null}}$ contains the linear component of $H$ (which takes the form $\kappa_N^{-1}R\dotp{e_1}{y}$ after the gauge transformation), including that of the signal.

\subsection{The Reverse Planting Trick and Adaptive Contiguity}
\label{subsec:reverse-planting}

We use $\cL_{\mathrm{pl}}\doteq\cL_{H_\phi,\beta}$ and $\cL_{\mathrm{null}}\doteq\cL_{H_{\mathrm{null}},\beta}$ to denote the generators of spherical Langevin dynamics at inverse temperature $\beta$ for $H_\phi$ and $H_{\mathrm{null}}$, respectively. By Dynkin's formula, the quantity $(\cL_{\mathrm{null}}m_\phi)(Y_t(e))$ encapsulates the contribution of the disorder to the time evolution of the correlation $m_\phi$ along the Langevin trajectory. Hence, in order to show that the signal contribution dominates that of the disorder until the dynamics achieves critical recovery, our analysis needs to control $(\cL_{\mathrm{null}}m_\phi)(Y_t(e))$ along the gauge-transformed trajectory. Explicitly, since $m_\phi$ is linear, \cref{eq:spherical-derivatives} implies $\gradsp m_\phi(y)=N^{-1/2}Q_y\phi$ and $\lapsp m_\phi=-\kappa_Nm_\phi$, and hence we obtain the following from \cref{eq:langevin-generator} for every $y\in\SN$:
\begin{align*}
    (\cL_{\mathrm{null}}m_\phi)(y)
    =-\frac{\kappa_N}{2}m_\phi(y)+\beta\dotp{\gradsp H_{\mathrm{null}}(y)}{\gradsp m_\phi(y)}
    =\frac{1}{\sqrt N}\dotp{\phi}{\beta\gradsp H_{\mathrm{null}}(y)-\frac{\kappa_N}{2}y}.
\end{align*}
Although the law of $\phi$ is invariant under all rotations fixing $e_1$, the above quantity is hard to control since $\phi$ is correlated with $\gradsp H_{\mathrm{null}}(y)=\kappa_N^{-1}RQ_ye_1+\gradsp W(y)$ through $R$ as per \cref{eq:gauge-variables}. Moreover, the trajectory $Y_t(e)$ along which it is evaluated depends nontrivially on $\phi$ through $H_{\mathrm{sig}}$, and its initial state $e$ may also depend on $\phi$.

We circumvent these difficulties via a change of measure to a carefully constructed null law, under which $\phi$ is uniformly distributed on $\bS^{N-1}$ and independent of the dynamics initialized at $\sqrt Ne_1$.

Since the initialization considered in \cref{thm:critical-recovery-warm-start} may depend on the disorder, we fix an \emph{initialization rule}, i.e., a random initial state of the form $e=e(\Phi,R,W,\omega)\in\SN$, where $e$ is a measurable function and $\omega$ is auxiliary randomness whose law is the same under the planted and null laws below, independent of $(\Phi,R,W)$ and of the driving Brownian motion of the dynamics.

We formally define both the law of the true dynamics (which we call the planted law) and this null law, representing them as laws of spherical Langevin dynamics in the random environment $(\Phi,R,W)$.

\begin{definition}[Planted Law]
\label{def:planted-law}
The planted law $\Prb$ is the law of spherical Langevin dynamics in a random environment, which factorizes as $\Prb=\Prb_{\mathrm{env}}(\Phi,R,W)\cdot\Prb_{\mathrm{dyn}}(\cdot\mid \Phi,R,W)$, where $\Prb_{\mathrm{env}}$ and $\Prb_{\mathrm{dyn}}$ are defined as follows:
\begin{enumerate}
    \item $R=\norm{U}$ and $\Phi=\Psi\Lambda(U)$, where $U\sim\cN(N\lambda\alpha_N\plant,\sigma_N^2I)$, and $\Psi$ is distributed according to the Haar measure on $\{S\in O(N,\R):Se_1=e_1\}$, independently of $U$.
    \item $W$ is independent of $(\Phi,R)$ and has the same law as $H_0^\perp$.
    \item Conditional on the environment $(\Phi,R,W)$, $\Prb_{\mathrm{dyn}}$ is the law of spherical Langevin dynamics at inverse temperature $\beta$ for the Hamiltonian
    \begin{align*}
        H_\phi(y)=\kappa_N^{-1}R\dotp{e_1}{y}+W(y)+N\lambda m_\phi(y)^k-N\lambda\kappa_N^{-2}\sigma_N^2m_\phi(y),
        \qquad
        \phi=\Phi\plant,
    \end{align*}
    initialized at $Y_0=e(\Phi,R,W,\omega)$. We use $\cL_{\mathrm{pl}}$ to denote its generator.
\end{enumerate}
\end{definition}

\begin{definition}[Null Law]
\label{def:null-law}
The null law $\Q$ is the law of spherical Langevin dynamics in a random environment, which factorizes as $\Q=\Q_{\mathrm{env}}(\Phi,R,W)\cdot\Q_{\mathrm{dyn}}(\cdot\mid \Phi,R,W)$, where $\Q_{\mathrm{env}}$ and $\Q_{\mathrm{dyn}}$ are defined as follows:
\begin{enumerate}
    \item $R=\norm{U}$ and $\Phi=\Psi\Lambda(U)$, where $U\sim\cN(0,\sigma_N^2I)$, and $\Psi$ is distributed according to the Haar measure on $\{S\in O(N,\R):Se_1=e_1\}$, independently of $U$.
    \item $W$ is independent of $(\Phi,R)$ and has the same law as $H_0^\perp$.
    \item Conditional on the environment $(\Phi,R,W)$, $\Q_{\mathrm{dyn}}$ is the law of spherical Langevin dynamics at inverse temperature $\beta$ for the Hamiltonian $H_{\mathrm{null}}(y)=\kappa_N^{-1}R\dotp{e_1}{y}+W(y)$, initialized at ${Y_0=e(\Phi,R,W,\omega)}$. We use $\cL_{\mathrm{null}}$ to denote its generator.
\end{enumerate}
\end{definition}

Under both laws, we write $Y_t=Y_t(e)$ for the associated trajectory. By \cref{thm:avg-spherical-grad-snr,lem:tpca-hamiltonian-under-gauge}, $\Prb$ is the law of the true dynamics in the gauge-transformed coordinates. The two laws are best understood by inverting the gauge transformation, which, for simplicity, we do only for $e=\sqrt Ne_1$. Under $\Prb$, $\Phi^\top Y_t(\sqrt Ne_1)$ is spherical Langevin dynamics for $H$ initialized at its normalized average gradient $\widehat u_H=\sqrt N\Phi^\top e_1$. Similarly, under $\Q$, $\Phi^\top Y_t(\sqrt Ne_1)$ is spherical Langevin dynamics for the Hamiltonian
\begin{align*}
    \widetilde H(x)
    \doteq H_{\mathrm{null}}(\Phi x)
    =\kappa_N^{-1}\dotp{\widetilde\zeta}{x}+W(\Phi x),
    \qquad
    \widetilde\zeta\doteq R\Phi^\top e_1,
\end{align*}
initialized at $\sqrt N\Phi^\top e_1=\sqrt N\widetilde\zeta/\norm{\widetilde\zeta}$. Since $\Phi U=\Psi\Lambda(U)U=Re_1$, we have $\widetilde\zeta=U\sim\cN(0,\sigma_N^2I)$ under $\Q$. Moreover, since $W$ is independent of $(\Phi,R)$, $W(\Phi\,\cdot)$ is independent of $\widetilde\zeta$, and is identically distributed as $H_0^\perp$ by isotropy. Applying \cref{lem:kspin-gaussian-conditioning}, we observe that $\Law(\widetilde H,\widetilde\zeta)=\Law(H_0,\zeta)$, where $\zeta=\E_\pi[\gradsp H_0]$, and consequently, $\widetilde\zeta$ is the average spherical gradient of $\widetilde H$.

Thus, we conclude that both the planted and null laws describe spherical Langevin dynamics initialized at the normalized average gradient of their respective underlying Hamiltonians, which correspond to the true tensor PCA Hamiltonian under $\Prb$ and the pure disorder Hamiltonian under $\Q$. Under this setup, $\Q$ corresponds precisely to $\Prb$ with $\lambda=0$, as $U$ is then centered and $H_{\mathrm{sig}}$ vanishes. In particular, the null dynamics carries no information about $\plant$.

Moreover, the randomization over $\Psi$ decouples the null environment, in the sense that under the null law, the gauge rotation $\Phi$ is independent of $(R,W)$, which determine the null Hamiltonian $H_{\mathrm{null}}$. This removes the first difficulty discussed above, namely the correlation between $\phi$ and $\gradsp H_{\mathrm{null}}$ through $R$ under the planted law.

\begin{lemma}[Decoupling of the Null Environment]
\label{lem:null-environment-decoupling}
Under the null law $\Q$, $\Phi$ is distributed according to the Haar measure on $O(N,\R)$, and the random variables $\Phi$, $R$ and $W$ are independent. In particular, $\phi=\Phi\plant$ is uniformly distributed on $\bS^{N-1}$ and independent of $(R,W)$.
\end{lemma}

\begin{proof}
Throughout, we work under $\Q$. Since $U\sim\cN(0,\sigma_N^2I)$, its direction $U/R$ is uniformly distributed on $\bS^{N-1}$ and independent of $R$~\citep[Exercise~3.3.7]{vershynin2018high}. Moreover, as noted in \cref{subsec:average-gradient-gauge}, $\Lambda(U)$ depends on $U$ only through $U/R$, and hence $\Phi=\Psi\Lambda(U)$ is a function of $(\Psi,U/R)$, which is independent of $R$ since $\Psi$ is independent of $U$. Since $\Lambda(U)$ is symmetric and orthogonal with $\Lambda(U)(U/R)=e_1$, its first column is $\Lambda(U)e_1=U/R$. Consequently, $\Phi^\top=\Lambda(U)\Psi^\top$ is the product of a random orthogonal matrix whose first column is uniformly distributed on $\bS^{N-1}$, and an independent matrix $\Psi^\top$ distributed according to the Haar measure on $\{S\in O(N,\R):Se_1=e_1\}\cong O(N-1,\R)$. By the inductive construction of the Haar measure on $O(N,\R)$~\citep[cf.][Section~1.2]{meckes2019random}, $\Phi$ is distributed according to the Haar measure on $O(N,\R)$.

Hence, we conclude that $\Phi$ is Haar distributed and independent of $R$, and since $W$ is independent of $(\Phi,R)$ by \cref{def:null-law}, the random variables $\Phi$, $R$ and $W$ are independent. Finally, since ${\plant\in\bS^{N-1}}$ is a fixed vector, $\phi=\Phi\plant$ is uniformly distributed on $\bS^{N-1}$~\citep[cf.][Section~1.2]{meckes2019random}. Moreover, since $\Phi$ is independent of $(R,W)$, so is $\phi$.
\end{proof}

By \cref{lem:null-environment-decoupling}, $\phi$ is independent of the null dynamics initialized at $\sqrt Ne_1$, which depends on the environment only through $H_{\mathrm{null}}$. This independence yields the following uniform concentration guarantee for $\cL_{\mathrm{null}}m_\phi$, which extends to nearby initializations by driving the null trajectories with common random rotations.

\begin{lemma}[Uniform Concentration along the Null Trajectory]
\label{lem:null-trajectory-uniform-concentration}
For every $C_0>0$, there exist constants $c_*,C_*>0$, depending only on $\beta$ and $C_0$, such that the following holds for every $T\in(0,1]$, every $\delta\ge0$, and every initialization rule $e$, possibly randomized:
\begin{align*}
    \E_{\Q}\left[\exp\left(c_*N\sup_{t\in[0,T]}\bigl|(\cL_{\mathrm{null}}m_\phi)(Y_t(e))\bigr|^2\right)\ind{\norm{H_{\mathrm{null}}}_{\cG^4}\le C_0N,\,\norm{e-\sqrt Ne_1}\le\delta\sqrt N}\right]
    \le C_*\exp\bigl(C_*N\delta^2\bigr).
\end{align*}
\end{lemma}

\begin{proof}
Throughout, we work under $\Q$ and write $\cE_{\mathrm{reg}}\doteq\{\norm{H_{\mathrm{null}}}_{\cG^4}\le C_0N\}$. The proof proceeds in two steps. The first step establishes the bound for the null trajectory initialized at $\sqrt Ne_1$, which is independent of the uniformly random direction $\phi$. The second step transfers this bound to the null trajectory initialized at $e$, which may depend on $\phi$, by coupling the two trajectories through common random rotations.

We begin with uniform bounds on $\cE_{\mathrm{reg}}$, which are used in both steps.

Since $\norm{\beta H_{\mathrm{null}}}_{\cG^4}\le\beta C_0N$ on $\cE_{\mathrm{reg}}$, \cref{lem:ladder-relations} with $g_1=\beta H_{\mathrm{null}}$, together with \cref{lem:overlap-sobolev-bounds} applied to $m_\phi$, implies that ${\norm{\cL_{\mathrm{null}}m_\phi}_{\cG^1}\le K}$ and ${\norm{\cL_{\mathrm{null}}^2m_\phi}_{\cG^0}\le K}$ on $\cE_{\mathrm{reg}}$ for some constant $K>0$ depending only on $\beta$ and $C_0$.

Moreover, as $m_\phi$ is linear and $\cL_{\mathrm{null}}$ acts entrywise on vector-valued observables, we infer from \cref{eq:spherical-derivatives,eq:langevin-generator} that the following holds for every $y\in\SN$:
\begin{align}
\label{eq:null-drift-linear}
\begin{aligned}
    (\cL_{\mathrm{null}}m_\phi)(y)&=\frac{1}{\sqrt N}\dotp{\phi}{F(y)},
    &\qquad
    F(y)&\doteq\beta\gradsp H_{\mathrm{null}}(y)-\frac{\kappa_N}{2}y,\\
    (\cL_{\mathrm{null}}^2m_\phi)(y)&=\frac{1}{\sqrt N}\dotp{\phi}{(\cL_{\mathrm{null}}F)(y)},
    &\qquad
    \gradsp(\cL_{\mathrm{null}}m_\phi)(y)&=\frac{1}{\sqrt N}\gradsp F(y)\phi.
\end{aligned}
\end{align}
Since the above bounds hold with any unit vector in place of $\phi$, taking the supremum over $\phi\in\bS^{N-1}$ in \cref{eq:null-drift-linear} and using \cref{def:spherical-sobolev-norm}, we obtain the following on $\cE_{\mathrm{reg}}$ for every $y\in\SN$:
\begin{align}
\label{eq:null-drift-vector-bounds}
    \norm{F(y)}\le K\sqrt N,
    \qquad
    \norm{(\cL_{\mathrm{null}}F)(y)}\le K\sqrt N,
    \qquad
    \norm{\gradsp F(y)}_{\mathrm{op}}\le K.
\end{align}

For the first step, consider the null trajectory $(Y_t(\sqrt Ne_1))_{t\ge0}$ initialized at $\sqrt Ne_1$, and let $(B_t)_{t\ge0}$ denote its driving Brownian motion, which is independent of the environment. By \cref{lem:null-environment-decoupling}, $\phi$ is uniformly distributed on $\bS^{N-1}$ and independent of $(R,W,B)$, whereas $\cE_{\mathrm{reg}}$ and $(Y_t(\sqrt Ne_1))_{t\ge0}$ are determined by $(R,W,B)$, since $H_{\mathrm{null}}$ depends on the environment only through $(R,W)$. Since the one-dimensional marginals of $\sqrt N\phi\sim\pi$ are uniformly subgaussian~\citep[cf.][Theorem~3.4.6]{vershynin2018high}, this independence yields the following for every $(R,W,B)$-measurable random vector $v\in\R^N$ satisfying $\norm{v}\le K\sqrt N$ on $\cE_{\mathrm{reg}}$, where $c>0$ is a sufficiently small universal constant:
\begin{align}
\label{eq:sphere-projection-conditional}
    \E_\Q\left[\exp\left(\frac{c\dotp{\phi}{v}^2}{K^2}\right)\ind{\cE_{\mathrm{reg}}}\right]
    \le2.
\end{align}

To control the supremum over $t\in[0,T]$, let $M_t$ denote the Dynkin martingale of $\cL_{\mathrm{null}}m_\phi$ along $Y_t(\sqrt Ne_1)$. Applying \cref{lem:langevin-dynkin} conditionally on the environment, we obtain the following for every $t\ge0$:
\begin{align}
    (\cL_{\mathrm{null}}m_\phi)(Y_t(\sqrt Ne_1))
    &=(\cL_{\mathrm{null}}m_\phi)(\sqrt Ne_1)+\int_0^t(\cL_{\mathrm{null}}^2m_\phi)(Y_s(\sqrt Ne_1))\,ds+M_t,\nonumber\\
    \sup_{t\in[0,T]}\bigl|(\cL_{\mathrm{null}}m_\phi)(Y_t(\sqrt Ne_1))\bigr|
    &\le\bigl|(\cL_{\mathrm{null}}m_\phi)(\sqrt Ne_1)\bigr|+\int_0^T\bigl|(\cL_{\mathrm{null}}^2m_\phi)(Y_s(\sqrt Ne_1))\bigr|\,ds+\sup_{t\in[0,T]}|M_t|.
    \label{eq:null-center-decomposition}
\end{align}
By \cref{eq:null-drift-linear}, the first term and the integrand of the second term on the right-hand side of \cref{eq:null-center-decomposition} are projections of vectors that are determined by $(R,W,B)$ onto the uniformly random direction $\phi$, modulo the normalizing factor $N^{-1/2}$. Moreover, the last term is the supremum of a martingale with quadratic variation at most $K^2T/N$ on $\cE_{\mathrm{reg}}$. Hence, all three terms are of order $N^{-1/2}$. For the first term, we apply \cref{eq:sphere-projection-conditional} with $v=F(\sqrt Ne_1)$ to conclude
\begin{align}
\label{eq:null-center-first-term}
    \E_\Q\left[\exp\left(\frac{cN}{K^2}\bigl|(\cL_{\mathrm{null}}m_\phi)(\sqrt Ne_1)\bigr|^2\right)\ind{\cE_{\mathrm{reg}}}\right]
    \le2.
\end{align}
For the second term, the Cauchy--Schwarz inequality and Jensen's inequality imply
\begin{align}
\label{eq:null-center-jensen}
\begin{aligned}
    \exp\left(\frac{cN}{K^2}\left(\int_0^T\bigl|(\cL_{\mathrm{null}}^2m_\phi)(Y_s(\sqrt Ne_1))\bigr|\,ds\right)^2\right)
    &\le\exp\left(\frac{cNT}{K^2}\int_0^T(\cL_{\mathrm{null}}^2m_\phi)(Y_s(\sqrt Ne_1))^2\,ds\right)\\
    &\le\frac{1}{T}\int_0^T\exp\left(\frac{cNT^2}{K^2}(\cL_{\mathrm{null}}^2m_\phi)(Y_s(\sqrt Ne_1))^2\right)ds.
\end{aligned}
\end{align}
Taking expectations on both sides and applying \cref{eq:sphere-projection-conditional} with ${v=(\cL_{\mathrm{null}}F)(Y_s(\sqrt Ne_1))}$ for each $s\in[0,T]$ and using $T\le1$, we conclude:
\begin{align}
\label{eq:null-center-second-term}
    \E_\Q\left[\exp\left(\frac{cN}{K^2}\left(\int_0^T\bigl|(\cL_{\mathrm{null}}^2m_\phi)(Y_s(\sqrt Ne_1))\bigr|\,ds\right)^2\right)\ind{\cE_{\mathrm{reg}}}\right]
    \le2.
\end{align}
For the third term, since $\norm{\gradsp(\cL_{\mathrm{null}}m_\phi)}\le K/\sqrt N$ on $\cE_{\mathrm{reg}}$, \cref{lem:dynkin-martingale-concentration} applied conditionally on the environment implies the following for every $r>0$ on $\cE_{\mathrm{reg}}$:
\begin{align}
\label{eq:null-center-martingale-tail}
    \Q\left[\sup_{t\in[0,T]}|M_t|\ge r\,\middle|\,\Phi,R,W\right]
    \le2\exp\left(-\frac{Nr^2}{2K^2T}\right).
\end{align}
Since $T\le1$ and $c$ is sufficiently small, integrating the tail bound in \cref{eq:null-center-martingale-tail} conditionally on the environment~\citep[cf.][Proposition~2.5.2]{vershynin2018high}, we conclude:
\begin{align}
\label{eq:null-center-term-bounds}
    \E_\Q\left[\exp\left(\frac{cN}{K^2}\sup_{t\in[0,T]}|M_t|^2\right)\ind{\cE_{\mathrm{reg}}}\right]
    \le2.
\end{align}
Let $a_1$, $a_2$ and $a_3$ denote the three terms on the right-hand side of \cref{eq:null-center-decomposition}. Then
\begin{align}
\label{eq:null-center-mgf}
\begin{aligned}
    &\E_\Q\left[\exp\left(\frac{cN}{9K^2}\sup_{t\in[0,T]}\bigl|(\cL_{\mathrm{null}}m_\phi)(Y_t(\sqrt Ne_1))\bigr|^2\right)\ind{\cE_{\mathrm{reg}}}\right]\\
    &\qquad\le\E_\Q\left[\prod_{i=1}^3\left(\exp\left(\frac{cN}{3K^2}a_i^2\right)\ind{\cE_{\mathrm{reg}}}\right)\right]
    \le\prod_{i=1}^3\E_\Q\left[\exp\left(\frac{cN}{K^2}a_i^2\right)\ind{\cE_{\mathrm{reg}}}\right]^{\nicefrac{1}{3}}
    \le2,
\end{aligned}
\end{align}
where the first inequality follows from \cref{eq:null-center-decomposition} and $(a_1+a_2+a_3)^2\le3(a_1^2+a_2^2+a_3^2)$, the second follows from H\"older's inequality, and the last follows from \cref{eq:null-center-first-term,eq:null-center-second-term,eq:null-center-term-bounds}.

It remains to transfer \cref{eq:null-center-mgf} to the initial state $e$. By \cref{eq:null-drift-linear,eq:null-drift-vector-bounds,eq:spherical-lipschitz}, the following holds on $\cE_{\mathrm{reg}}$ for all $y,y'\in\SN$, where $C>0$ is a universal constant:
\begin{align}
\label{eq:null-drift-lipschitz}
    \norm{F(y)-F(y')}\le CK\norm{y-y'},
    \qquad
    \bigl|(\cL_{\mathrm{null}}m_\phi)(y)-(\cL_{\mathrm{null}}m_\phi)(y')\bigr|\le\frac{CK}{\sqrt N}\norm{y-y'}.
\end{align}

By the second bound in \cref{eq:null-drift-lipschitz}, the following holds on $\cE_{\mathrm{reg}}$ for any two paths $(y_t)_{t\in[0,T]}$ and $(y'_t)_{t\in[0,T]}$ in $\SN$:
\begin{align}
\label{eq:null-path-comparison}
    \sup_{t\in[0,T]}\bigl|(\cL_{\mathrm{null}}m_\phi)(y_t)\bigr|
    \le\sup_{t\in[0,T]}\bigl|(\cL_{\mathrm{null}}m_\phi)(y'_t)\bigr|
    +\frac{CK}{\sqrt N}\sup_{t\in[0,T]}\norm{y_t-y'_t}.
\end{align}
Hence, to transfer \cref{eq:null-center-mgf} to $e$, it suffices to couple two copies of the null trajectory, initialized at $e$ and at $\sqrt Ne_1$ respectively, with each having the same joint law with the environment, such that their distance remains of order $\norm{e-\sqrt Ne_1}$ up to time $T$. This ensures that the last term in \cref{eq:null-path-comparison} is of order $\delta$ on $\{\norm{e-\sqrt Ne_1}\le\delta\sqrt N\}$.

For constructing such a coupling, driving both trajectories by the same Brownian motion in \cref{eq:spherical-langevin-sde} is inconvenient. This is because their noise terms are projections of a common increment $dB_t$ onto different tangent spaces, which leads to the Brownian term itself perturbing the distance between them. To alleviate this, we realize the projected Brownian increments $Q_{Y_t}\,dB_t$ by infinitesimal random rotations of $Y_t$. Both realizations produce a centered Gaussian displacement of $Y_t$ in its tangent space, with covariance $Q_{Y_t}\,dt$, and hence yield the same law for each trajectory. However, a common random rotation moves all states in a distance-preserving fashion, and thus does not perturb the distance between trajectories driven by it.

Precisely, let $\mathfrak B$ be a Brownian motion on the space of antisymmetric $N\times N$ matrices, i.e., $\mathfrak B_{ji}=-\mathfrak B_{ij}$, where $(\mathfrak B_{ij})_{i<j}$ are independent standard Brownian motions on $\R$, and let $B^0$ be a standard Brownian motion on $\R$, such that $\mathfrak B$, $B^0$, the environment and the auxiliary randomness of the initialization rule are independent. By the antisymmetry of $\mathfrak B$, we have the following for every $v\in\R^N$:
\begin{align*}
    d\mathfrak B_t\,v
    =\sum_{i,j=1}^Nv_j\,(d\mathfrak B_t)_{ij}\,e_i
    =\sum_{i<j}\bigl(v_je_i-v_ie_j\bigr)\,(d\mathfrak B_t)_{ij}.
\end{align*}
Consequently, since the entries $(\mathfrak B_{ij})_{i<j}$ are independent standard Brownian motions, we obtain the following for every $v\in\R^N$:
\begin{align}
\label{eq:rotation-identities}
\begin{aligned}
    \dotp{v}{d\mathfrak B_t\,v}
    &=\sum_{i<j}\bigl(v_jv_i-v_iv_j\bigr)\,(d\mathfrak B_t)_{ij}
    =0,\\
    (d\mathfrak B_t\,v)(d\mathfrak B_t\,v)^\top
    &=\sum_{i<j}\bigl(v_je_i-v_ie_j\bigr)\bigl(v_je_i-v_ie_j\bigr)^\top\,dt\\
    &=\frac12\sum_{i,j=1}^N\bigl(v_j^2e_ie_i^\top+v_i^2e_je_j^\top-v_iv_je_ie_j^\top-v_iv_je_je_i^\top\bigr)\,dt
    =\bigl(\norm{v}^2I-vv^\top\bigr)\,dt,
\end{aligned}
\end{align}
For every initial state $y\in\SN$ determined by the environment and the auxiliary randomness, let $(Y^{\mathrm{rot}}_t(y))_{t\ge0}$ denote the solution of
\begin{align}
\label{eq:rotation-sde}
    dY^{\mathrm{rot}}_t(y)=F(Y^{\mathrm{rot}}_t(y))\,dt+\frac{1}{\sqrt N}\,d\mathfrak B_t\,Y^{\mathrm{rot}}_t(y),
    \qquad
    Y^{\mathrm{rot}}_0(y)=y.
\end{align}
Since $F$ is the drift of \cref{eq:spherical-langevin-sde} for $H_{\mathrm{null}}$, \cref{eq:rotation-sde} differs from \cref{eq:spherical-langevin-sde} only in its noise term. We now verify that $Y^{\mathrm{rot}}(y)$ is a copy of the null dynamics initialized at $y$, and throughout this verification, we write $Y^{\mathrm{rot}}_t\doteq Y^{\mathrm{rot}}_t(y)$. We first show that $Y^{\mathrm{rot}}_t$ remains on $\SN$. By It\^o's formula and \cref{eq:rotation-sde}, we have
\begin{align*}
    d\norm{Y^{\mathrm{rot}}_t}^2
    &=2\dotp{Y^{\mathrm{rot}}_t}{F(Y^{\mathrm{rot}}_t)}\,dt
    +\frac{2}{\sqrt N}\dotp{Y^{\mathrm{rot}}_t}{d\mathfrak B_t\,Y^{\mathrm{rot}}_t}
    +\frac1N\Tr\Bigl((d\mathfrak B_t\,Y^{\mathrm{rot}}_t)(d\mathfrak B_t\,Y^{\mathrm{rot}}_t)^\top\Bigr)\\
    &=-\kappa_N\norm{Y^{\mathrm{rot}}_t}^2\,dt+\kappa_N\norm{Y^{\mathrm{rot}}_t}^2\,dt
    =0,
\end{align*}
where the second equality follows from \cref{eq:rotation-identities} and $\dotp{x}{F(x)}=-\tfrac{\kappa_N}{2}\norm{x}^2$, which holds since $Q_xx=0$. Hence, $\norm{Y^{\mathrm{rot}}_t}^2=N$ for all $t\ge0$.

We then identify the driving Brownian motion. Consider the process
\begin{align*}
    \widehat B_t
    \doteq\frac{1}{\sqrt N}\int_0^t\bigl(d\mathfrak B_s\,Y^{\mathrm{rot}}_s+Y^{\mathrm{rot}}_s\,dB_s^0\bigr),
\end{align*}
which is a continuous martingale with respect to the filtration generated by the environment, the auxiliary randomness, $\mathfrak B$ and $B^0$. Since $\mathfrak B$ and $B^0$ are independent, \cref{eq:rotation-identities} and $\norm{Y^{\mathrm{rot}}_t}^2=N$ imply that
\begin{align*}
    d\langle\widehat B\rangle_t
    &=\frac1N\Bigl((d\mathfrak B_t\,Y^{\mathrm{rot}}_t)(d\mathfrak B_t\,Y^{\mathrm{rot}}_t)^\top+Y^{\mathrm{rot}}_t(Y^{\mathrm{rot}}_t)^\top\,dt\Bigr)\\
    &=\frac1N\Bigl(\norm{Y^{\mathrm{rot}}_t}^2I-Y^{\mathrm{rot}}_t(Y^{\mathrm{rot}}_t)^\top+Y^{\mathrm{rot}}_t(Y^{\mathrm{rot}}_t)^\top\Bigr)\,dt
    =I\,dt.
\end{align*}
Hence, by L\'evy's characterization of Brownian motion~\citep[cf.][Theorem~4.1.1]{stroock2007multidimensional}, $\widehat B$ is a standard Brownian motion with respect to this filtration. In particular, $\widehat B$ is independent of the environment and of the auxiliary randomness.

Lastly, we show that the noise term in \cref{eq:rotation-sde} is the projected Brownian increment $Q_{Y^{\mathrm{rot}}_t}\,d\widehat B_t$. Since $Q_{Y^{\mathrm{rot}}_t}Y^{\mathrm{rot}}_t=0$ and $Q_{Y^{\mathrm{rot}}_t}=I-\tfrac1NY^{\mathrm{rot}}_t(Y^{\mathrm{rot}}_t)^\top$, we infer from \cref{eq:rotation-identities} that
\begin{align*}
    Q_{Y^{\mathrm{rot}}_t}\,d\widehat B_t
    &=\frac{1}{\sqrt N}\Bigl(Q_{Y^{\mathrm{rot}}_t}\,d\mathfrak B_t\,Y^{\mathrm{rot}}_t+Q_{Y^{\mathrm{rot}}_t}Y^{\mathrm{rot}}_t\,dB_t^0\Bigr)\\
    &=\frac{1}{\sqrt N}\Bigl(d\mathfrak B_t\,Y^{\mathrm{rot}}_t-\frac1NY^{\mathrm{rot}}_t\dotp{Y^{\mathrm{rot}}_t}{d\mathfrak B_t\,Y^{\mathrm{rot}}_t}\Bigr)
    =\frac{1}{\sqrt N}\,d\mathfrak B_t\,Y^{\mathrm{rot}}_t.
\end{align*}
Substituting this into \cref{eq:rotation-sde}, we obtain
\begin{align*}
    dY^{\mathrm{rot}}_t=F(Y^{\mathrm{rot}}_t)\,dt+Q_{Y^{\mathrm{rot}}_t}\,d\widehat B_t,
    \qquad
    Y^{\mathrm{rot}}_0=y,
\end{align*}
so that $Y^{\mathrm{rot}}$ solves \cref{eq:spherical-langevin-sde} for $H_{\mathrm{null}}$, initialized at $y$ and driven by $\widehat B$. For a given Hamiltonian, the solution of \cref{eq:spherical-langevin-sde} is a deterministic function of its initial point and its driving Brownian motion (\cref{subsec:spherical-langevin}). Moreover, both $\widehat B$ and the driving Brownian motion of $Y(y)$ are independent of the environment and of the auxiliary randomness, which determine $H_{\mathrm{null}}$ and $y$. Consequently, the following holds jointly with the environment and the auxiliary randomness:
\begin{align}
\label{eq:rotation-law}
    (Y^{\mathrm{rot}}_t(y))_{t\ge0}\disteq(Y_t(y))_{t\ge0}.
\end{align}

We now show that the trajectories initialized at $e$ and at $\sqrt Ne_1$ remain close under this coupling.

Let $D_t\doteq Y^{\mathrm{rot}}_t(e)-Y^{\mathrm{rot}}_t(\sqrt Ne_1)$, and note that $D_0=e-\sqrt Ne_1$. Since the Brownian term in \cref{eq:rotation-sde} is linear in $Y^{\mathrm{rot}}_t(y)$, we have
\begin{align*}
    dD_t=\bigl(F(Y^{\mathrm{rot}}_t(e))-F(Y^{\mathrm{rot}}_t(\sqrt Ne_1))\bigr)\,dt+\frac{1}{\sqrt N}\,d\mathfrak B_t\,D_t.
\end{align*}
Consequently, by It\^o's formula, we obtain
\begin{align}
\label{eq:null-coupling-ito}
\begin{aligned}
    d\norm{D_t}^2
    &=2\dotp{D_t}{F(Y^{\mathrm{rot}}_t(e))-F(Y^{\mathrm{rot}}_t(\sqrt Ne_1))}\,dt+\frac{2}{\sqrt N}\dotp{D_t}{d\mathfrak B_t\,D_t}\\
    &\qquad+\frac1N\Tr\Bigl((d\mathfrak B_t\,D_t)(d\mathfrak B_t\,D_t)^\top\Bigr)\\
    &=2\dotp{D_t}{F(Y^{\mathrm{rot}}_t(e))-F(Y^{\mathrm{rot}}_t(\sqrt Ne_1))}\,dt+\frac1N\Tr\bigl(\norm{D_t}^2I-D_tD_t^\top\bigr)\,dt\\
    &=2\dotp{D_t}{F(Y^{\mathrm{rot}}_t(e))-F(Y^{\mathrm{rot}}_t(\sqrt Ne_1))}\,dt+\kappa_N\norm{D_t}^2\,dt\\
    &=2\beta\dotp{D_t}{\gradsp H_{\mathrm{null}}(Y^{\mathrm{rot}}_t(e))-\gradsp H_{\mathrm{null}}(Y^{\mathrm{rot}}_t(\sqrt Ne_1))}\,dt,
\end{aligned}
\end{align}
where the second equality follows from \cref{eq:rotation-identities}, and the last equality follows from the definition of $F$ in \cref{eq:null-drift-linear}. As anticipated, the same infinitesimal rotation is applied to both trajectories and preserves the distance between them, so that only the difference of the gradient drifts controls their separation.

Since \cref{eq:null-coupling-ito} contains no martingale term, $t\mapsto\norm{D_t}^2$ is continuously differentiable, with derivative given by \cref{eq:null-coupling-ito}. Hence, by the Cauchy--Schwarz inequality, \cref{eq:null-drift-lipschitz} and $\kappa_N\le1$, the following holds on $\cE_{\mathrm{reg}}$ for every $t\ge0$:
\begin{align}
\label{eq:null-coupling-differential}
\begin{aligned}
    \frac{d}{dt}\norm{D_t}^2
    &=2\dotp{D_t}{F(Y^{\mathrm{rot}}_t(e))-F(Y^{\mathrm{rot}}_t(\sqrt Ne_1))}+\kappa_N\norm{D_t}^2\\
    &\le2\norm{D_t}\bigl\|F(Y^{\mathrm{rot}}_t(e))-F(Y^{\mathrm{rot}}_t(\sqrt Ne_1))\bigr\|+\norm{D_t}^2
    \le(2CK+1)\norm{D_t}^2.
\end{aligned}
\end{align}

Applying Gr\"onwall's inequality to \cref{eq:null-coupling-differential} and using $T\le1$, we obtain the following on $\cE_{\mathrm{reg}}$:
\begin{align}
\label{eq:null-coupling-distance}
    \sup_{t\in[0,T]}\norm{D_t}
    \le e^{CK+\nicefrac{1}{2}}\norm{D_0}
    =e^{CK+\nicefrac{1}{2}}\norm{e-\sqrt Ne_1}.
\end{align}

Finally, we combine the two steps. By \cref{eq:rotation-law} with $y=\sqrt Ne_1$, the bound \cref{eq:null-center-mgf} also holds with $Y^{\mathrm{rot}}_t(\sqrt Ne_1)$ in place of $Y_t(\sqrt Ne_1)$, since the expectation therein depends only on the joint law of the environment and the trajectory. Let $c_*\doteq\tfrac{c}{18K^2}$ and $C_*\doteq\max\{2,\tfrac{cC^2}{9}e^{2CK+1}\}$.

On $\cE_{\mathrm{reg}}\cap\{\norm{e-\sqrt Ne_1}\le\delta\sqrt N\}$, \cref{eq:null-path-comparison,eq:null-coupling-distance} imply that
\begin{align*}
    \sup_{t\in[0,T]}\bigl|(\cL_{\mathrm{null}}m_\phi)(Y^{\mathrm{rot}}_t(e))\bigr|
    \le\sup_{t\in[0,T]}\bigl|(\cL_{\mathrm{null}}m_\phi)(Y^{\mathrm{rot}}_t(\sqrt Ne_1))\bigr|+CKe^{CK+\nicefrac{1}{2}}\delta.
\end{align*}

Hence, the inequality $(a+b)^2\le2a^2+2b^2$ and \cref{eq:null-center-mgf} for $Y^{\mathrm{rot}}_t(\sqrt Ne_1)$ yields the following
\begin{align*}
    &\E_{\Q}\left[\exp\left(c_*N\sup_{t\in[0,T]}\bigl|(\cL_{\mathrm{null}}m_\phi)(Y^{\mathrm{rot}}_t(e))\bigr|^2\right)\ind{\cE_{\mathrm{reg}}}\ind{\norm{e-\sqrt Ne_1}\le\delta\sqrt N}\right]\\
    &\qquad\le e^{C_*N\delta^2}\,\E_{\Q}\left[\exp\left(\frac{cN}{9K^2}\sup_{t\in[0,T]}\bigl|(\cL_{\mathrm{null}}m_\phi)(Y^{\mathrm{rot}}_t(\sqrt Ne_1))\bigr|^2\right)\ind{\cE_{\mathrm{reg}}}\right]
    \le2e^{C_*N\delta^2}
    \le C_*e^{C_*N\delta^2}.
\end{align*}

By \cref{eq:rotation-law} with $y=e$, the left-hand side is unchanged if $Y^{\mathrm{rot}}_t(e)$ is replaced by $Y_t(e)$, which concludes the proof.
\end{proof}

By Markov's inequality, \cref{lem:null-trajectory-uniform-concentration} implies sharp subgaussian-style concentration for $\cL_{\mathrm{null}}m_\phi$ along the entire Langevin trajectory under the null law, provided that the null Hamiltonian is regular and the initialization is sufficiently close to $\sqrt Ne_1$. We wish to use this estimate to control the influence of the disorder along the true Langevin trajectory, which requires transferring it from the null law to the planted law. We call this approach the \emph{reverse planting trick}, as it runs in the direction opposite to canonical applications of the planting trick~\citep{achlioptas2008algorithmic}, where the null model is the true distribution of interest, and statements proved under the planted model are transferred to it via contiguity.

However, the planted and null laws are far from contiguous, as their Hamiltonians differ by the signal term $H_{\mathrm{sig}}$, whose leading order term is $N\lambda m_\phi^k$. The discrepancy between the two laws is thus governed by the strength of the signal. Since we only need to transfer the uniform concentration of $\cL_{\mathrm{null}}m_\phi$ to the planted law until the dynamics achieves critical recovery, we instead prove the following fine-grained contiguity estimate, which adapts to the signal strength along the trajectory, and applies to any stopping time.

\begin{theorem}[Adaptive Contiguity Estimate]
\label{thm:adaptive-contiguity}
Fix an initialization rule $e$ and $T>0$. Let $(\cF_t)_{0\le t\le T}$ denote the joint filtration of the environment $(\Phi,R,W)$ and of the trajectory $(Y_t(e))_{0\le t\le T}$. Let $\tau$ be an $(\cF_t)$-adapted stopping time satisfying $\tau\le T$ almost surely. Then, for every $E\in\cF_\tau$,
\begin{align*}
    \Prb(E)
    \le\exp\left(\frac{N\lambda^2\sigma_N^2}{2\kappa_N^2}\right)\E_{\Q}\left[\exp\left(\beta^2\int_0^\tau\norm{\gradsp H_{\mathrm{sig}}(Y_t(e))}^2\,dt\right)\ind{E}\right]^{\nicefrac{1}{2}}.
\end{align*}
\end{theorem}

\begin{proof}
Throughout, we write $Y_t=Y_t(e)$. For $t\le T$, let $\Prb^t$ and $\Q^t$ denote the joint laws of the environment $(\Phi,R,W)$ and of the trajectory $(Y_s)_{0\le s\le t}$ under the planted and null laws, respectively. By \cref{def:planted-law,def:null-law}, these laws factorize as $\Prb^t=\Prb_{\mathrm{env}}\cdot\Prb_{\mathrm{dyn}}^t$ and $\Q^t=\Q_{\mathrm{env}}\cdot\Q_{\mathrm{dyn}}^t$, where the superscript $t$ denotes the restriction to the trajectory up to time $t$. Hence, the likelihood ratio $Z(t)\doteq\frac{d\Prb^t}{d\Q^t}$ factorizes as follows:
\begin{align*}
    Z(t)=Z_{\mathrm{env}}Z_{\mathrm{dyn}}(t),
    \qquad
    Z_{\mathrm{env}}\doteq\frac{d\Prb_{\mathrm{env}}}{d\Q_{\mathrm{env}}},
    \qquad
    Z_{\mathrm{dyn}}(t)\doteq\frac{d\Prb_{\mathrm{dyn}}^t(\cdot\mid \Phi,R,W)}{d\Q_{\mathrm{dyn}}^t(\cdot\mid \Phi,R,W)}.
\end{align*}

We begin by analyzing $Z_{\mathrm{env}}$. By \cref{def:planted-law,def:null-law}, $W$ is independent of $(\Phi,R)$ and has the same law under $\Prb_{\mathrm{env}}$ and $\Q_{\mathrm{env}}$. Moreover, $(\Phi,R)=(\Psi\Lambda(U),\norm{U})$, where $U\sim\cN(N\lambda\alpha_N\plant,\sigma_N^2I)$ under $\Prb_{\mathrm{env}}$ and $U\sim\cN(0,\sigma_N^2I)$ under $\Q_{\mathrm{env}}$, and $U=R\Phi^\top e_1$ since $\Psi\Lambda(U)U=\norm{U}e_1$. Hence, $Z_{\mathrm{env}}$ is the likelihood ratio of these Gaussian laws evaluated at $R\Phi^\top e_1$, and we obtain the following via a standard Gaussian MGF computation.
\begin{align}
\label{eq:environment-likelihood-second-moment}
    \begin{aligned}
        \E_{\Q}\bigl[Z_{\mathrm{env}}^2\bigr]
        &=\E_{U\sim\cN(0,\sigma_N^2I)}\left[\exp\left(-\frac{1}{\sigma_N^2}\left(\norm{U-N\lambda\alpha_N\plant}^2-\norm{U}^2\right)\right)\right]\\
        &=\exp\left(-\frac{N^2\lambda^2\alpha_N^2}{\sigma_N^2}\right)\E_{U\sim\cN(0,\sigma_N^2I)}\left[\exp\left(\frac{2N\lambda\alpha_N}{\sigma_N^2}\dotp{\plant}{U}\right)\right]\\
        &=\exp\left(\frac{N^2\lambda^2\alpha_N^2}{\sigma_N^2}\right)
        =\exp\left(\frac{N\lambda^2\sigma_N^2}{\kappa_N^2}\right),
    \end{aligned}
\end{align}
where the last equality uses $\sigma_N^2=\kappa_N\sqrt N\alpha_N$.

We now analyze $Z_{\mathrm{dyn}}$. By \cref{def:planted-law,def:null-law}, conditional on the environment, $\Prb_{\mathrm{dyn}}$ and $\Q_{\mathrm{dyn}}$ are the laws of spherical Langevin dynamics at inverse temperature $\beta$ for $H_\phi$ and $H_{\mathrm{null}}$, respectively, initialized at $Y_0=e$, and we let $(B_t)_{0\le t\le T}$ denote the driving Brownian motion under $\Q_{\mathrm{dyn}}$. Since $m_\phi$ is linear, \cref{eq:spherical-derivatives} implies $\gradsp m_\phi(y)=N^{-1/2}Q_y\phi$, and hence, for every $y\in\SN$,
\begin{align}
\label{eq:signal-spherical-gradient}
    \gradsp H_{\mathrm{sig}}(y)
    =\sqrt N\lambda\left(km_\phi(y)^{k-1}-\frac{\sigma_N^2}{\kappa_N^2}\right)Q_y\phi.
\end{align}
For $t\le T$, define $\theta_t\doteq\beta\gradsp H_{\mathrm{sig}}(Y_t)$. Observe that, conditional on the environment, $\theta_t$ is progressively measurable with respect to the joint filtration of the initial point and the driving Brownian motion. Moreover, since $|m_\phi|\le1$ and $\norm{Q_y\phi}\le1$, \cref{eq:signal-spherical-gradient} implies that $\norm{\theta_t}\le\beta\sqrt N\lambda(k+\kappa_N^{-2}\sigma_N^2)$, which is deterministic. We define the stochastic process $(L_t^\theta)_{0\le t\le T}$ and the tilted law $\Q_{\mathrm{dyn}}^{T,\theta}$ as follows:
\begin{align}
\label{eq:dynamical-exponential-martingale}
    L_t^\theta
    \doteq\exp\left(\int_0^t\dotp{\theta_s}{dB_s}-\frac12\int_0^t\norm{\theta_s}^2\,ds\right),
    \qquad
    \frac{d\Q_{\mathrm{dyn}}^{T,\theta}}{d\Q_{\mathrm{dyn}}^T}
    \doteq L_T^\theta.
\end{align}
Then, by \cref{lem:langevin-change-measure} applied conditionally on the environment, $L_t^\theta$ is a mean-one martingale under $\Q_{\mathrm{dyn}}^T$, and thus $\Q_{\mathrm{dyn}}^{T,\theta}$ is a well-defined probability measure. Moreover, by \cref{lem:langevin-change-measure} with $(V,U)=(H_{\mathrm{null}},H_{\mathrm{sig}})$, the trajectory $(Y_t)_{0\le t\le T}$ under $\Q_{\mathrm{dyn}}^{T,\theta}$ is spherical Langevin dynamics at inverse temperature $\beta$ for $H_{\mathrm{null}}+H_{\mathrm{sig}}=H_\phi$, initialized at $Y_0=e$. Hence, the law of the trajectory under $\Q_{\mathrm{dyn}}^{T,\theta}$ is $\Prb_{\mathrm{dyn}}^T$. Finally, since $\theta_t\in T_{Y_t}\SN$, we infer the following from \cref{eq:spherical-langevin-sde}
\begin{align*}
    \dotp{\theta_t}{dB_t}
    =\dotp{\theta_t}{Q_{Y_t}\,dB_t}
    =\dotp{\theta_t}{dY_t}-\beta\dotp{\theta_t}{\gradsp H_{\mathrm{null}}(Y_t)}\,dt.
\end{align*}
Hence $L_t^\theta$ is determined by the environment and the trajectory up to time $t$. Consequently,
\begin{align}
\label{eq:dynamical-likelihood-ratio-path}
    Z_{\mathrm{dyn}}(t)
    =L_t^\theta
    =\exp\left(\int_0^t\dotp{\theta_s}{dY_s}-\beta\int_0^t\dotp{\theta_s}{\gradsp H_{\mathrm{null}}(Y_s)}\,ds-\frac12\int_0^t\norm{\theta_s}^2\,ds\right),
    \qquad t\le T.
\end{align}
Since $L_t^\theta$ is a mean-one martingale under $\Q_{\mathrm{dyn}}$, so is $Z_{\mathrm{dyn}}(t)$.

We now control $Z_{\mathrm{dyn}}(t)^2$. Since $2\theta_t$ is also bounded by a deterministic constant, using \cref{eq:dynamical-likelihood-ratio-path} and applying \cref{lem:langevin-change-measure} with $2\theta$ in place of $\theta$, we conclude that the following process is a mean-one martingale under $\Q_{\mathrm{dyn}}$:
\begin{align}
\label{eq:dynamical-likelihood-ratio-square}
    L_t\doteq L_t^{2\theta}
    =\exp\left(2\int_0^t\dotp{\theta_s}{dB_s}-2\int_0^t\norm{\theta_s}^2\,ds\right)
    =Z_{\mathrm{dyn}}(t)^2\exp\left(-\int_0^t\norm{\theta_s}^2\,ds\right).
\end{align}

Finally, we combine the two factors via an optional stopping argument. Since $Z_{\mathrm{env}}$ is a function of the environment, it is $\cF_0$-measurable. Moreover, recall that conditional on the environment, $Z_{\mathrm{dyn}}(t)$ and $L_t$ are mean-one martingales with respect to the filtration of the trajectory. Hence, for any $0\le s\le t\le T$,
\begin{align*}
    \E_\Q\bigl[Z(t)\mid\cF_s\bigr]
    =Z_{\mathrm{env}}\E_\Q\bigl[Z_{\mathrm{dyn}}(t)\mid\cF_s\bigr]
    =Z_{\mathrm{env}}Z_{\mathrm{dyn}}(s)
    =Z(s),
    \qquad
    \E_\Q\bigl[L_t\mid\cF_s\bigr]
    =L_s,
\end{align*}
and thus, $Z(t)$ and $L_t$ are martingales with respect to $(\cF_t)$ under $\Q$. Since $\tau\le T$ is an $(\cF_t)$-adapted stopping time, Doob's optional stopping theorem~\citep[Theorem~1.2.5]{stroock2007multidimensional} implies that $\E_\Q[Z(T)\mid\cF_\tau]=Z(\tau)$ and $\E_\Q[L_\tau\mid\cF_0]=L_0=1$. Consequently, for any $E\in\cF_\tau$,
\begin{align}
\label{eq:adaptive-contiguity-optional-stopping}
    \begin{aligned}
        \E_\Q\bigl[Z(\tau)\ind{E}\bigr]
        &=\E_\Q\bigl[\E_\Q[Z(T)\mid\cF_\tau]\ind{E}\bigr]
        =\E_\Q\bigl[\E_\Q[Z(T)\ind{E}\mid\cF_\tau]\bigr]\\
        &=\E_\Q\bigl[Z(T)\ind{E}\bigr]
        =\E_\Q\left[\frac{d\Prb^T}{d\Q^T}\ind{E}\right]
        =\Prb(E).
    \end{aligned}
\end{align}
By \cref{eq:adaptive-contiguity-optional-stopping} and the Cauchy--Schwarz inequality, we obtain
\begin{align}
\label{eq:adaptive-contiguity-cauchy-schwarz}
    \begin{aligned}
        \Prb(E)
        &=\E_\Q\left[Z(\tau)\exp\left(-\frac12\int_0^\tau\norm{\theta_s}^2\,ds\right)\cdot\exp\left(\frac12\int_0^\tau\norm{\theta_s}^2\,ds\right)\ind{E}\right]\\
        &\le\E_\Q\left[Z(\tau)^2\exp\left(-\int_0^\tau\norm{\theta_s}^2\,ds\right)\right]^{\nicefrac{1}{2}}\E_\Q\left[\exp\left(\int_0^\tau\norm{\theta_s}^2\,ds\right)\ind{E}\right]^{\nicefrac{1}{2}}.
    \end{aligned}
\end{align}
For the first expectation on the right-hand side, since $Z(\tau)=Z_{\mathrm{env}}Z_{\mathrm{dyn}}(\tau)$ and $Z_{\mathrm{env}}$ is $\cF_0$-measurable, we infer from \cref{eq:environment-likelihood-second-moment,eq:dynamical-likelihood-ratio-square} that
\begin{align}
\label{eq:adaptive-contiguity-first-expectation}
    \begin{aligned}
        \E_\Q\left[Z(\tau)^2\exp\left(-\int_0^\tau\norm{\theta_s}^2\,ds\right)\right]
        &=\E_\Q\left[Z_{\mathrm{env}}^2\,\E_\Q\left[Z_{\mathrm{dyn}}(\tau)^2\exp\left(-\int_0^\tau\norm{\theta_s}^2\,ds\right)\,\middle|\,\cF_0\right]\right]\\
        &=\E_\Q\bigl[Z_{\mathrm{env}}^2\,\E_\Q[L_\tau\mid\cF_0]\bigr]
        =\E_\Q\bigl[Z_{\mathrm{env}}^2\bigr]
        =\exp\left(\frac{N\lambda^2\sigma_N^2}{\kappa_N^2}\right),
    \end{aligned}
\end{align}
where the third equality uses $\E_\Q[L_\tau\mid\cF_0]=1$. Substituting \cref{eq:adaptive-contiguity-first-expectation} into \cref{eq:adaptive-contiguity-cauchy-schwarz} and recalling that $\norm{\theta_s}^2=\beta^2\norm{\gradsp H_{\mathrm{sig}}(Y_s)}^2$, we conclude the proof.
\end{proof}

\begin{remark}
\label{rem:adaptive-contiguity-transfer}
If $\int_0^\tau\norm{\gradsp H_{\mathrm{sig}}(Y_t(e))}^2\,dt\le\Gamma$ holds on $E$ for some deterministic $\Gamma\ge0$, then \cref{thm:adaptive-contiguity} implies that $\Prb(E)\le\exp\bigl(\tfrac{N\lambda^2\sigma_N^2}{2\kappa_N^2}+\tfrac12\beta^2\Gamma\bigr)\Q(E)^{1/2}$. Thus, events that are (sufficiently) exponentially unlikely under the null law remain so under the planted law, albeit at an exponential cost due to the change in environment and the accumulated signal contribution along the trajectory.
\end{remark}

\subsection{Critical Recovery via the Reverse Planting Trick}
\label{subsec:critical-recovery-reverse-planting}

In this subsection, we use the reverse planting trick to establish \cref{thm:critical-recovery-warm-start}. Our proof makes use of the following regularity estimate for the null Hamiltonian.

\begin{lemma}[Regularity of the Null Hamiltonian]
\label{lem:null-hamiltonian-regularity}
Suppose that $\lambda\sigma_N\le1$. There exist constants $C_0,c>0$, depending only on $k$, such that the following holds under the planted law for all sufficiently large $N$:
\begin{align*}
    \Prb\bigl(\norm{H_{\mathrm{null}}}_{\cG^4}>C_0N\bigr)
    \le\exp(-cN).
\end{align*}
\end{lemma}

\begin{proof}
Since $\sigma_N^2\asymp_kN^{-\tfrac{k-1}{2}}$ and $N$ is sufficiently large, we have $\sigma_N\le1$. By \cref{lem:tpca-hamiltonian-under-gauge,def:planted-law,lem:kspin-gaussian-conditioning}, the null Hamiltonian takes the form
\begin{align*}
    H_{\mathrm{null}}=\kappa_N^{-1}R\dotp{e_1}{\cdot}+W,
    \qquad
    R=\norm{U},
    \qquad
    W\disteq H_0^\perp=H_0-\kappa_N^{-1}\dotp{\zeta}{\cdot},
\end{align*}
where $U\sim\cN(N\lambda\alpha_N\plant,\sigma_N^2I)$ and $\zeta\sim\cN(0,\sigma_N^2I)$. To bound the Sobolev norms of the linear terms, observe that, for every nonzero $v\in\R^N$, we have ${\norm{\dotp{v}{\cdot}}_{\cG^4}\lesssim\sqrt N\norm{v}}$ by \cref{lem:overlap-sobolev-bounds} applied to the unit vector $v/\norm{v}$. Combining this bound with the triangle inequality and $\kappa_N\ge\nicefrac{2}{3}$, we obtain
\begin{align}
\label{eq:critical-recovery-null-sobolev}
    \norm{H_{\mathrm{null}}}_{\cG^4}
    \lesssim\sqrt N\,R+\norm{W}_{\cG^4},
    \qquad
    \norm{H_0^\perp}_{\cG^4}
    \lesssim\norm{H_0}_{\cG^4}+\sqrt N\norm{\zeta}.
\end{align}

We now show that $R\lesssim\sqrt N$ and $\norm{W}_{\cG^4}\lesssim_kN$ with high probability. Since $\sigma_N^2=\kappa_N\sqrt N\alpha_N$, $\kappa_N\ge\nicefrac{2}{3}$, $\lambda\sigma_N\le1$ and $\sigma_N\le1$, the triangle inequality yields
\begin{align*}
    R
    =\norm{U}
    \le N\lambda\alpha_N+\norm{U-N\lambda\alpha_N\plant}
    =\frac{\sqrt N\lambda\sigma_N^2}{\kappa_N}+\norm{U-N\lambda\alpha_N\plant}
    \lesssim\sqrt N+\norm{U-N\lambda\alpha_N\plant}.
\end{align*}
Since $U-N\lambda\alpha_N\plant\sim\cN(0,\sigma_N^2I)$ with $\sigma_N\le1$, we obtain the following via Gaussian concentration~\citep[cf.][Theorem~3.1.1]{vershynin2018high}
\begin{align}
\label{eq:critical-recovery-R-tail}
    \Prb\bigl(R>C\sqrt N\bigr)
    \le\exp(-cN).
\end{align}

We now bound $\norm{W}_{\cG^4}$. Since $\zeta\sim\cN(0,\sigma_N^2I)$ with $\sigma_N\le1$, \cref{lem:kspin-sobolev-norm-concentration} with $n=4$ and $r=N$, together with Gaussian concentration, implies that $\norm{H_0}_{\cG^4}\lesssim_kN$ and $\norm{\zeta}\lesssim\sqrt N$ with probability at least $1-\exp(-c_kN)$. Since $W\disteq H_0^\perp$, we infer the following from the second bound in \cref{eq:critical-recovery-null-sobolev}
\begin{align}
\label{eq:critical-recovery-W-tail}
    \Prb\bigl(\norm{W}_{\cG^4}>C_kN\bigr)
    =\Prb\bigl(\norm{H_0^\perp}_{\cG^4}>C_kN\bigr)
    \le\exp(-c_kN).
\end{align}

On the event ${\{R\le C\sqrt N\}\cap\{\norm{W}_{\cG^4}\le C_kN\}}$, the first bound in \cref{eq:critical-recovery-null-sobolev} yields $\norm{H_{\mathrm{null}}}_{\cG^4}\lesssim_kN$. We now fix $C_0$ sufficiently large, depending only on $k$, so that $\norm{H_{\mathrm{null}}}_{\cG^4}\le C_0N$ on this event. Hence, by a union bound, we conclude the following from \cref{eq:critical-recovery-R-tail,eq:critical-recovery-W-tail}
\begin{align*}
    \Prb\bigl(\norm{H_{\mathrm{null}}}_{\cG^4}>C_0N\bigr)
    \le\Prb\bigl(R>C\sqrt N\bigr)+\Prb\bigl(\norm{W}_{\cG^4}>C_kN\bigr)
    \le\exp(-cN),
\end{align*}
\end{proof}

Our proof also makes use of the following Gronwall-type comparison inequality from Ben Arous, Gheissari, and Jagannath~\citep{arous2020algorithmic}.

\begin{lemma}[{\cite[Lemma~5.1]{arous2020algorithmic}}]
\label{lem:power-law-gronwall}
Let $a,c,T>0$, and let $f:[0,T]\to\R$ be a continuous function which satisfies the following:
\begin{align*}
    f(t)\ge a+c\int_0^tf(s)^{k-1}\,ds
    \qquad\forall\,t\in[0,T].
\end{align*}
Then, the following holds for every $t\in[0,T]$ with $(k-2)ca^{k-2}t<1$:
\begin{align*}
    f(t)\ge\bigl[a^{-(k-2)}-(k-2)ct\bigr]^{-\frac{1}{k-2}}.
\end{align*}
\end{lemma}

Equipped with \cref{lem:null-hamiltonian-regularity,lem:power-law-gronwall}, we now establish \cref{thm:critical-recovery-warm-start} by applying the adaptive contiguity estimate up to a carefully chosen stopping time, thereby proving that the null drift is unlikely to overtake the signal drift before the dynamics achieves critical recovery.

\subsubsection{Proof of \texorpdfstring{\cref{thm:critical-recovery-warm-start}}{Theorem~\ref{thm:critical-recovery-warm-start}}}
\label{subsubsec:proof-critical-recovery-warm-start}

\begin{proof}
Throughout this proof, we work under the planted law. We assume $N$ is sufficiently large and that $C_{\beta,k}$ is a sufficiently large constant, depending only on $\beta$ and $k$. Since $e$ may depend on $H$, we write $e=e(H,\omega)$, where $e$ is a measurable function and $\omega$ is auxiliary randomness with a fixed distribution, independent of $H$ and of the driving Brownian motion $(B_t)_{t\ge0}$ of the second phase. For instance, in \cref{alg:averaged-annealed-langevin,alg:annealed-langevin-ergodic-second-moments}, $\omega$ encapsulates the initialization $\tilX_0$ and the driving Brownian motion $(\widetilde B_t)_{0\le t\le T_{\mathrm{av}}}$ of the first phase.

Consider the trajectory $(X_t(e))_{t\ge0}$ of spherical Langevin dynamics for $H$ initialized at $e$. Applying the gauge transformation in \cref{eq:gauge-transformation} to this trajectory, we define
\begin{align}
\label{eq:critical-recovery-gauge-trajectory}
    Y_t\doteq \Phi X_t(e),
    \qquad
    m_t\doteq m_\phi(Y_t).
\end{align}
Since $e= \Phi^\top \Phi e$, \cref{eq:gauge-transformation} implies that $Y_t=Y_t(\Phi e)$, and hence, by \cref{lem:tpca-hamiltonian-under-gauge}, $(Y_t)_{t\ge0}$ is spherical Langevin dynamics for $H_\phi$ initialized at $Y_0=\Phi e$. Since \cref{lem:tpca-hamiltonian-under-gauge} defines $m_\phi(y)=m(\Phi^\top y)$, we also have $m_t=m(X_t(e))$, i.e., $m_t$ is the correlation along the actual trajectory of the algorithm. Moreover, since $H(x)=H_\phi(\Phi x)$, where $H_\phi$ is determined by $(\Phi,R,W)$ as per \cref{lem:tpca-hamiltonian-under-gauge}, the gauge-transformed initial state $\Phi e$ is a measurable function of $(\Phi,R,W,\omega)$, and is therefore a valid initialization rule in the sense of \cref{subsec:reverse-planting}. Now, let $\delta\doteq\lambda\sigma_NN^{-\vartheta}$, and consider the event
\begin{align*}
    \cA\doteq\bigl\{\norm{Y_0-\sqrt Ne_1}\le\delta\sqrt N\bigr\}.
\end{align*}
Recall that $u_H=\E_\pi[\gradsp H]$ denotes the average spherical gradient, and ${\widehat u_H=\tfrac{\sqrt N\,u_H}{\norm{u_H}}}$ denotes its normalization. Since $\Phi \in O(N,\R)$ with $\Phi\widehat u_H=\sqrt Ne_1$, we have $\cA=\{\norm{e-\widehat u_H}\le\delta\sqrt N\}$, which corresponds to the warm start event in \cref{thm:critical-recovery-warm-start}.

We begin by fixing some constants. Since $\sigma_N^2\asymp_kN^{-\tfrac{k-1}{2}}$ and $\tfrac{k-2}{4}\le\gamma\le\tfrac{k-2}{4}+\tfrac{1}{4k}$, we have
\begin{align*}
    N^{-\nicefrac{1}{4}}\lesssim\lambda\sigma_N\lesssim N^{-\tfrac14+\tfrac{1}{4k}},
    \qquad
    \min\bigl\{1,\lambda N^{-\tfrac{k-1}{4}}\bigr\}\asymp_k\lambda\sigma_N,
    \qquad
    \min\bigl\{N,\lambda^2N^{-\tfrac{k-3}{2}}\bigr\}\asymp_kN\lambda^2\sigma_N^2\to\infty.
\end{align*}
Moreover, since $\Phi\widehat u_H=\sqrt Ne_1$, we have $m_\phi(\sqrt Ne_1)=m(\widehat u_H)$ under the planted law. Adjusting constants appropriately, we infer from \cref{cor:avg-spherical-grad-correlation} that there exists a constant $c_0>0$, depending only on $k$, such that
\begin{align}
\label{eq:critical-recovery-warm-start}
    \Prb\bigl[m_\phi(\sqrt Ne_1)\ge2c_0\lambda\sigma_N\bigr]
    \ge1-\exp\bigl(-cN\lambda^2\sigma_N^2\bigr).
\end{align}
We now define
\begin{align*}
    a_0\doteq c_0\lambda\sigma_N,
    \qquad
    \eta\doteq\frac{3\beta k}{16},
    \qquad
    T_\lambda\doteq\frac{C_T}{\lambda a_0^{k-2}},
    \qquad
    C_T\doteq\frac{2^{k-1}}{(k-2)\eta}.
\end{align*}
Since $C_{\beta,k}$ is sufficiently large and
\begin{align*}
    \lambda a_0^{k-2}
    =c_0^{k-2}\lambda^{k-1}\sigma_N^{k-2}
    \asymp_kC_{\beta,k}^{k-1}N^{(k-1)\left(\gamma-\tfrac{k-2}{4}\right)},
\end{align*}
we may assume that $\lambda a_0^{k-2}\ge C_T$, and hence $T_\lambda\le1$.

We now compute the signal drift. Since $m_\phi$ is linear, \cref{eq:spherical-derivatives} implies $\gradsp m_\phi(y)=N^{-1/2}Q_y\phi$, and $\norm{Q_y\phi}^2=1-m_\phi(y)^2$. Since $H_\phi=H_{\mathrm{null}}+H_{\mathrm{sig}}$, where ${H_{\mathrm{sig}}=N\lambda m_\phi^k-\tfrac{N\lambda\sigma_N^2}{\kappa_N^2}m_\phi}$ by \cref{lem:tpca-hamiltonian-under-gauge}, we obtain the following from \cref{eq:langevin-generator,eq:signal-spherical-gradient} for every $y\in\SN$:
\begin{align}
\label{eq:critical-recovery-signal-drift}
\begin{aligned}
    (\cL_{\mathrm{pl}}-\cL_{\mathrm{null}})m_\phi(y)
    &=\beta\dotp{\gradsp H_{\mathrm{sig}}(y)}{\gradsp m_\phi(y)}
    =\beta\lambda\left(km_\phi(y)^{k-1}-\frac{\sigma_N^2}{\kappa_N^2}\right)\bigl(1-m_\phi(y)^2\bigr),\\
    \norm{\gradsp H_{\mathrm{sig}}(y)}^2
    &=N\lambda^2\left(km_\phi(y)^{k-1}-\frac{\sigma_N^2}{\kappa_N^2}\right)^2\bigl(1-m_\phi(y)^2\bigr).
\end{aligned}
\end{align}
Since $\lambda\sigma_N\gtrsim N^{-\nicefrac{1}{4}}$ and $\sigma_N^2\asymp_kN^{-\tfrac{k-1}{2}}$, we have ${\sigma_N^2/\kappa_N^2\le\tfrac{k}{2}(a_0/2)^{k-1}}$. Moreover, $Cr_N\le\nicefrac{1}{2}$, since $r_N\to0$. Hence, whenever $a_0/2\le m_\phi(y)\le Cr_N$, we have ${\sigma_N^2/\kappa_N^2\le\tfrac{k}{2}m_\phi(y)^{k-1}}$ and $1-m_\phi(y)^2\ge\nicefrac{3}{4}$, and we infer from \cref{eq:critical-recovery-signal-drift} that
\begin{align}
\label{eq:signal-gradient-bounds}
    (\cL_{\mathrm{pl}}-\cL_{\mathrm{null}})m_\phi(y)
    \ge2\eta\lambda m_\phi(y)^{k-1},
    \qquad
    \norm{\gradsp H_{\mathrm{sig}}(y)}^2
    \le k^2N\lambda^2m_\phi(y)^{2k-2}.
\end{align}

We now define the Dynkin martingale of $m_\phi$. By \cref{def:planted-law}, conditional on the environment, $(Y_t)_{t\ge0}$ is spherical Langevin dynamics for $H_\phi$ at inverse temperature $\beta$. Hence, by \cref{lem:langevin-dynkin}, the following process is a continuous martingale conditional on the environment:
\begin{align}
\label{eq:critical-recovery-dynkin}
    M_t\doteq m_t-m_0-\int_0^t(\cL_{\mathrm{pl}}m_\phi)(Y_s)\,ds.
\end{align}
We now define the following regularity events:
\begin{align*}
    \cI\doteq\{m_0\ge a_0\},
    \qquad
    \cE_{\mathrm{reg}}\doteq\bigl\{\norm{H_{\mathrm{null}}}_{\cG^4}\le C_0N\bigr\},
    \qquad
    \calM_{T_\lambda}\doteq\left\{\sup_{0\le t\le T_\lambda}\abs{M_t}\le\frac{a_0}{32}\right\},
\end{align*}
where $C_0>0$ is the constant in \cref{lem:null-hamiltonian-regularity}. We shall show that each of these events holds with high probability on $\cA$.

We begin with the event $\cI$. Since $\phi \in \bS^{N-1}$,  $m_\phi$ is $N^{-1/2}$-Lipschitz on $\SN$. Moreover, $\delta\le a_0$ for all sufficiently large $N$, since $\delta=\lambda\sigma_NN^{-\vartheta}$ and $a_0=c_0\lambda\sigma_N$. Hence, the following holds on ${\cA\cap\{m_\phi(\sqrt Ne_1)\ge2a_0\}}$:
\begin{align*}
    m_0
    =m_\phi(Y_0)
    \ge m_\phi(\sqrt Ne_1)-\frac{1}{\sqrt N}\norm{Y_0-\sqrt Ne_1}
    \ge2a_0-\delta
    \ge a_0.
\end{align*}
Consequently, $\cA\cap\cI^c\subseteq\{m_\phi(\sqrt Ne_1)<2a_0\}$. Since $a_0=c_0\lambda\sigma_N$, we infer from \cref{eq:critical-recovery-warm-start} that
\begin{align}
\label{eq:critical-recovery-initial-correlation}
    \Prb(\cA\cap\cI^c)\le\exp\bigl(-cN\lambda^2\sigma_N^2\bigr).
\end{align}

We now consider the event $\cE_{\mathrm{reg}}$. Since $\lambda\sigma_N\le1$, \cref{lem:null-hamiltonian-regularity} implies that
\begin{align}
\label{eq:critical-recovery-regularity}
    \Prb(\cE_{\mathrm{reg}}^c)\le\exp(-cN).
\end{align}

Finally, we consider the event $\calM_{T_\lambda}$. Since $\gradsp m_\phi(y)=N^{-1/2}Q_y\phi$ and $\norm{Q_y\phi}\le1$, the observable $m_\phi$ satisfies the hypothesis of \cref{lem:dynkin-martingale-concentration} with $A=N^{-1/2}$. Hence, applying \cref{lem:dynkin-martingale-concentration} to the Dynkin martingale in \cref{eq:critical-recovery-dynkin} conditionally on the environment, with $T=T_\lambda$ and $r=a_0/32$, we obtain
\begin{align}
\label{eq:critical-recovery-martingale}
    \Prb(\calM_{T_\lambda}^c)
    \le2\exp\left(-\frac{N(a_0/32)^2}{2T_\lambda}\right)
    \le2\exp\bigl(-cN\lambda^2\sigma_N^2\bigr),
\end{align}
where the last step uses $T_\lambda\le1$ and $a_0=c_0\lambda\sigma_N$.

We now define the stopping time up to which we apply the adaptive contiguity estimate. To this end, let $(\cF_t)_{t\ge0}$ denote the joint filtration of the environment $(\Phi,R,W)$ and of the trajectory $(Y_t)_{t\ge0}$, as in \cref{thm:adaptive-contiguity}. Since $m_t$ and $(\cL_{\mathrm{null}}m_\phi)(Y_t)$ are continuous and $(\cF_t)$-adapted, the following are $(\cF_t)$-adapted stopping times:
\begin{align*}
    T_*^Y
    &\doteq\inf\bigl\{t\ge0:m_t\ge Cr_N\bigr\},\\
    \tau^{\downarrow}
    &\doteq\inf\Bigl\{t\ge0:m_t\le\frac{a_0}{2}\Bigr\},\\
    \tau^{\mathrm{bad}}
    &\doteq\inf\bigl\{t\ge0:\abs{(\cL_{\mathrm{null}}m_\phi)(Y_t)}\ge\eta\lambda m_t^{k-1}\bigr\},\\
    \tau
    &\doteq T_\lambda\wedge T_*^Y\wedge\tau^{\downarrow}\wedge\tau^{\mathrm{bad}}.
\end{align*}

For every $t<\tau$, we have $a_0/2<m_t<Cr_N$. Hence, \cref{eq:signal-gradient-bounds} implies the following:
\begin{align}
\label{eq:critical-recovery-signal-gradient}
    \norm{\gradsp H_{\mathrm{sig}}(Y_t)}^2
    \le k^2N\lambda^2m_t^{2k-2}
    \qquad\forall\,t<\tau.
\end{align}
Moreover, \cref{eq:signal-gradient-bounds} and $\tau\le\tau^{\mathrm{bad}}$ imply that
\begin{align}
\label{eq:critical-recovery-planted-drift}
\begin{aligned}
    (\cL_{\mathrm{pl}}m_\phi)(Y_t)
    &=(\cL_{\mathrm{null}}m_\phi)(Y_t)+(\cL_{\mathrm{pl}}-\cL_{\mathrm{null}})m_\phi(Y_t)\\
    &\ge2\eta\lambda m_t^{k-1}-\abs{(\cL_{\mathrm{null}}m_\phi)(Y_t)}
    \ge\eta\lambda m_t^{k-1}
    >0
    \qquad\forall\,t<\tau.
\end{aligned}
\end{align}

We now control the correlation up to time $\tau$. Define $\calM_\tau\doteq\{\sup_{0\le t\le\tau}\abs{M_t}\le a_0/32\}$. By \cref{eq:critical-recovery-dynkin,eq:critical-recovery-planted-drift}, the following holds on $\cI\cap\calM_\tau$ for any $0\le s\le t\le\tau$:
\begin{align*}
    m_t-m_s
    =\int_s^t(\cL_{\mathrm{pl}}m_\phi)(Y_u)\,du+M_t-M_s
    \ge M_t-M_s
    \ge-\frac{a_0}{16}.
\end{align*}
Since $M_0=0$ and $m_0\ge a_0$, we also have
\begin{align}
\label{eq:critical-recovery-overlap-bounds}
    m_t\ge m_0+M_t\ge\frac{31a_0}{32}
    \quad\forall\,t\le\tau,
    \qquad
    m_s\le m_t+\frac{a_0}{16}\le2m_t
    \quad\forall\,s\le t\le\tau.
\end{align}
Similarly, we infer from \cref{eq:critical-recovery-dynkin,eq:critical-recovery-planted-drift} that
\begin{align}
\label{eq:critical-recovery-integral-inequality}
    m_t
    =m_0+\int_0^t(\cL_{\mathrm{pl}}m_\phi)(Y_s)\,ds+M_t
    \ge\frac{31a_0}{32}+\eta\lambda\int_0^tm_s^{k-1}\,ds
    \qquad\forall\,t\le\tau.
\end{align}

We now show that the null drift is unlikely to dominate the signal drift before the dynamics achieves critical recovery. To this end, define the failure event
\begin{align*}
    \cE_{\mathrm{fail}}\doteq\cA\cap\cI\cap\cE_{\mathrm{reg}}\cap\calM_\tau\cap\bigl\{\tau=\tau^{\mathrm{bad}}<T_*^Y\bigr\}.
\end{align*}
Since $\cA$ and $\cI$ are determined by the environment and the initial state $Y_0$, $\cE_{\mathrm{reg}}$ depends only on the environment, and the remaining events are determined by the environment and the trajectory up to time $\tau$, we have $\cE_{\mathrm{fail}}\in\cF_\tau$. To bound $\Prb(\cE_{\mathrm{fail}})$ by \cref{thm:adaptive-contiguity}, we first control the accumulated signal drift on $\cE_{\mathrm{fail}}$ pathwise, i.e., irrespective of the underlying law. By \cref{eq:critical-recovery-signal-gradient,eq:critical-recovery-overlap-bounds}, the following holds on $\cE_{\mathrm{fail}}$:
\begin{align}
\label{eq:critical-recovery-energy-moment}
    \int_0^\tau\norm{\gradsp H_{\mathrm{sig}}(Y_t)}^2\,dt
    \le k^2N\lambda^2\int_0^\tau m_t^{2k-2}\,dt
    \le k^2N\lambda^2(2m_\tau)^{k-1}\int_0^\tau m_t^{k-1}\,dt.
\end{align}
Since $m_0+M_\tau\ge31a_0/32$ by \cref{eq:critical-recovery-overlap-bounds}, we infer the following from \cref{eq:critical-recovery-dynkin,eq:critical-recovery-planted-drift}:
\begin{align}
\label{eq:critical-recovery-drift-integral}
    \eta\lambda\int_0^\tau m_t^{k-1}\,dt
    \le\int_0^\tau(\cL_{\mathrm{pl}}m_\phi)(Y_t)\,dt
    =m_\tau-m_0-M_\tau
    \le m_\tau-\frac{a_0}{2}.
\end{align}
Moreover, since $m_\tau\ge a_0/2$, we also have
\begin{align}
\label{eq:critical-recovery-power-gap}
    m_\tau^{2k-2}-\Bigl(\frac{a_0}{2}\Bigr)^{2k-2}
    \ge m_\tau^{2k-3}\Bigl(m_\tau-\frac{a_0}{2}\Bigr)
    \ge\Bigl(\frac{a_0}{2}\Bigr)^{k-2}m_\tau^{k-1}\Bigl(m_\tau-\frac{a_0}{2}\Bigr).
\end{align}
Combining \cref{eq:critical-recovery-energy-moment,eq:critical-recovery-drift-integral,eq:critical-recovery-power-gap}, we infer that the following holds on $\cE_{\mathrm{fail}}$:
\begin{align}
\label{eq:critical-recovery-energy-combined}
\begin{aligned}
    \int_0^\tau\norm{\gradsp H_{\mathrm{sig}}(Y_t)}^2\,dt
    &\le\frac{2^{k-1}k^2N\lambda}{\eta}m_\tau^{k-1}\Bigl(m_\tau-\frac{a_0}{2}\Bigr)\\
    &\le\frac{2^{k-1}k^2N}{\eta^3\lambda(a_0/2)^{k-2}}\Bigl[\eta^2\lambda^2m_\tau^{2k-2}-\eta^2\lambda^2\Bigl(\frac{a_0}{2}\Bigr)^{2k-2}\Bigr].
\end{aligned}
\end{align}
Since $\tau=\tau^{\mathrm{bad}}$ on $\cE_{\mathrm{fail}}$, we have
\begin{align}
\label{eq:critical-recovery-bad-time}
    \eta^2\lambda^2m_\tau^{2k-2}
    \le\bigl|(\cL_{\mathrm{null}}m_\phi)(Y_\tau)\bigr|^2
    \le\sup_{t\in[0,T_\lambda]}\bigl|(\cL_{\mathrm{null}}m_\phi)(Y_t)\bigr|^2.
\end{align}
Let $c_*$ and $C_*$ denote the constants in \cref{lem:null-trajectory-uniform-concentration} for our choice of $C_0$, which depend only on $\beta$ and $k$. Since $C_{\beta,k}$ is sufficiently large, we may assume that $\eta^3\lambda(a_0/2)^{k-2}\ge2^{k-1}\beta^2k^2/c_*$. Hence, writing $c_3\doteq2^{k-1}\beta^2k^2/\eta$ and combining \cref{eq:critical-recovery-energy-combined,eq:critical-recovery-bad-time}, we conclude that the following holds on $\cE_{\mathrm{fail}}$:
\begin{align}
\label{eq:critical-recovery-energy-bound}
    \beta^2\int_0^\tau\norm{\gradsp H_{\mathrm{sig}}(Y_t)}^2\,dt
    \le c_*N\sup_{t\in[0,T_\lambda]}\bigl|(\cL_{\mathrm{null}}m_\phi)(Y_t)\bigr|^2-c_3N\lambda\Bigl(\frac{a_0}{2}\Bigr)^k.
\end{align}

We now bound $\Prb(\cE_{\mathrm{fail}})$ by transferring the uniform concentration of $\cL_{\mathrm{null}}m_\phi$ from the null law to the planted law. Since $\cE_{\mathrm{fail}}\in\cF_\tau$ and $\tau\le T_\lambda$, applying \cref{thm:adaptive-contiguity} with $\Phi e$ in place of $e$ and $T=T_\lambda$ to $E=\cE_{\mathrm{fail}}$, we infer from \cref{eq:critical-recovery-energy-bound} that
\begin{align*}
    \Prb(\cE_{\mathrm{fail}})
    &\le\exp\left(\frac{N\lambda^2\sigma_N^2}{2\kappa_N^2}\right)\E_\Q\left[\exp\left(\beta^2\int_0^\tau\norm{\gradsp H_{\mathrm{sig}}(Y_t)}^2\,dt\right)\ind{\cE_{\mathrm{fail}}}\right]^{\nicefrac{1}{2}}\\
    &\le\exp\left(\frac{N\lambda^2\sigma_N^2}{2\kappa_N^2}-\frac{c_3}{2}N\lambda\Bigl(\frac{a_0}{2}\Bigr)^k\right)\E_\Q\left[\exp\left(c_*N\sup_{t\in[0,T_\lambda]}\bigl|(\cL_{\mathrm{null}}m_\phi)(Y_t)\bigr|^2\right)\ind{\cE_{\mathrm{reg}}\cap\cA}\right]^{\nicefrac{1}{2}}\\
    &\le C_*^{\nicefrac{1}{2}}\exp\left(\frac{N\lambda^2\sigma_N^2}{2\kappa_N^2}+\frac{C_*}{2}N\delta^2-\frac{c_3}{2}N\lambda\Bigl(\frac{a_0}{2}\Bigr)^k\right),
\end{align*}
where the second inequality uses $\cE_{\mathrm{fail}}\subseteq\cE_{\mathrm{reg}}\cap\cA$, and the last step uses \cref{lem:null-trajectory-uniform-concentration} with $\Phi e$ in place of $e$ and $T=T_\lambda\le1$. Moreover, since $C_{\beta,k}$ is sufficiently large, we may assume that $c_3c_0^2\lambda(a_0/2)^{k-2}\ge32$, and hence
\begin{align*}
    \frac{c_3}{2}N\lambda\Bigl(\frac{a_0}{2}\Bigr)^k
    =\frac{c_3c_0^2}{8}\lambda\Bigl(\frac{a_0}{2}\Bigr)^{k-2}N\lambda^2\sigma_N^2
    \ge4N\lambda^2\sigma_N^2.
\end{align*}
Since $\kappa_N\ge\nicefrac{2}{3}$, $N\delta^2=N^{-2\vartheta}N\lambda^2\sigma_N^2$ and $N\lambda^2\sigma_N^2\gtrsim\sqrt N$, we conclude that
\begin{align}
\label{eq:critical-recovery-failure-probability}
    \Prb(\cE_{\mathrm{fail}})
    \le C_*^{\nicefrac{1}{2}}\exp\bigl(-2N\lambda^2\sigma_N^2\bigr)
    \le\exp\bigl(-N\lambda^2\sigma_N^2\bigr).
\end{align}

We are now ready to show that the dynamics achieves critical recovery before time $T_\lambda$. Define the good event
\begin{align*}
    \cE_{\mathrm{good}}\doteq\cA\cap\cI\cap\cE_{\mathrm{reg}}\cap\calM_{T_\lambda}\cap\bigl(\{\tau<\tau^{\mathrm{bad}}\}\cup\{\tau=T_*^Y\}\bigr).
\end{align*}
Observe that
\begin{align*}
    \bigl\{\tau=\tau^{\mathrm{bad}}<T_*^Y\bigr\}^c
    =\{\tau<\tau^{\mathrm{bad}}\}\cup\{\tau=T_*^Y\}.
\end{align*}
Moreover, since $\tau\le T_\lambda$, we have $\calM_{T_\lambda}\subseteq\calM_\tau$. Hence, $\cE_{\mathrm{good}}=\cA\cap\cI\cap\cE_{\mathrm{reg}}\cap\calM_{T_\lambda}\cap\cE_{\mathrm{fail}}^c$, and by a union bound, we infer from \cref{eq:critical-recovery-initial-correlation,eq:critical-recovery-regularity,eq:critical-recovery-martingale,eq:critical-recovery-failure-probability} that
\begin{align*}
    \Prb(\cA\setminus\cE_{\mathrm{good}})
    \le\Prb(\cA\cap\cI^c)+\Prb(\cE_{\mathrm{reg}}^c)+\Prb(\calM_{T_\lambda}^c)+\Prb(\cE_{\mathrm{fail}})
    \le5\exp\bigl(-cN\lambda^2\sigma_N^2\bigr).
\end{align*}
Henceforth, we work on $\cE_{\mathrm{good}}$. Since $\calM_{T_\lambda}\subseteq\calM_\tau$, \cref{eq:critical-recovery-overlap-bounds,eq:critical-recovery-integral-inequality} hold. We first turn \cref{eq:critical-recovery-integral-inequality} into an explicit lower bound on the correlation. To this end, define the blow-up time
\begin{align*}
    t_\infty
    \doteq\frac{(a_0/2)^{-(k-2)}}{(k-2)\eta\lambda}
    =\frac{2^{k-2}}{(k-2)\eta\lambda a_0^{k-2}}.
\end{align*}
By our choice of $C_T$, we have $t_\infty=T_\lambda/2<T_\lambda$. Since $31a_0/32\ge a_0/2$ and $t\mapsto m_t$ is continuous, applying \cref{lem:power-law-gronwall} to $f(t)=m_t$ on $[0,\tau]$ with $a=a_0/2$ and $c=\eta\lambda$, we infer from \cref{eq:critical-recovery-integral-inequality} that
\begin{align}
\label{eq:critical-recovery-power-law}
    m_t
    \ge\Bigl[\Bigl(\frac{a_0}{2}\Bigr)^{-(k-2)}-(k-2)\eta\lambda t\Bigr]^{-\frac{1}{k-2}}
    \qquad\forall\,t\in[0,\tau\wedge t_\infty).
\end{align}

We now examine which of the stopping times attains the minimum in the definition of $\tau$. By \cref{eq:critical-recovery-overlap-bounds}, we have $m_\tau\ge31a_0/32>a_0/2$, and hence $\tau\ne\tau^{\downarrow}$. Therefore, one of the following two cases holds.

\textbf{Case 1: $\tau=T_*^Y$.} In this case, we have $T_*^Y=\tau\le T_\lambda$, and we are done.

\textbf{Case 2: $\tau<T_*^Y$.} By definition of $\cE_{\mathrm{good}}$, we must have $\tau<\tau^{\mathrm{bad}}$, and hence $\tau=T_\lambda$. Moreover, since $T_*^Y>0$, we have $a_0\le m_0<Cr_N$. Define
\begin{align*}
    t_*\doteq\frac{(a_0/2)^{-(k-2)}-(2Cr_N)^{-(k-2)}}{(k-2)\eta\lambda}.
\end{align*}
Since $a_0/2<2Cr_N$, we have $0<t_*<t_\infty<T_\lambda=\tau$. Hence, \cref{eq:critical-recovery-power-law} holds at $t=t_*$, which implies that
\begin{align*}
    m_{t_*}
    \ge\Bigl[\Bigl(\frac{a_0}{2}\Bigr)^{-(k-2)}-(k-2)\eta\lambda t_*\Bigr]^{-\frac{1}{k-2}}
    =2Cr_N.
\end{align*}
Consequently, $T_*^Y\le t_*<\tau$, which contradicts $\tau<T_*^Y$. Hence, this case cannot occur.

Combining the two cases, we have $T_*^Y\le T_\lambda$ on $\cE_{\mathrm{good}}$. Finally, we invert the gauge transformation in \cref{eq:critical-recovery-gauge-trajectory}. Since $m_t=m(X_t(e))$, we have $T_*^Y=T_*(e)$, and we recall that $\cA=\{\norm{e-\widehat u_H}\le\delta\sqrt N\}$ with $\delta=\lambda\sigma_NN^{-\vartheta}$. Hence, after adjusting constants, we conclude that
\begin{align*}
    \Prb\left[\frac{1}{\sqrt N}\norm{e-\widehat u_H}\le\lambda\sigma_NN^{-\vartheta}
    \ \ \text{and}\ \
    T_*(e)>T_\lambda\right]
    \le\Prb(\cA\setminus\cE_{\mathrm{good}})
    \le\exp\bigl(-cN\lambda^2\sigma_N^2\bigr).
\end{align*}
The proof is concluded by noting that $T_\lambda=C_T/(\lambda a_0^{k-2})\lesssim N^{-(k-1)(\gamma-\frac{k-2}{4})}$, with implicit constants depending only on $\beta$ and $k$.
\end{proof}

\subsection{Proof of \texorpdfstring{\cref{thm:persistent-strong-recovery-algorithmic-threshold}}{Theorem~\ref{thm:persistent-strong-recovery-algorithmic-threshold}}}
\label{subsec:proof-persistent-strong-recovery-algorithmic-threshold}

\begin{proof}
Throughout the proof, we assume that $N$ is sufficiently large, and we set $T_{\mathrm{str}}\doteq C_1+2$, where $C_1$ is the constant in \cref{thm:critical-recovery-warm-start}, which depends only on $\beta$ and $k$. Since $k$ is odd, we have ${\lfloor k/2\rfloor=\tfrac{k-1}{2}}$. Since $\beta_0^{-2}\lambda^{-2}N^{k-1+2\vartheta}\asymp_kN^{2k-2+4\vartheta}$, our choices of $\beta_0$ and $T_{\mathrm{av}}$ satisfy the assumptions of \cref{thm:tpca-averaged-differential-recovery}. By construction, $(X_t)_{t\ge0}$ is spherical Langevin dynamics for $H$ at inverse temperature $\beta$, and on the event of \cref{thm:tpca-averaged-differential-recovery}, we have $\Xbar_{T_{\mathrm{av}}}\ne0$ and $X_0=\widehat X_{T_{\mathrm{av}}}$. If $\lambda\ge C_{\beta,k}N^{\tfrac{k-2}{4}+\tfrac{1}{4k}}$, the claim follows from \cref{cor:averaged-langevin-persistent-strong-recovery}, since $T_{\mathrm{str}}\ge2$ and $\tfrac{k+1}{2k}>\nicefrac{1}{2}$.

Henceforth, we assume that $\lambda=C_{\beta,k}N^\gamma$ with $\tfrac{k-2}{4}\le\gamma<\tfrac{k-2}{4}+\tfrac{1}{4k}$, and let $r_N\doteq\lambda^{-\tfrac{1}{k-1}}$, as in \cref{subsec:critical-recovery-warm-start}. We first show that the second phase attains critical correlation in constant time. To begin, note that \cref{thm:tpca-averaged-differential-recovery} implies that the following holds with probability at least $1-\exp(-c_kN)$:
\begin{align*}
    \frac{1}{\sqrt N}\norm{X_0-\widehat u_H}
    \le N^{-\vartheta}\min\{1,\lambda\sigma_N\}
    \le\lambda\sigma_NN^{-\vartheta}.
\end{align*}
Let $\bar r_N\doteq N^{-\tfrac{\gamma}{k-1}}$, which is a constant multiple of $r_N$. As \cref{thm:critical-recovery-warm-start} allows the initial state to depend on $H$, we may apply it to $e=X_0$, with the constant $C$ therein chosen such that $Cr_N=\bar r_N$. Since $\gamma\ge\tfrac{k-2}{4}$ and $N\lambda^2\sigma_N^2\gtrsim_k\sqrt N$, we obtain the following with probability at least $1-\exp(-c_{\beta,k}\sqrt N)$:
\begin{align}
\label{eq:algorithmic-threshold-critical-time}
    T_*(X_0)=\inf\bigl\{t\ge0:m(X_t)\ge\bar r_N\bigr\}\le C_1.
\end{align}
Moreover, since $\lambda\bar r_N^{k-1}=C_{\beta,k}$ and $\tfrac{\gamma}{k-1}<\tfrac14$, we may apply \cref{thm:strong-recovery-after-critical-correlation} with $s=\tfrac{\gamma}{k-1}$, for which ${\tau_{\bar r_N}=T_*(X_0)}$. Since ${T_{\mathrm{str}}\ge T_*(X_0)+2}$ on the event in \cref{eq:algorithmic-threshold-critical-time} and $N\bar r_N^2\ge\sqrt N$, we conclude the following via a union bound
\begin{align*}
    \Prb\left[\inf_{t\in[T_{\mathrm{str}},T]}m(X_t)<1-N^{-\eta}\right]
    \le\exp\bigl(-c_{\beta,k}\sqrt N\bigr)+\exp\bigl(-c_kN^{\min\{1/2,\,1-2\eta\}}\bigr),
\end{align*}
which proves the claim after adjusting constants appropriately.
\end{proof}

\section{Strong Recovery at the Algorithmic Threshold for Every Tensor Order}
\label{sec:strong-recovery-every-order}

In this section, we analyze \cref{alg:annealed-langevin-ergodic-second-moments} and show that it achieves persistent strong recovery above the algorithmic threshold for any $k\ge3$, irrespective of its parity. Specifically, we prove the following theorem, which is a formal version of \cref{thm:langevin-every-k-informal}. The complete proof is presented in \cref{subsec:proof-persistent-strong-recovery-every-order}.

\begin{theorem}[Persistent Strong Recovery above the Algorithmic Threshold for Every Tensor Order]
\label{thm:persistent-strong-recovery-every-order}
Let $k\ge3$, let $\beta>0$, and assume $\lambda\ge C_{\beta,k}N^{\tfrac{k-2}{4}}$, where $C_{\beta,k}$ is a sufficiently large constant. Fix ${\vartheta\in(0,1)}$ and ${\eta\in(0,\min\{\nicefrac{1}{2},\tfrac{k-2}{4}\})}$, and let $(X_t)_{t\ge0}$ denote the trajectory of the second phase of \cref{alg:annealed-langevin-ergodic-second-moments}, initialized at $\tilX_0\sim\pi$ with inverse temperatures $\beta_0$ and $\beta$ for the two respective phases, where ${\beta_0\asymp_k\lambda^{-1}N^{-\nicefrac{k}{2}-\vartheta}}$ and ${T_{\mathrm{av}}\asymp_kN^{2k+4\vartheta}}$. Then, there exists a constant $T_{\mathrm{str}}>0$, depending only on $\beta$ and $k$, such that the following holds for all sufficiently large $N$ and every deterministic $T$ satisfying ${T_{\mathrm{str}}\le T\le\poly(N)}$:
\begin{align*}
    \Prb\left[\inf_{t\in[T_{\mathrm{str}},T]}\abs{m(X_t)}\ge1-N^{-\eta}\right]
    \ge1-\exp\left(-c_{\beta,k}N^{\min\{1/2,\,1-2\eta\}}\right).
\end{align*}
Moreover, if $k$ is odd, the same bound holds with $m(X_t)$ in place of $\abs{m(X_t)}$.
\end{theorem}

\subsection{Properties of the Augmented Ergodic Second Moment}
\label{subsec:recovery-ergodic-second-moments}

In this section, we analyze the properties of the augmented ergodic second moment of the first phase of \cref{alg:annealed-langevin-ergodic-second-moments}, defined by $\bar M_T\doteq\frac1T\int_0^TF(\tilX_t)^{\otimes2}\,dt$, where $F(x)\doteq(x,1)\in\R^{N+1}$ and $(\tilX_t)_{t\ge0}$ denotes the trajectory of this phase. Recall from \cref{sec:averaged-spherical-derivatives} that the average spherical gradient and the average spherical Hessian of $H$ are $u_H=\E_\pi[\gradsp H]$ and $A_H=\E_\pi[\gradsp^2H]$, respectively. Expanding $F(\tilX_t)^{\otimes2}$ into blocks, we obtain
\begin{align*}
    \bar M_T
    =\begin{pmatrix}
        \Sigmabar_T & \Xbar_T\\
        \Xbar_T^\top & 1
    \end{pmatrix},
\end{align*}
where $\Xbar_T\doteq\frac1T\int_0^T\tilX_t\,dt$ and $\Sigmabar_T\doteq\frac1T\int_0^T\tilX_t\tilX_t^\top\,dt$ denote the ergodic average and the ergodic second moment of the first phase, respectively. Consequently, by \cref{thm:langevin-recovers-average-gradient,thm:langevin-recovers-average-hessian}, at high temperature we expect $\bar M_{T_{\mathrm{av}}}$ to concentrate around $I_{N+1}+2\beta_0\cA_H$, where
\begin{align*}
    \cA_H\doteq\begin{pmatrix}
        A_H & \kappa_N^{-1}u_H\\
        \kappa_N^{-1}u_H^\top & 0
    \end{pmatrix}.
\end{align*}
Here, the factor $\kappa_N^{-1}$ in the off-diagonal blocks of $\cA_H$ stems from the normalization $\tfrac{2\beta_0}{\kappa_N}$ in \cref{thm:langevin-recovers-average-gradient}, as opposed to $2\beta_0$ in \cref{thm:langevin-recovers-average-hessian}.

As discussed in \cref{subsec:sketch-any-order}, the parity of $k$ determines which block of $\cA_H$ carries the signal. To observe this, let us first consider the case when $k$ is odd. In this case, \cref{thm:avg-spherical-grad-snr,thm:avg-spherical-hess-snr} imply that $A_H=0$ and $u_H$ is a noisy observation of $\xopt$. Hence, the top eigenvector of $\cA_H$ is $\tfrac{1}{\sqrt2}(u_H/\norm{u_H},1)$, and rescaling its first $N$ coordinates to lie on $\SN$ yields the normalized average spherical gradient ${\widehat u_H=\tfrac{\sqrt N\,u_H}{\norm{u_H}}}$. Consequently, we expect the initial state of the second phase of \cref{alg:annealed-langevin-ergodic-second-moments} to lie close to $\widehat u_H$, which parallels the setting in \cref{alg:averaged-annealed-langevin}.

When $k$ is even, \cref{thm:avg-spherical-grad-snr} implies that $u_H=0$, so $\cA_H$ is block diagonal, and its top eigenvector is obtained by appending a zero to a top eigenvector of $A_H$. By \cref{thm:avg-spherical-hess-snr}, $A_H$ has a rank-one spike in the planted direction $\plant$, and hence, by \cref{rem:average-spherical-hessian-snr}, its top eigenvector achieves constant absolute correlation with $\plant$. Consequently, we expect the initial state of the second phase to weakly recover the planted signal up to a global sign.

We note that \cref{alg:annealed-langevin-ergodic-second-moments} does not need to know the parity of $k$, since the same top eigenvector computation extracts the signal from whichever block of $\cA_H$ carries it. We formalize these heuristics in the following theorem, which can be interpreted as an analogue of \cref{thm:tpca-averaged-differential-recovery} for \cref{alg:annealed-langevin-ergodic-second-moments}.

\begin{theorem}[Recovery by the Augmented Ergodic Second Moment for Tensor PCA]
\label{thm:recovery-ergodic-second-moments}
Let $\lambda\ge C_kN^{\tfrac{k-2}{4}}$, where $C_k$ is a sufficiently large constant, and let $(\tilX_t)_{t\ge0}$ denote spherical Langevin dynamics for the tensor PCA Hamiltonian $H$ at inverse temperature $\beta_0>0$, initialized at $\tilX_0\sim\pi$. Define
\begin{align*}
    \bar M_T\doteq\frac{1}{T}\int_0^TF(\tilX_t)^{\otimes2}\,dt,
    \qquad
    F(x)\doteq(x,1)\in\R^{N+1}.
\end{align*}
Define $u_H=\E_\pi[\gradsp H]$, and let $(v,a)\in\R^N\times\R$ be a top unit eigenvector of $\bar M_{T_{\mathrm{av}}}$ with $a\ge0$. For every fixed $\vartheta\in(0,1)$ and all sufficiently large $N$, the following hold whenever $0<\beta_0\lesssim_k\lambda^{-1}N^{-\nicefrac{k}{2}-\vartheta}$ and $T_{\mathrm{av}}\gtrsim_k\beta_0^{-2}\lambda^{-2}N^{k+2\vartheta}$, with sufficiently small and large implicit constants, respectively:
\begin{enumerate}
    \item If $k$ is odd, then, with probability at least $1-\exp(-c_kN)$, we have $v\ne0$ and
    \begin{align*}
        \frac{1}{N}\norm{\frac{\sqrt N\,v}{\norm{v}}-\widehat u_H}^2
        \le N^{-2\vartheta}\min\{1,\lambda^2\sigma_N^2\},
    \end{align*}
    where ${\widehat u_H=\tfrac{\sqrt N\,u_H}{\norm{u_H}}}$ denotes the normalized average spherical gradient and $\sigma_N\asymp_kN^{-\tfrac{k-1}{4}}$ is as defined in \cref{thm:avg-spherical-grad-snr}.

    \item If $k$ is even, then, with probability at least $1-\exp(-c_kN)$, we have $v\ne0$ and
    \begin{align*}
        \left|m\left(\frac{\sqrt N\,v}{\norm{v}}\right)\right|
        \ge 1-C_k\left(\frac{N^{\tfrac{k-2}{2}}}{\lambda^2}+N^{-2\vartheta}\right).
    \end{align*}
\end{enumerate}
\end{theorem}

\begin{proof}
Throughout the proof, we assume that $N$ is sufficiently large. We first show that the rescaled augmented ergodic second moment $\widehat\cA\doteq(2\beta_0)^{-1}(\bar M_{T_{\mathrm{av}}}-I_{N+1})$ is close to $\cA_H$ in operator norm. Since $\beta_0>0$, $(v,a)$ is also a top unit eigenvector of $\widehat\cA$. In the notation of \cref{thm:tpca-averaged-differential-recovery}, we have
\begin{align*}
    \bar M_{T_{\mathrm{av}}}
    =\begin{pmatrix}
        \Sigmabar_{T_{\mathrm{av}}} & \Xbar_{T_{\mathrm{av}}}\\
        \Xbar_{T_{\mathrm{av}}}^\top & 1
    \end{pmatrix}.
\end{align*}
Since the bottom-right entries of $\bar M_{T_{\mathrm{av}}}-I_{N+1}$ and $\cA_H$ both vanish, we obtain
\begin{align*}
    \widehat\cA-\cA_H
    =\begin{pmatrix}
        \frac{\Sigmabar_{T_{\mathrm{av}}}-I}{2\beta_0}-A_H & 0\\
        0 & 0
    \end{pmatrix}
    +\frac{1}{\kappa_N}\begin{pmatrix}
        0 & \frac{\kappa_N}{2\beta_0}\Xbar_{T_{\mathrm{av}}}-u_H\\
        \bigl(\frac{\kappa_N}{2\beta_0}\Xbar_{T_{\mathrm{av}}}-u_H\bigr)^\top & 0
    \end{pmatrix}.
\end{align*}
The first matrix on the right-hand side is block diagonal, and hence its operator norm equals that of its top-left block. The second matrix has zero diagonal blocks, and its operator norm equals the Euclidean norm of its top-right block. Consequently, the triangle inequality implies
\begin{align}
\label{eq:second-moment-blockwise-error}
    \norm{\widehat\cA-\cA_H}_{\mathrm{op}}
    \le\norm{\frac{\Sigmabar_{T_{\mathrm{av}}}-I}{2\beta_0}-A_H}_{\mathrm{op}}
    +\frac{1}{\kappa_N}\norm{\frac{\kappa_N}{2\beta_0}\Xbar_{T_{\mathrm{av}}}-u_H}.
\end{align}

By \cref{eq:tpca-averaging-normalized-errors} in the proof of \cref{thm:tpca-averaged-differential-recovery}, the following event holds with probability at least $1-\exp(-c_kN)$:
\begin{align*}
    \cE&\doteq\left\{\norm{\frac{\kappa_N}{2\beta_0}\Xbar_{T_{\mathrm{av}}}-u_H}\le\eps_N\right\}
    \cap\left\{\norm{\frac{\Sigmabar_{T_{\mathrm{av}}}-I}{2\beta_0}-A_H}_{\mathrm{op}}\le\sqrt N\,\eps_N\right\},\\
    \eps_N&\doteq C_k\left[(\beta_0^{-1}+\lambda)\sqrt{\frac{N}{T_{\mathrm{av}}}}+\beta_0\lambda^2\sqrt N\right].
\end{align*}
Since our assumptions on $\beta_0$ and $T_{\mathrm{av}}$ are those of \cref{thm:tpca-averaged-differential-recovery} with $\nicefrac{k}{2}$ in place of $\lfloor k/2\rfloor$, the computation in \cref{eq:tpca-averaging-error-scale} also shows that $\sqrt N\eps_N\le c_k\lambda N^{-\tfrac{k-2}{2}-\vartheta}$, where $c_k>0$ can be made arbitrarily small. Henceforth, we work on $\cE$. By \cref{eq:second-moment-blockwise-error} and the definition of $\cE$, and since $\kappa_N^{-1}\le\tfrac32\le\sqrt N$, we obtain
\begin{align}
\label{eq:second-moment-augmented-error}
    \norm{\widehat\cA-\cA_H}_{\mathrm{op}}
    \le\Bigl(\sqrt N+\frac{1}{\kappa_N}\Bigr)\eps_N
    \le2c_k\lambda N^{-\tfrac{k-2}{2}-\vartheta}.
\end{align}

We first consider odd $k$. By \cref{thm:avg-spherical-grad-snr,thm:avg-spherical-hess-snr}, $A_H=0$ and $u_H=N\lambda\alpha_N\plant+\zeta$, where $\zeta\sim\cN(0,\sigma_N^2I)$, $\alpha_N\asymp_kN^{-\nicefrac{k}{2}}$, and $\sigma_N^2=\kappa_N\sqrt N\alpha_N\asymp_kN^{-\tfrac{k-1}{2}}$. Moreover, we infer the following lower bound on $\norm{u_H}$ from \cref{cor:avg-spherical-grad-correlation}:
\begin{align}
\label{eq:second-moment-gradient-norm}
    \Prb\left[\norm{u_H}\ge c\bigl(N\lambda\alpha_N+\sigma_N\sqrt N\bigr)\right]
    \ge1-C\exp(-cN).
\end{align}
Henceforth, we also work on the event in \cref{eq:second-moment-gradient-norm}, on which $u_H\ne0$.

We now show that $v$ is close to a positive multiple of $u_H$. Since $A_H=0$, we have $\cA_H=\kappa_N^{-1}\bigl(\begin{smallmatrix}0&u_H\\u_H^\top&0\end{smallmatrix}\bigr)$, which has top unit eigenvector $w\doteq\tfrac{1}{\sqrt2}(u_H/\norm{u_H},1)$ and eigengap $\kappa_N^{-1}\norm{u_H}$, since its remaining eigenvalues are $0$ and $-\kappa_N^{-1}\norm{u_H}$. Applying the Davis--Kahan theorem~\citep[Theorem~4.5.5]{vershynin2018high} to $\widehat\cA$ and $\cA_H$, we obtain
\begin{align}
\label{eq:second-moment-davis-kahan}
    \min_{s\in\{\pm1\}}\norm{(v,a)-sw}
    \lesssim\frac{\kappa_N\norm{\widehat\cA-\cA_H}_{\mathrm{op}}}{\norm{u_H}}
    \lesssim\frac{c_k\lambda N^{-\tfrac{k-2}{2}-\vartheta}}{N\lambda\alpha_N+\sigma_N\sqrt N}
    \lesssim_kc_kN^{-\vartheta}\min\left\{1,\frac{\lambda\sigma_N}{\kappa_N}\right\},
\end{align}
where the implicit constants do not depend on $c_k$. Here, the second inequality uses \cref{eq:second-moment-augmented-error,eq:second-moment-gradient-norm} and $\kappa_N\le1$, while the last one follows from $\alpha_N\asymp_kN^{-\nicefrac{k}{2}}$ and $\sigma_N^2=\kappa_N\sqrt N\alpha_N$, as defined in \cref{thm:avg-spherical-grad-snr}.

Since the Davis--Kahan theorem determines $(v,a)$ only up to a global sign, we now show that the minimum in \cref{eq:second-moment-davis-kahan} is attained at $s=1$. Indeed, since $a\ge0$ and the last coordinate of $w$ equals $1/\sqrt2$, the last coordinate of $(v,a)+w$ is at least $1/\sqrt2$, and hence
\begin{align*}
    \norm{(v,a)+w}
    \ge a+\frac{1}{\sqrt2}
    \ge\frac{1}{\sqrt2}.
\end{align*}
In contrast, the right-hand side of \cref{eq:second-moment-davis-kahan} is smaller than $1/\sqrt2$ for all sufficiently large $N$. Consequently, the minimum in \cref{eq:second-moment-davis-kahan} is attained at $s=1$. Since the first $N$ coordinates of $w$ are $\tfrac{u_H}{\sqrt2\norm{u_H}}$, we conclude from \cref{eq:second-moment-davis-kahan} that
\begin{align}
\label{eq:second-moment-eigenvector-error}
    \norm{v-\frac{u_H}{\sqrt2\norm{u_H}}}
    \le\norm{(v,a)-w}
    \lesssim_kc_kN^{-\vartheta}\min\left\{1,\frac{\lambda\sigma_N}{\kappa_N}\right\}.
\end{align}

We now compare the normalizations of $v$ and $u_H$. Since ${\norm{\tfrac{u_H}{\sqrt2\norm{u_H}}}=1/\sqrt2}$ and the right-hand side of \cref{eq:second-moment-eigenvector-error} is smaller than $1/\sqrt2$ for all sufficiently large $N$, we have $v\ne0$. Moreover, since $\tfrac{u_H}{\sqrt2\norm{u_H}}$ is a positive multiple of $u_H$, its normalization is ${\tfrac{u_H}{\norm{u_H}}=\tfrac{\widehat u_H}{\sqrt N}}$. Hence, applying \cref{lem:normalization-lipschitz} to $v$ and $\tfrac{u_H}{\sqrt2\norm{u_H}}$, we infer the following from \cref{eq:second-moment-eigenvector-error}:
\begin{align*}
    \frac{1}{\sqrt N}\norm{\frac{\sqrt N\,v}{\norm{v}}-\widehat u_H}
    &=\norm{\frac{v}{\norm{v}}-\frac{u_H}{\norm{u_H}}}
    \le\frac{2}{\max\{\norm{v},1/\sqrt2\}}\norm{v-\frac{u_H}{\sqrt2\norm{u_H}}}\\
    &\le2\sqrt2\,\norm{v-\frac{u_H}{\sqrt2\norm{u_H}}}
    \lesssim_kc_kN^{-\vartheta}\min\left\{1,\frac{\lambda\sigma_N}{\kappa_N}\right\}.
\end{align*}
Since $\kappa_N\ge\tfrac23$ and $c_k$ can be made arbitrarily small, squaring and taking a union bound over $\cE$ and the event in \cref{eq:second-moment-gradient-norm} proves the claim for odd $k$ after adjusting constants appropriately.

We now consider even $k$. By \cref{thm:avg-spherical-grad-snr,thm:avg-spherical-hess-snr}, we have $u_H=0$ and
\begin{align*}
    A_H=A_{\mathrm{sig}}+W_N,
    \qquad
    A_{\mathrm{sig}}\doteq\lambda\gamma_N\left(\plant(\plant)^\top-\frac{I}{N}\right),
    \qquad
    \gamma_N\asymp_kN^{-\tfrac{k-2}{2}},
\end{align*}
where $\norm{W_N}_{\mathrm{op}}\le C_kN^{-\tfrac{k-2}{4}}$ with probability at least $1-\exp(-cN)$. Henceforth, we also work on this event. Since $\cA_H=\bigl(\begin{smallmatrix}A_{\mathrm{sig}}+W_N&0\\0&0\end{smallmatrix}\bigr)$ for even $k$, we infer the following from \cref{eq:second-moment-augmented-error}
\begin{align*}
    \norm{\widehat\cA-\begin{pmatrix}A_{\mathrm{sig}}&0\\0&0\end{pmatrix}}_{\mathrm{op}}
    \le\norm{\widehat\cA-\cA_H}_{\mathrm{op}}+\norm{W_N}_{\mathrm{op}}
    \lesssim_k\lambda N^{-\tfrac{k-2}{2}-\vartheta}+N^{-\tfrac{k-2}{4}}.
\end{align*}
The matrix $\bigl(\begin{smallmatrix}A_{\mathrm{sig}}&0\\0&0\end{smallmatrix}\bigr)$ has top unit eigenvector $(\plant,0)$ and eigengap $\lambda\gamma_N(1-\nicefrac{1}{N})$, since its remaining eigenvalues are $0$ and $-\lambda\gamma_N/N$. Applying the Davis--Kahan theorem to $\widehat\cA$ and this matrix, we obtain
\begin{align}
\label{eq:second-moment-even-davis-kahan}
    \min_{s\in\{\pm1\}}\norm{v-s\plant}
    \le\min_{s\in\{\pm1\}}\norm{(v,a)-s(\plant,0)}
    \lesssim_k\frac{\lambda N^{-\frac{k-2}{2}-\vartheta}+N^{-\frac{k-2}{4}}}{\lambda N^{-\frac{k-2}{2}}}
    =N^{-\vartheta}+\frac{N^{\tfrac{k-2}{4}}}{\lambda}.
\end{align}
In particular, since the leading constant $C_k$ in our assumption $\lambda\ge C_kN^{\tfrac{k-2}{4}}$ is sufficiently large, we have $\min_{s\in\{\pm1\}}\norm{v-s\plant}<1=\norm{\plant}$, and hence $v\ne0$. Since $m(\sqrt N\,v/\norm{v})=\dotp{\plant}{v/\norm{v}}$, applying \cref{lem:normalization-lipschitz} to $v$ and the unit vector $s\plant$ for each $s\in\{\pm1\}$, we infer the following from \cref{eq:second-moment-even-davis-kahan}:
\begin{align*}
    1-\left|m\left(\frac{\sqrt N\,v}{\norm{v}}\right)\right|
    &=\min_{s\in\{\pm1\}}\left(1-s\dotp{\plant}{\frac{v}{\norm{v}}}\right)
    =\frac12\min_{s\in\{\pm1\}}\norm{\frac{v}{\norm{v}}-s\plant}^2\\
    &\le\frac12\min_{s\in\{\pm1\}}\left(\frac{2\norm{v-s\plant}}{\max\{\norm{v},1\}}\right)^2
    \le2\min_{s\in\{\pm1\}}\norm{v-s\plant}^2
    \lesssim_k\frac{N^{\tfrac{k-2}{2}}}{\lambda^2}+N^{-2\vartheta}.
\end{align*}
Here, the second equality holds since ${\norm{u-s\plant}^2=2-2s\dotp{\plant}{u}}$ for every unit vector $u$.
A union bound proves the claim for even $k$.
\end{proof}

\begin{remark}
\label{rem:second-moment-parameter-cost}
Compared to \cref{thm:tpca-averaged-differential-recovery}, \cref{thm:recovery-ergodic-second-moments} requires, for odd $k$, an inverse temperature $\beta_0$ that is smaller by a factor of $\sqrt N$, and hence an averaging time that is longer by a factor of $N^2$. This is because, by \cref{eq:tpca-averaging-normalized-errors} in the proof of \cref{thm:tpca-averaged-differential-recovery}, the estimation error of the ergodic second moment is larger than that of the ergodic average by a factor of $\sqrt N$. Moreover, the top eigenvector of $\bar M_{T_{\mathrm{av}}}$ inherits the estimation error of the ergodic second moment, irrespective of the parity of $k$. For even $k$, \cref{thm:tpca-averaged-differential-recovery} already relies on the ergodic second moment, and both theorems impose the same requirements on $\beta_0$ and $T_{\mathrm{av}}$.
\end{remark}

\subsection{Boosting Weak Recovery to Persistent Strong Recovery}
\label{subsec:boosting-weak-to-strong-recovery}

The following theorem shows that, for even $k$, once spherical Langevin dynamics has weakly recovered the planted signal up to a global sign, running the process for an additional constant time suffices to achieve $1-\poly(\nicefrac{1}{N})$ correlation with high probability. Beyond that, the correlation stays above this level throughout any $\poly(N)$ time horizon.

The proof relies on the fact that, for even $k$, the tensor PCA Hamiltonian does not change when the planted vector $\xopt$ is replaced by $\pm \xopt$. Hence, at the time of weak recovery, we may flip the sign of the planted vector if necessary, so that the correlation is positive and bounded away from zero, and consequently, far exceeds the critical correlation $\lambda^{-\tfrac{1}{k-1}}$. In this regime, the signal drift dominates the disorder drift uniformly over $\SN$, and the proof of \cref{thm:strong-recovery-after-critical-correlation}, which does not rely on the parity of $k$, applies with minor modifications. Consequently, the guarantee holds for any initial state, even one that depends on the disorder, such as the initial state of the second phase of \cref{alg:annealed-langevin-ergodic-second-moments}.

\begin{theorem}[Persistent Strong Recovery after Weak Recovery]
\label{thm:strong-recovery-after-weak-recovery}
Let $k$ be even, let $(X_t)_{t\ge0}$ denote spherical Langevin dynamics for $H$ at inverse temperature $\beta>0$, and fix arbitrary constants $\rho\in(0,1)$ and $\eta\in(0,\nicefrac{1}{2})$. Assume $\lambda\ge C_{\beta,k}N^{\tfrac{k-2}{4}}$, where $C_{\beta,k}$ is a sufficiently large constant, and define the stopping time
\begin{align*}
    \tau_\rho\doteq\inf\bigl\{t\ge0:\abs{m(X_t)}\ge\rho\bigr\}.
\end{align*}
Then, the following holds for all sufficiently large $N$, every deterministic $T$ satisfying $2\le T\le\poly(N)$, and any initial state $X_0$, possibly depending on $H$,
\begin{align*}
    \Prb\left[\inf_{t\in[2,T]}\abs{m(X_{\tau_\rho+t})}\ge1-N^{-\eta}\right]
    \ge1-\exp(-c_kN)-2\exp\left(-cN^{1-2\eta}\right).
\end{align*}
\end{theorem}

\begin{proof}
Throughout the proof, we assume that $N$ is sufficiently large and $C_{\beta,k}$ is a sufficiently large constant. As in the proof of \cref{thm:strong-recovery-after-critical-correlation}, $\tau_\rho<\infty$ holds almost surely by the support theorem~\citep{millet1994simple}. Hence, by the strong Markov property~\citep[Theorem~6.2.2]{stroock2007multidimensional} applied at $\tau_\rho$, it suffices to consider an initial state $X_0$ with $\abs{m(X_0)}\ge\rho$, possibly depending on $H$. For such an initial state, we define
\begin{align*}
    \widetilde x^\star\doteq\sgn\bigl(m(X_0)\bigr)\xopt,
    \qquad
    \widetilde m(x)\doteq\frac1N\dotp{\widetilde x^\star}{x}.
\end{align*}
Observe that $\widetilde m(X_0)=\abs{m(X_0)}\ge\rho$ and $\abs{m(x)}\ge\widetilde m(x)$ for every $x\in\SN$. Moreover, since $k$ is even, we have $H=N\lambda\widetilde m^k+H_0$. In other words, $H$ is also the tensor PCA Hamiltonian with planted vector $\widetilde x^\star\in\SN$ and disorder $H_0$, and $\widetilde m$ is the corresponding correlation.

We now repeat the steps of the proof of \cref{thm:strong-recovery-after-critical-correlation} with $\widetilde x^\star$ and $\widetilde m$ in place of $\xopt$ and $m$. To this end, we take $C=1$ and $\bar r_N\doteq N^{-s}$, where ${s\doteq\min\{\eta,\tfrac{k-2}{4(k-1)}\}}$. With these choices, the assumptions of \cref{thm:strong-recovery-after-critical-correlation}, together with the standing assumption of \cref{sec:strong-recovery-odd} with $\gamma=\tfrac{k-2}{4}$, read as follows:
\begin{align*}
    \eta<\min\left\{\frac12,\frac{k-2}{4}\right\},
    \qquad
    s<\frac12,
    \qquad
    \lambda\ge C_{\beta,k}N^{\tfrac{k-2}{4}},
    \qquad
    \lambda\bar r_N^{k-1}\ge C_{\beta,k}.
\end{align*}
Here, the first condition holds since $k\ge4$ is even and $\eta<\nicefrac{1}{2}$, the second holds since $s\le\tfrac{k-2}{4(k-1)}<\nicefrac{1}{2}$, and the third is our assumption on $\lambda$. Since $(k-1)s\le\tfrac{k-2}{4}$, the third condition also implies the fourth. Moreover, since $\bar r_N\to0$, we have $\widetilde m(X_0)\ge\rho\ge\bar r_N$, and hence the stopping time $\tau_{\bar r_N}$ in \cref{thm:strong-recovery-after-critical-correlation}, defined with $\widetilde m$ in place of $m$, equals zero.

As noted in \cref{rem:strong-recovery-parity-free}, the proof of \cref{thm:strong-recovery-after-critical-correlation} does not use the parity of $k$. In particular, since $H=N\lambda\widetilde m^k+H_0$, repeating the computation in \cref{eq:strong-recovery-overlap-drift} with $\widetilde m$ in place of $m$, we obtain
\begin{align*}
    \cL_{H,\beta}\widetilde m(x)
    &=\widetilde\nu(x)\bigl(1-\widetilde m(x)\bigr)-\widetilde b(x),\\
    \widetilde\nu(x)
    &\doteq\beta\lambda k\,\widetilde m(x)^{k-1}\bigl(1+\widetilde m(x)\bigr),
    \qquad
    \widetilde b(x)\doteq\frac{\kappa_N}{2}\widetilde m(x)-\beta\dotp{\gradsp H_0(x)}{\gradsp\widetilde m(x)}.
\end{align*}
Similarly, multiplying \cref{eq:strong-recovery-dynkin} by $\sgn(m(X_0))$, which does not depend on $t$, we infer that the Dynkin martingale associated with $\widetilde m$ is $\sgn(m(X_0))M_t^m$, i.e.,
\begin{align*}
    \sgn\bigl(m(X_0)\bigr)M_t^m
    =\widetilde m(X_t)-\widetilde m(X_0)-\int_0^t\cL_{H,\beta}\widetilde m(X_s)\,ds.
\end{align*}

As in the proof of \cref{thm:strong-recovery-after-critical-correlation}, we consider the events
\begin{align*}
    \cE\doteq\left\{\max_{x\in\SN}\norm{\gradsp H_0(x)}\le C_1\sqrt N\right\},
    \qquad
    \calM\doteq\bigcap_{j=0}^{\lceil T\rceil-1}\left\{\sup_{j\le t\le j+1}\abs{M_t^m-M_j^m}\le\frac{N^{-\eta}}{32}\right\},
\end{align*}
where the threshold $\tfrac{1}{32}N^{-\eta}$ in $\calM$ equals $\tfrac{1}{32}\min\{\bar r_N,N^{-\eta}\}$, since $s\le\eta$. As shown therein, the following holds irrespective of how the initial state $X_0$ depends on $H$:
\begin{align}
\label{eq:weak-to-strong-event-bounds}
    \Prb[\cE^c]\le\exp(-c_kN),
    \qquad
    \Prb[\calM^c]\le2\lceil T\rceil\exp\left(-cN^{1-2\eta}\right).
\end{align}
Henceforth, we work on $\cE\cap\calM$.

We emphasize that the initial state $X_0$, and hence the planted vector $\widetilde x^\star$, may depend on the disorder $H_0$. This does not affect the argument, since neither $\cE$ nor $\calM$ involves $\widetilde x^\star$. Moreover, on $\calM$, the increments of the Dynkin martingale $\sgn(m(X_0))M_t^m$ associated with $\widetilde m$ are bounded exactly as those of $M_t^m$. Since the remaining steps of the proof of \cref{thm:strong-recovery-after-critical-correlation} are deterministic on $\cE\cap\calM$, repeating them, we conclude that $\widetilde m(X_t)>1-N^{-\eta}$ for every $t\in[2,T]$.

Finally, since $\abs{m}\ge\widetilde m$, we infer that $\abs{m(X_t)}>1-N^{-\eta}$ for every $t\in[2,T]$ on $\cE\cap\calM$. Hence, we obtain the following from \cref{eq:weak-to-strong-event-bounds} via a union bound:
\begin{align*}
    \Prb\left[\inf_{t\in[2,T]}\abs{m(X_t)}<1-N^{-\eta}\right]
    \le\Prb[\cE^c]+\Prb[\calM^c]
    \le\exp(-c_kN)+2\lceil T\rceil\exp\left(-cN^{1-2\eta}\right).
\end{align*}
Since $T\le\poly(N)$ and $\eta<\nicefrac{1}{2}$, the proof is concluded by absorbing the factor $\lceil T\rceil$ after adjusting constants appropriately.
\end{proof}

\subsection{Proof of \texorpdfstring{\cref{thm:persistent-strong-recovery-every-order}}{Theorem~\ref{thm:persistent-strong-recovery-every-order}}}
\label{subsec:proof-persistent-strong-recovery-every-order}

\begin{proof}
Throughout the proof, we assume that $N$ is sufficiently large. We set $T_{\mathrm{str}}\doteq C_1+2$, where $C_1$ is the constant in \cref{thm:critical-recovery-warm-start}, which depends only on $\beta$ and $k$. Since $\beta_0^{-2}\lambda^{-2}N^{k+2\vartheta}\asymp_kN^{2k+4\vartheta}$, our choices of $\beta_0$ and $T_{\mathrm{av}}$ satisfy the assumptions of \cref{thm:recovery-ergodic-second-moments}. By construction, $(X_t)_{t\ge0}$ is spherical Langevin dynamics for $H$ at inverse temperature $\beta$, and on the event of \cref{thm:recovery-ergodic-second-moments}, we have $v\ne0$ and ${X_0=\tfrac{\sqrt N\,v}{\norm{v}}}$. Since \cref{thm:strong-recovery-after-critical-correlation,thm:critical-recovery-warm-start,thm:strong-recovery-after-weak-recovery} allow the initial state to depend on $H$, we may apply them to $(X_t)_{t\ge0}$.

We first analyze the case when $k$ is odd, wherein we prove the claim with $m(X_t)$ in place of $\abs{m(X_t)}$. Without loss of generality, let $\lambda=C_{\beta,k}N^\gamma$ with $\gamma\ge\tfrac{k-2}{4}$. We infer from \cref{thm:recovery-ergodic-second-moments} that the following holds with probability at least $1-\exp(-c_kN)$:
\begin{align}
\label{eq:every-order-warm-start}
    \frac{1}{\sqrt N}\norm{X_0-\widehat u_H}
    \le N^{-\vartheta}\min\{1,\lambda\sigma_N\}.
\end{align}

We first consider the regime $\gamma\ge\tfrac{k-2}{4}+\tfrac{1}{4k}$. In this regime, we show that the correlation of $X_0$ is already large enough to apply \cref{thm:strong-recovery-after-critical-correlation}. Define $a\doteq\tfrac{k-1}{4k}$, and note that $\gamma\ge(k-1)a$ and $\gamma-\tfrac{k-1}{4}\ge-a$. Since $\sigma_N\asymp_kN^{-\tfrac{k-1}{4}}$ by \cref{thm:avg-spherical-grad-snr}, it follows that $\lambda\sigma_N\gtrsim_kC_{\beta,k}N^{-a}$. Moreover, $\abs{m(x)-m(y)}\le N^{-1/2}\norm{x-y}$ for all $x,y\in\R^N$ by the Cauchy--Schwarz inequality. Hence, we infer from \cref{eq:every-order-warm-start}, \cref{cor:avg-spherical-grad-correlation} and a union bound that the following holds with probability at least $1-\exp(-c_kN^{(k+1)/(2k)})$, after choosing $C_{\beta,k}$ sufficiently large:
\begin{align}
\label{eq:every-order-initial-correlation}
    m(X_0)
    \ge m(\widehat u_H)-\frac{1}{\sqrt N}\norm{X_0-\widehat u_H}
    \ge\bigl(c_k-N^{-\vartheta}\bigr)\min\{1,\lambda\sigma_N\}
    \ge N^{-a}.
\end{align}
Here, the second inequality combines the lower bound on $m(\widehat u_H)$ from \cref{cor:avg-spherical-grad-correlation} with \cref{eq:every-order-warm-start}, and the last inequality uses $\lambda\sigma_N\gtrsim_kC_{\beta,k}N^{-a}$. Let $\bar r_N\doteq N^{-a}$. Since $\gamma\ge(k-1)a$, we have $\lambda\bar r_N^{k-1}\ge C_{\beta,k}$. Hence, applying \cref{thm:strong-recovery-after-critical-correlation} with $s=a$ and $C=1$, for which $\tau_{\bar r_N}=0$ on the event in \cref{eq:every-order-initial-correlation}, and using $T_{\mathrm{str}}\ge2$ and $N\bar r_N^2=N^{\tfrac{k+1}{2k}}$, we obtain the following via a union bound:
\begin{align*}
    \Prb\left[\inf_{t\in[T_{\mathrm{str}},T]}m(X_t)<1-N^{-\eta}\right]
    \le\exp\left(-c_kN^{\min\{\frac{k+1}{2k},\,1-2\eta\}}\right).
\end{align*}
Since $\tfrac{k+1}{2k}>\nicefrac{1}{2}$, the claim follows in this regime after adjusting constants appropriately.

We now consider the regime $\tfrac{k-2}{4}\le\gamma<\tfrac{k-2}{4}+\tfrac{1}{4k}$, and let $r_N\doteq\lambda^{-\tfrac{1}{k-1}}$, as in \cref{subsec:critical-recovery-warm-start}. In this regime, our argument closely mirrors the proof of \cref{thm:persistent-strong-recovery-algorithmic-threshold}. We first show that the second phase attains critical correlation in constant time. By \cref{eq:every-order-warm-start}, we have ${\tfrac{1}{\sqrt N}\norm{X_0-\widehat u_H}\le\lambda\sigma_NN^{-\vartheta}}$. Let $\bar r_N\doteq N^{-\tfrac{\gamma}{k-1}}$, which is a constant multiple of $r_N$. Applying \cref{thm:critical-recovery-warm-start} to $e=X_0$, with the constant $C$ therein chosen such that $Cr_N=\bar r_N$, and using $\gamma\ge\tfrac{k-2}{4}$ and $N\lambda^2\sigma_N^2\gtrsim_k\sqrt N$, we obtain the following with probability at least $1-\exp(-c_{\beta,k}\sqrt N)$:
\begin{align}
\label{eq:every-order-critical-time}
    T_*(X_0)=\inf\bigl\{t\ge0:m(X_t)\ge\bar r_N\bigr\}\le C_1.
\end{align}
Moreover, since $\lambda\bar r_N^{k-1}=C_{\beta,k}$ and $\tfrac{\gamma}{k-1}<\tfrac14$, we may apply \cref{thm:strong-recovery-after-critical-correlation} with $s=\tfrac{\gamma}{k-1}$ and $C=1$, for which $\tau_{\bar r_N}=T_*(X_0)$. Since $T_{\mathrm{str}}\ge T_*(X_0)+2$ on the event in \cref{eq:every-order-critical-time} and $N\bar r_N^2\ge\sqrt N$, we obtain the following via a union bound:
\begin{align*}
    \Prb\left[\inf_{t\in[T_{\mathrm{str}},T]}m(X_t)<1-N^{-\eta}\right]
    \le\exp\bigl(-c_{\beta,k}\sqrt N\bigr)+\exp\bigl(-c_kN^{\min\{1/2,\,1-2\eta\}}\bigr),
\end{align*}
which proves the claim for odd $k$ after adjusting constants appropriately.

Finally, we analyze the case when $k$ is even. Since $C_{\beta,k}$ is sufficiently large, we infer from \cref{thm:recovery-ergodic-second-moments} that the following holds with probability at least $1-\exp(-c_kN)$:
\begin{align}
\label{eq:every-order-weak-recovery}
    \abs{m(X_0)}
    =\left|m\left(\frac{\sqrt N\,v}{\norm{v}}\right)\right|
    \ge\frac12.
\end{align}
In particular, the stopping time $\tau_{1/2}$ in \cref{thm:strong-recovery-after-weak-recovery} equals zero on the event in \cref{eq:every-order-weak-recovery}. Hence, applying \cref{thm:strong-recovery-after-weak-recovery} with $\rho=\nicefrac{1}{2}$, and using $T_{\mathrm{str}}\ge2$, we obtain the following via a union bound:
\begin{align*}
    \Prb\left[\inf_{t\in[T_{\mathrm{str}},T]}\abs{m(X_t)}<1-N^{-\eta}\right]
    \le\exp(-c_kN)+\exp\bigl(-cN^{1-2\eta}\bigr).
\end{align*}
Since $1-2\eta\ge\min\{\nicefrac{1}{2},1-2\eta\}$, the claim follows after adjusting constants appropriately.
\end{proof}

\printbibliography

\end{document}